\documentclass[11pt]{article}
\usepackage{bbm}
\usepackage{eqnarray,amsmath,amsfonts,amsthm,mathrsfs}
\usepackage{color}
\usepackage{bm}
\usepackage{amssymb}
\usepackage[dvips]{graphicx}
\usepackage{epsfig}
\usepackage{epsf}
\usepackage{float}
\usepackage{subfigure}
\usepackage{amsfonts,amsmath,amsthm,amssymb,graphicx,float,fancyhdr,multirow,hyperref}
\usepackage{booktabs,longtable,authblk}
\usepackage{mathrsfs,hhline}

\usepackage{fancyhdr}
\usepackage[top=2.5cm, bottom=2.5cm, left=3cm, right=3cm]{geometry}
\usepackage{graphicx}
\newtheorem{theorem}{Theorem}
\newtheorem{assumption}{Assumption}
\newtheorem{corollary}{Corollary}
\newtheorem{definition}{Definition}

\newtheorem{lemma}{Lemma}
\newtheorem{proposition}{Proposition}
\newtheorem{remark}{Remark}
\newtheorem{example}{Example}

\allowdisplaybreaks[4]

\def\begeqn{\begin{equation}}
\def\endeqn{\end{equation}}
\def\begth{\begin{theorem}}
\def\endth{\end{theorem}}
\def\begprop{\begin{proposition}}
\def\endprop{\end{proposition}}
\def\begcor{\begin{corollary}}
\def\endcor{\end{corollary}}
\def\begdef{\begin{definition}}
\def\enddef{\end{definition}}
\def\beglemm{\begin{lemma}}
\def\endlemm{\end{lemma}}
\def\begexm{\begin{example}}
\def\endexm{\end{example}}
\def\begrem{\begin{remark}}
\def\endrem{\end{remark}}
\def\begassum{\begin{assumption}}
\def\endassum{\end{assumption}}

\numberwithin{equation}{section}

\title{Pairwise Ranking with Gaussian Kernels$^\dag$\footnotetext{\dag~The work described in this paper is supported by the National Natural Science Foundation of China 
(Grants Nos.12171039 and 12061160462) and Shanghai Science and Technology Program [Project Nos. 21JC1400600 and 20JC1412700]. 
Email addresses: ghlei21@m.fudan.edu.cn (G. Lei), leishi@fudan.edu.cn (L. Shi).}}

\author{Guanhang Lei}
\author{Lei Shi}
\affil{School of Mathematical Sciences, Shanghai\linebreak
Key Laboratory for Contemporary Applied Mathematics\linebreak
Fudan University, Shanghai, 200433, P. R. China}

\date{}

\begin{document}
	\maketitle
\begin{abstract}
	Regularized pairwise ranking with Gaussian kernels is one of the cutting-edge learning algorithms. 
	Despite a wide range of applications, a rigorous theoretical demonstration still lacks to support the performance of such ranking estimators. 
	This work aims to fill this gap by developing novel oracle inequalities for regularized pairwise ranking.  With the help of these oracle inequalities, 
	we derive fast learning rates of Gaussian ranking estimators under a general box-counting dimension assumption on the input domain combined with the noise conditions or the standard smoothness condition. 
	Our theoretical analysis improves the existing estimates and shows that a low intrinsic dimension of input space can help the rates circumvent the curse of dimensionality.
\end{abstract}
	
	
\section{Introduction}\label{Section: Introduction}
Ranking a set of objects based on their underlying utility, relevance, or quality is an important topic in statistical inference.
It has been an intense recent study in various fields as diverse as economics, information retrieval, advertising, and collaborative filtering, see, e.g., \cite{Wong1988Linear, Joachims2002Optimizing, Cao2006Adapting}.
Ranking is closely related to classification, which can also be formulated as a supervised learning problem.
However, they are essentially different because ranking aims at the correct ordering of objects rather than the correct prediction of their categories.
This paper considers the \emph{pairwise ranking} problem, which involves comparing two objects.
Generally, we denote the data relevant to one single object as $(x,y)$,
where $x$ is a $d$-dimensional feature vector belonging to a compact set $\mathcal{X} \subset \mathbb{R}^d$ and $y\in \mathcal{Y}$ is the label of the object.
Here $\mathcal{Y}\subset \mathbb{R}$ consists of either discrete or continuous values.
For any two distinct objects, described by $z=(x,y)$ and $z'=(x', y')$, we assign a relation of order according to their labels:
$z$ ranking higher than $z'$ is equivalent to $y>y'$.
In a pairwise ranking task, we predict the ordering between objects based on their features $x$ and $x'$.
To this end, a bivariate function $f:\mathcal{X} \times \mathcal{X} \to \mathbb{R}$, called a \emph{ranking rule}, is introduced to make predictions.
Namely, if $f(x,x') \geq 0$ then we predict that $y > y'$ implying $f$ ranks $z$ higher than $z'$ (here we break tie when $f(x,x') = 0$ in favor of $y > y'$).
Let $P$ be a probability distribution on $\mathcal{Z}:=\mathcal{X}\times \mathcal{Y}$.
We assume that the objects are randomly selected, meaning that they are described by independent and identically distributed (i.i.d.) random samples of distribution $P$.
The predictive performance of a ranking rule $f$ is measured by the \emph{ranking risk}, given by
\begin{equation}\label{rankingrisk}
\begin{split}
    \mathcal{R}(f) := &P \left(\left\{z=(x,y)\in \mathcal{Z},z'=(x',y')\in \mathcal{Z} \mid y > y', f(x , x') < 0 \right\}\right)  \\ 
    & + P \left(\left\{z=(x,y)\in \mathcal{Z},z'=(x',y')\in \mathcal{Z} \mid y < y', f(x , x') \geq 0 \right\}\right),
\end{split}
\end{equation}
or equivalently, the \emph{excess ranking risk}
\begin{equation}\label{excessrankingrisk}
\mathcal{E}(f):=\mathcal{R}(f)-\inf \left\{\mathcal{R}(g) \mid \textrm{$g:\mathcal{X}\times \mathcal{X}\to \mathbb{R}$ is Borel measurable}\right\}.
\end{equation}

In this paper, we construct ranking rules through a regularized pairwise learning scheme with a loss
$\phi: \mathcal{Y} \times \mathcal{Y} \times \mathbb{R} \to [0,\infty)$ and Gaussian kernels.
We call $\phi$ a margin-based loss if $\phi$ is defined through a margin-based loss for classification, namely, $\phi(y, y', t) = \psi(\mathrm{sgn}(y-y')t)$ 
where $\psi: \mathbb{R} \to [0,\infty)$ is taken to be a margin-based classification loss (cf. \cite{Zhang2004Statistical, Lin2004note, Bartlett2006Convexitya}). 
Here for $t\in \mathbb{R}$, $\mathrm{sgn}(t)=1$ if $t > 0$,  $\mathrm{sgn}(t)=-1$ if $t<0$, and $\mathrm{sgn}(t)=0$ if $t=0$. Typical choices of $\psi$, which will be the prime consideration in this work, are the
hinge loss $\psi_{\mathrm{hinge}}(t) := \max\left\{1-t,0\right\}$ and
the square loss $\psi_{\mathrm{square}}(t) := (1-t)^2$.
Let $\mathcal{X}^2:=\mathcal{X} \times \mathcal{X}$. The \emph{pairwise Gaussian kernel} with variance $\sigma>0$ is a function on $\mathcal{X}^2\times \mathcal{X}^2$ given by
\begin{equation}\label{pairwisegaussian}
\begin{split}
    	K^\sigma ((x,x'),(u,u'))&: =\frac{1}{2}\exp\left(-\frac{\|(x,x') - (u,u')\|^2_2}{\sigma^2}\right)\\
    &\qquad - \frac{1}{2}\exp\left(-\frac{\|(x',x) - (u,u')\|^2_2}{\sigma^2}\right),
\end{split}
\end{equation}
where $\|\cdot\|_2$ denotes the usual Euclidean norm. One can verify that $K^{\sigma}$ is continuous and positive semi-definite (thus is symmetric) on $\mathcal{X}^2 \times \mathcal{X}^2$. 
According to \cite{Aronszajn1950Theory}, $K^{\sigma}$ uniquely defines a reproducing kernel Hilbert space (RKHS) $\mathcal{H}_{K^{\sigma}}$.
Concretely, let $K^{\sigma}_{(u,u')}:\mathcal{X}^2 \to \mathbb{R}$ be a function defined by $K^{\sigma}_{(x,x')}(\cdot,\cdot):=K^{\sigma}((x,x'),(\cdot,\cdot))$ for any $(x,x') \in \mathcal{X}^2$.
The RKHS $\mathcal{H}_{K^\sigma}$ induced by $K^{\sigma}$ is the completion of the linear span of
$\left\{K^\sigma_{(x,x')}\mid (x,x')\in \mathcal{X}^2\right\}$ with inner product denoted by $\langle \cdot, \cdot\rangle_{K^{\sigma}}$ satisfying the reproducing property,
that is, $\langle K^{\sigma}_{(x,x')},f \rangle_{K^{\sigma}}=f(x,x')$ for any $(x,x')\in \mathcal{X}^2$ and all $f\in \mathcal{H}_{K^\sigma}$.
Moreover, $\mathcal{H}_{K^\sigma}$ consists of skew-symmetric continuous functions on $\mathcal{X}^2$, i.e., $f(x,x')=-f(x',x),\forall f\in \mathcal{H}_{K^\sigma}$.
The construction of $K^\sigma$ follows the idea of \cite{Pahikkala2010Learning}, which introduces the so-called intransitive kernels on pairs of data to characterize the pairwise skew-symmetric relation. 
Let $\|\cdot\|_{K^{\sigma}}$ denote the norm of $\mathcal{H}_{K^\sigma}$ induced by the inner product.
Given an i.i.d. sample ${\mathbf z}:=\left\{(X_i,Y_i)\right\}_{i=1}^n$ of $P$, the pairwise learning algorithm seeks ranking rules in $\mathcal{H}_{K^{\sigma}}$ through solving an optimization problem of the form
\begin{equation}\label{fz}
	f_{\mathbf z}:=f^{\phi}_{\mathbf z,\sigma, \lambda} \in \mathop{\arg\min}_{f\in \mathcal{H}_{K^\sigma}}\left\{\mathcal{R}^{\phi}_{\mathbf z}(f) + \lambda \|f\|^2_{K^{\sigma}}\right\},
\end{equation}
where $\lambda >0$ is a tuning parameter and $\mathcal{R}^{\phi}_{\mathbf z}(f)$ denotes the \emph{empirical $\phi$-ranking risk} given by
\begin{equation}\label{empiricalrisk}
    \begin{split}
        \mathcal{R}^{\phi}_{\mathbf z}(f) &:= \mathbb{E}_{\mathbf{z}}\phi(Y,Y',f(X,X')) \\
        &= \frac{1}{n(n-1)}\sum_{i=1}^n \sum_{j \neq i} \phi\left(Y_i,Y_j,f(X_i,X_j)\right).
    \end{split}
\end{equation}
When using a margin-based loss, $\mathcal{R}^{\phi}_{\mathbf z}(f)$ takes the following form
\begin{equation}
    \begin{split}
        \mathcal{R}^{\phi}_{\mathbf z}(f) &= \mathbb{E}_{\mathbf{z}}\psi(\mathrm{sgn}(Y -Y')f(X,X')) \\
        &= \frac{1}{n(n-1)}\sum_{i=1}^n \sum_{j \neq i} \psi\left(\mathrm{sgn}(Y_i - Y_j) f(X_i,X_j)\right).
    \end{split}
\end{equation}
Choosing $\mathcal{H}_{K^\sigma}$ as the family of candidate ranking rules naturally imposes the skew-symmetry restriction on $f_{\mathbf z}$,
which will lead to ordering predictions capable of guaranteeing the reciprocal relations from data.
Note that $\mathcal{R}^{\phi}_{\mathbf z}(f)$ is the empirical analogy of \emph{$\phi$-ranking risk}
\begin{equation}\label{phirisk}
	\begin{split}
	\mathcal{R}^{\phi}(f):&=\mathbb{E}\phi\left(Y,Y',f(X,X')\right)\\
	&=\int_{\mathcal{Z}} \int_{\mathcal{Z}} \phi\left(y,y',f(x,x')\right)dP(x,y)dP(x',y').
	\end{split}
\end{equation}
Thus, the \emph{excess $\phi$-ranking risk} is defined by
\begin{equation}\label{excessphirankingrisk}
    \mathcal{E}^{\phi}(f):=\mathcal{R}^{\phi}(f)-\inf \left\{\mathcal{R}^{\phi}(g) \mid \textrm{$g:\mathcal{X}^2 \to \mathbb{R}$ is Borel measurable}\right\}.
\end{equation}
To a certain extent, it also reflects the generalization performance of the ranking rule $f$.

The algorithm \eqref{fz}, or more precisely, the kernel-based regularized empirical risk minimization, is one of the learning schemes that have recently drawn much theoretical attention. 
Its cutting-edge empirical performance in applications, relatively simple implementation, and, last but not least, its flexibility all contribute to the grown interest. 
The flexibility of algorithm \eqref{fz} is made possible by its two main ingredients, namely the RKHS $\mathcal{H}_{K^\sigma}$ and
the loss function $\phi$. To be more specific, the loss function can be used to model
the learning target, while $\mathcal{H}_{K^\sigma}$ with varying width $\sigma$ adapts to the smoothness of the target function, which incorporates the information of the distribution and the input domain. 
It is worthwhile to note that kernels can be defined on arbitrary input domains, allowing them to handle various types of data in addition to standard $\mathbb{R}^d$-valued data. 
For example,  using a bivariate kernel on ${\mathcal X}\times {\mathcal X}$, we can define a general pairwise kernel $K$ to model reciprocal or antisymmetric relations. 
Concretely, let $G:{\cal X}\times {\cal X}\rightarrow \mathbb{R} $ be a  positive semi-definite kernel and $({\cal H}_G, \langle \cdot, \cdot \rangle_G)$ be its associated RKHS. 
A pairwise  kernel $K:{\cal X}^2\times {\cal X}^2\rightarrow \mathbb{R}$ can be defined by
\begin{align*}
  K((x,x'),(u,u'))=\langle G_x-G_{x'},G_{u}-G_{u'}\rangle_G,
\end{align*} where $G_x:=G(x,\cdot)$, $x\in{\cal X}.$ One may refer to \cite{Pahikkala2010Learning,Pahikkala2013Efficient} for more details about these specific pairwise kernels. 
Moreover, the optimization problem related to algorithm \eqref{fz}, which can reduce to a convex quadratic optimization problem, 
are well understood for widely used loss functions such as the hinge loss and the square loss, see, e.g.,  \cite{Cortes1995Supportvector, Suykens1999Least}. 

One of the main topics in theoretical studies on nonparametric ranking methods is the \emph{learning rates}, i.e., the convergence rates of excess ranking error or excess $\phi$-ranking risk, 
which describe their ranking performance on specific classes of distributions. Therefore, more and more effort has been put into deriving well-established learning rates for various ranking algorithms. 
One can refer to \cite{Clemencon2008Ranking, Rejchel2012Ranking, Chen2014Learning, Ying2016Online, Lei2018Generalization} and references therein for some recent progress. 
However, there is little known about under which nontrivial conditions we can obtain fast learning rates. 
In this paper, we introduce suitable noise conditions in pairwise ranking and employ local Rademacher analysis to establish tight learning rates of algorithm \eqref{fz}. 
Our methodology combines the local Rademacher averages with the theory of U-statistics, which incorporates peeling and symmetrization tricks as well as a contraction principle for U-processes. 
Our work is also motivated by the recent work of \cite{Hamm2021Adaptive} in which the authors theoretically analyze the performance of Gaussian support vector machines 
under an assumption of a low intrinsic dimension of the input domain. 
As far as we know, our paper is the first to consider the learning behavior of regularized pairwise ranking with Gaussian kernels. 
This is in contrast with existing literature, which either only considers ranking with unregularized empirical risk minimization, e.g., \cite{Clemencon2008Ranking, Rejchel2012Ranking}, 
or only handles the case, where the kernel is fixed during the training process, e.g., \cite{Chen2014Learning, Zhao2017Learning, Lei2018Generalization, Wang2022Error}. 
We discuss the findings of these articles and compare them to our results in Section \ref{Section: Discussions on Related Work}. Our main contributions are summarized as follows.
\begin{itemize}
    \item We establish an oracle inequality to bound the excess $\phi$-ranking risk of the estimators produced by regularized empirical risk minimization 
    \eqref{fzk} with a general pairwise kernel $K$ (see Theorem \ref{Oracle inequality of pairwise ranking with a general loss}). 
    Our conclusion is based on a capacity assumption, which requires the RKHS $\mathcal{H}_K$ to satisfy the empirical entropy condition (see Assumption \ref{Capacity condition}). 
    This well-established inequality deals with the stochastic part of the convergence analysis and provides a framework for bounding the excess $\phi$-ranking risk. 
    Directly applying this inequality to the marginal-based loss refines the previous analysis of regularized ranking with general pairwise kernels 
    and yields the best learning rates so far (see Theorem \ref{Oracle inequality of pairwise ranking with a margin-base loss} and Section \ref{Section: Discussions on Related Work} for detailed comparisons).
    \item Inspired by the analysis of support vector machine for classification, see, e.g., \cite{Steinwart2008Support}, 
    we introduce two noise conditions in the pairwise ranking setting. The first condition describes the amount of noise in the labels, 
    which is analogous to Tsybakov's noise condition in binary classification (see Assumption \ref{Tsybakov's noise condition}). 
    This condition enables us to establish an elegant calibration inequality that bounds the excess ranking risk by their excess $\phi$-ranking risk, 
    and a refined variance bound for marginal-based losses which further improves the stochastic part of the analysis. 
    The second condition is a geometric assumption for distributions that allows us to estimate the approximation properties of Gaussian kernels (see Assumption \ref{Margin-noise condition} and 
    Subsection \ref{Subsection: Bounding the Approximation Error}). Both of these two conditions play pivotal roles in deriving fast learning rates.
    \item We obtain the learning rates of the Gaussian regularized ranking algorithm \eqref{fz} under a general box-counting dimension assumption on the input domain combined with 
    the noise conditions or the standard regularity condition (see Theorem \ref{Oracle inequality of hinge loss} and Theorem \ref{Oracle inequality of square loss}). 
    Specifically, we first estimate the capacity of pairwise Gaussian kernel space $\mathcal{H}_{K^\sigma}$ under the assumption that 
    the marginal of the data generating distribution on the input space $\mathcal{X}\subset \mathbb{R}^d$ is supported on a set of upper box-counting dimension $\varrho \in(0, d]$. 
    Then, we derive approximation error bounds for the hinge loss under noise conditions and the square loss under the Besov smoothness. 
    Finally, we apply the well-established oracle inequality and calibration inequality to derive fast learning rates for the excess ranking risk. 
    We show that a low intrinsic dimension of input space can help the rates circumvent the curse of dimensionality. 
\end{itemize}

The rest of this paper is organized as follows. In Section \ref{Section: Main Results}, we first introduce basic notations and necessary assumptions. 
Then we present the main results of this paper, 
including a general oracle inequality for excess $\phi$-ranking risk and its application to Gaussian ranking estimators \eqref{fz} with hinge loss and square loss. 
Section \ref{Section: Discussions on Related Work} provides an overview of related work and compares our results with other contributions on pairwise ranking. 
Section \ref{Section: Proofs of the General Oracle Inequalities} presents a detailed proof of the general oracle inequality established in Theorem \ref{Oracle inequality of pairwise ranking with a general loss}. 
In Section \ref{Section: Deriving the Learning Rates for Gaussian Ranking Estimators}, we apply the oracle inequality to derive learning rates for Gaussian ranking estimators with hinge loss and square loss, 
which gives proofs of Theorem \ref{Oracle inequality of hinge loss} and Theorem \ref{Oracle inequality of square loss}. 
Section \ref{Section: Proofs of the General Oracle Inequalities} and Section \ref{Section: Deriving the Learning Rates for Gaussian Ranking Estimators} 
also include some preliminary estimates that deserve attention in their own right.

\section{Main Results}\label{Section: Main Results}
In this section, we state our main results. Let $(\mathcal{H}_K,\langle \cdot,\cdot\rangle_{K} )$ be an RKHS consisting of skew-symmetric functions on $\mathcal{X}^2:=\mathcal{X} \times \mathcal{X}$ 
which is induced by a positive semi-definite kernel $K : \mathcal{X}^2 \times \mathcal{X}^2 \to \mathbb{R}$. 
The first result is an oracle inequality established in Theorem \ref{Oracle inequality of pairwise ranking with a general loss}, providing upper bounds for excess $\phi$-ranking risk of the regularized estimator
\begin{equation}\label{fzk}
	f_{\mathbf z}:=f^{\phi}_{\mathbf z,\lambda} \in \mathop{\arg\min}_{f\in \mathcal{H}_K}\left\{\mathcal{R}^{\phi}_{\mathbf z}(f) + \lambda \|f\|^2_{K}\right\}, 
\end{equation} where $\lambda>0$, $\phi: \mathcal{Y}\times \mathcal{Y} \times \mathbb{R} \to [0,\infty)$ is a general loss function and $\mathcal{R}^{\phi}_{\mathbf z}(f)$ is defined by \eqref{empiricalrisk}.
Hereinafter, we will use $f_{\mathbf z}$ to denote either $f^{\phi}_{\mathbf z,\lambda}$ in \eqref{fzk} or $f^{\phi}_{\mathbf z,\sigma, \lambda}$ in \eqref{fz} which will be specified in the context. 
The oracle inequalities have been extensively studied in the literature of nonparametric statistics (see \cite{Johnstone1998Oracle} and references therein). 
An oracle inequality provides bounds on the risk of a statistical estimator compared with the one called an oracle that has an infinite amount of observation and minimizes the population risk over a specific function class. 
Recall that the $\phi$-ranking risk $\mathcal{R}^{\phi}(f)$ is given by \eqref{phirisk}. 
The oracle in our setting is the infimum that achieves $\min_{f \in \mathcal{H}_K} \lambda \|f\|^2_K +\mathcal{R}^{\phi}(f)$. 
The oracle inequality bounds the excess $\phi$-ranking risk and can be used to establish both consistency and learning rates for the ranking algorithm. 
We give a literature review on the existing oracle inequalities for different ranking algorithms in Section \ref{Section: Discussions on Related Work}.

To state the oracle inequality for $f_{\mathbf z}$ defined by \eqref{fzk}, we introduce some necessary notations and assumptions. Define the \emph{Bayes $\phi$-ranking rule} as
\begin{equation}\label{phirankingrule}
f^*_\phi: = \mathop{\arg\min}\big\{\mathcal{R}^{\phi}(f) \mid \textrm{$f : \mathcal{X}^2 \to \mathbb{R}$ is Borel measurable}\big\}.
\end{equation}
In the following, the terminology ``measurable'' means ``Borel measurable'' unless otherwise specified. Given a measurable function $f: \mathcal{X}^2 \to \mathbb{R}$, 
define $\phi_f: \mathcal{Z}\times \mathcal{Z} \to \mathbb{R}$ and $Q\phi_f:\mathcal{Z} \to \mathbb{R}$ by
\begin{align}
    \phi_f(z,z') &:= \phi(y, y', f(x,x')), \label{phif} \\
    Q\phi_f(z) &:= \mathbb{E}[\phi_f(z,Z')] = \int_{\mathcal{Z}} \phi\left(y,y',f(x,x')\right)dP(x',y'). \label{Qphif}
\end{align} 

We then introduce the truncation operation, allowing us to derive Bernstein conditions of sup-norm bound and variance bound for the truncated random variables, 
which is essential to establish tight concentration inequalities and derive fast learning rates. We say that the loss function $\phi:\mathcal{Y} \times \mathcal{Y} \times \mathbb{R} \to [0, \infty)$ can be truncated at $M>0$, 
if for all $(y, y', t) \in \mathcal{Y} \times \mathcal{Y} \times \mathbb{R}$, there holds $\phi(y, y', \pi(t)) \leq \phi(y, y', t)$, where
\[\pi(t):=
\begin{cases}
	-M &t<-M \\
	t &t \in [-M,M] \\
	M &t>M,
\end{cases}
\] denotes the truncated value of $t$ at $\pm M$. For any $f:\mathcal{X}^2 \to \mathbb{R}$, $\pi(f)$ denotes the truncation of $f$ onto $[-M,M]$ which is given by $\pi(f)(x,x'):=\pi(f(x,x')),\forall (x,x')\in \mathcal{X}^2$. 
The idea of truncation has already been used in the literature of binary classification, see e.g., \cite{Cucker2007Learning, Steinwart2008Support}. We first give the following assumption on the general loss $\phi$.
\begin{assumption}\label{assumption 1}
The loss $\phi: \mathcal{Y}\times \mathcal{Y} \times \mathbb{R} \to [0,\infty)$ can be truncated at some $M>0$ and for all $(y,y') \in \mathcal{Y}\times \mathcal{Y}$, $\phi(y,y',\cdot) : \mathbb{R} \to [0, \infty)$ 
is convex and locally $L$-Lipschitz on $[-M,M]$. 
The Bayes $\phi$-ranking rule $f^*_\phi$  defined by \eqref{phirankingrule} is measurable and skew-symmetric. 
For any skew-symmetric function $f$, $\phi_f$ defined by \eqref{phif} satisfies
\[\phi_f(z,z')=\phi_f(z',z), \quad \forall (z,z')\in \mathcal{Z}\times \mathcal{Z}.\]
Furthermore, there exist constants $B > 0, \tau \in [0,1]$ and $V \geq B^{2-\tau}$ such that 
\begin{equation}\label{sup-norm bound}
   \phi(y, y', t) \leq B, \quad \forall (y, y', t) \in \mathcal{Y} \times \mathcal{Y} \times [-M,M] 
\end{equation}
and for all measurable and skew-symmetric $f:\mathcal{X}^2\to \mathbb{R}$, there holds
\begin{equation}\label{variance bound}
    \mathbb{E}(Q\phi_{\pi(f)} - Q\phi_{f^*_\phi})^2 \leq 
    V\big(\mathbb{E}(Q\phi_{\pi(f)}  - Q\phi_{f^*_\phi})\big)^\tau, 
\end{equation} where $Q\phi_{\pi(f)}$ and $Q\phi_{f^*_\phi}$ are defined through \eqref{Qphif}.
\end{assumption}

Next, we introduce a capacity assumption on $\mathcal{H}_K$ which is described by the covering number or, equivalently, the entropy number of the unit norm ball in $\mathcal{H}_K$.
\begin{definition}
Let $(\mathcal{T},\mathrm{d})$ be a pseudo-metric space and $\varepsilon>0$.  We call $\mathcal{S} \subset \mathcal{T}$ an $\varepsilon$-net of $\mathcal{T}$, 
if for all $t \in \mathcal{T}$ there exists an $s\in \mathcal{S}$ with $\mathrm{d}(s,t)\leq \varepsilon$. The \emph{$\varepsilon$-covering number} of $(\mathcal{T},\mathrm{d})$ is defined by
\[\mathcal{N}(\mathcal{T},\mathrm{d},\varepsilon) := \inf\big\{\vert \mathcal{S} \vert: \mathcal{S} \subset \mathcal{T} \text{ and } \mathcal{S} \text{ is an } \varepsilon\text{-net of }\mathcal{T}\big\},\] 
where $\inf \emptyset=\infty$. For $i \in \mathbb{N}$, the \emph{$i$-th entropy number} of $(\mathcal{T},\mathrm{d})$ is defined by
\[e_i(\mathcal{T},\mathrm{d}) := \inf \bigg\{ \text{$\varepsilon > 0 : \exists s_1, \ldots, s_{2^{i-1}} \in \mathcal{T}$ such that $\mathcal{T} \subset \bigcup_{j=1}^{2^{i-1}}\mathcal{B}_{\mathrm{d}}(s_j,\varepsilon)$}\bigg\},\]
where $\mathcal{B}_{\mathrm{d}}(s,\varepsilon) := \big\{ t \in \mathcal{T} : \mathrm{d}(t,s) \leq \varepsilon \big\}$ denotes the closed ball with center $s\in \mathcal{T}$ and  radius $\varepsilon$. 
In particular, if $(\mathcal{T},\mathrm{d})$ is a subspace of a normed space $(\mathcal{F},\|\cdot\|_{\mathcal{F}})$ and the metric $\mathrm{d}$ is given by 
$\mathrm{d}(s,t):=\|s-t\|_{\mathcal{F}},\forall s,t\in \mathcal{T}$, we write $\mathcal{N}(\mathcal{T},\|\cdot\|_{\mathcal{F}},\varepsilon):=\mathcal{N}(\mathcal{T},\mathrm{d},\varepsilon)$. 
Moreover, if $S: \mathcal{F} \to \mathcal{F}'$ is a bounded, linear operator between the normed spaces $\mathcal{F}$ and $\mathcal{F}'$, 
we write $e_i(S) := e_i(S\mathcal{B}_{\mathcal{F}},\|\cdot\|_{\mathcal{F}'})$ and $\mathcal{N}(S,\varepsilon) := \mathcal{N}(S\mathcal{B}_{\mathcal{F}},\|\cdot\|_{\mathcal{F}'},\varepsilon)$, 
where $\mathcal{B}_{\mathcal{F}}$ is the closed unit ball of $\mathcal{F}$.
\end{definition}

 Let $P_{\mathcal{X}}$ denote the marginal distribution of $P$ on $\mathcal{X}$. Given i.i.d. sample ${\mathbf x}=\left\{X_i\right\}_{i=1}^n$ of $P_\mathcal{X}$, 
 define two empirical measures: $P^n_{\mathbf{x}}:=\frac{1}{n}\sum_{i=1}^n \delta_{X_i}$ and $P^n_{\mathbf{x}^2}:=\frac{1}{n(n-1)}\sum_{i\neq j} \delta_{(X_i,X_j)}$, 
 where $\delta_{(\cdot)}$ is the counting measure. We give the following capacity assumption on the RKHS $\mathcal{H}_K$.
\begin{assumption}\label{Capacity condition}
    For all $n \in \mathbb{N}$ and i.i.d. sample ${\mathbf x}=\left\{X_i\right\}_{i=1}^n$ of $P_\mathcal{X}$,
    there exist constants $0<p_1,p_2<1/2$, $a_1>0$ and $a_2>0$ such that 
    \begin{align*}
        e_i(\mathrm{id}: \mathcal{H}_K \to \mathcal{L}_2(P^n_{\mathbf{x}} \otimes P_{\mathcal{X}})) &\leq a_1i^{-\frac{1}{2p_1}}, \quad \forall i \in \mathbb{N}, \\
        e_i(\mathrm{id}: \mathcal{H}_K \to \mathcal{L}_2(P^n_{\mathbf{x}^2})) &\leq a_2i^{-\frac{1}{2p_2}}, \quad \forall i \in \mathbb{N}.
    \end{align*} 
    Here $\mathcal{L}_2(P^n_{\mathbf{x}} \otimes P_{\mathcal{X}})$ and $\mathcal{L}_2(P^n_{\mathbf{x}^2})$ are the Hilbert spaces of square-integrable functions with respect to 
    $P^n_{\mathbf{x}} \otimes P_{\mathcal{X}}$ and $P^n_{\mathbf{x}^2}$ in which the norms are defined for $f \in \mathcal{H}_K$ by 
    \begin{align*}
        \|f\|_{\mathcal{L}_2(P^n_{\mathbf{x}} \otimes P_{\mathcal{X}})} &:= \bigg(\frac{1}{n}\sum_{i=1}^n\mathbb{E}_{P_\mathcal{X}}
        \left[f(X_i,X)^2\right]\bigg)^{\frac{1}{2}} , \\
        \|f\|_{\mathcal{L}_2(P^n_{\mathbf{x}^2})} &:= \bigg(\frac{1}{n(n-1)}\sum_{i=1}^n\sum_{j \neq i}
        f(X_i,X_j)^2\bigg)^{\frac{1}{2}}.
    \end{align*}
\end{assumption}
Due to the connection between entropy numbers and covering numbers, cf. Lemma 6.21 of \cite{Steinwart2008Support}, 
the capacity assumption on entropy numbers can be restated by using the metric entropy, namely the logarithm of the covering number, i.e., there exist constants $0<p_1,p_2<1/2$, $a_1>0$ and $a_2>0$ such that 
\begin{align*}
    \log\mathcal{N}(\mathrm{id}: \mathcal{H}_K \to \mathcal{L}_2(P^n_{\mathbf{x}} \otimes P_{\mathcal{X}}), \varepsilon) &\leq \log(4)a_1^{2p_1}\varepsilon^{-2p_1}, \\
    \log\mathcal{N}(\mathrm{id}: \mathcal{H}_K \to \mathcal{L}_2(P^n_{\mathbf{x}^2}), \varepsilon) &\leq \log(4)a_2^{2p_2}\varepsilon^{-2p_2}.
\end{align*}

Recall that the $\phi$-ranking risk $\mathcal{R}^{\phi}(\cdot)$ is given by \eqref{phirisk} and the Bayes $\phi$-ranking rule $f^*_\phi$ is defined by \eqref{phirankingrule}. 
Note that if the loss $\phi$ can be truncated at some $M>0$, there holds 
\begin{equation*}
    \mathcal{R}^{\phi}(\pi(f)) -  \mathcal{R}^{\phi}(f^*_\phi)\leq \mathcal{R}^{\phi}(f) -  \mathcal{R}^{\phi}(f^*_\phi)
\end{equation*} for all $f:\mathcal{X}^2 \to \mathbb{R}$. 
In the following, we can always consider truncated estimators because projecting the values of decision functions onto $[-M, M]$ does not increase their excess $\phi$-ranking risk. 
Now we can give our oracle inequality for the estimators generated by algorithm \eqref{fzk} with a general loss $\phi$.

\begin{theorem}\label{Oracle inequality of pairwise ranking with a general loss}  
    Let ${\mathbf z}=\left\{(X_i,Y_i)\right\}_{i=1}^n$ be an i.i.d. sample of a probability distribution $P$ on 
$\mathcal{X} \times \mathcal{Y}$. Given a loss function $\phi:\mathcal{Y} \times \mathcal{Y} \times \mathbb{R} \to [0,\infty)$ and a measurable positive semi-definite kernel 
    $K: \mathcal{X}^2 \times \mathcal{X}^2 \to \mathbb{R}$ with the associated RKHS $\mathcal{H}_K$  consisting of skew-symmetric functions, the estimator $f_{\mathbf z}$ is defined by \eqref{fzk} with $\lambda>0$. 
    Assume that $\phi$ satisfies Assumption \ref{assumption 1} with $L>0, B > 0, \tau \in [0,1]$ and $V \geq B^{2-\tau}$, and $\mathcal{H}_K$ satisfies Assumption \ref{Capacity condition} with 
    $0<p_1,p_2<1/2$, $a_1 \geq B$ and $a_2 \geq B$. If there exists some $f_0 \in \mathcal{H}_K$ satisfying $\|\phi_{f_0}\|_\infty \leq B_0$ for some constant $B_0 \geq 0$, 
    then for all $n \geq 2 $ and $t>0$, with probability at least $1-(c_0 + 5)\exp(-t)$, there holds
    \begin{equation}\label{oracleinqualitygeneral}
        \begin{split}
            &\lambda\|f_\mathbf{z}\|^2_K + \mathcal{R}^\phi(\pi({f}_\mathbf{z})) - \mathcal{R}^\phi(f^*_\phi) \\
            \leq{}& 8\big(\lambda\|f_0\|^2_K + \mathcal{R}^\phi(f_0) - \mathcal{R}^\phi(f^*_\phi)\big)  
            + 6c_1 \left(\frac{a_1^{2p_1}V^{1-p_1}L^{2p_1}}{\lambda^{p_1} n}\right)^{\frac{1}{2-\tau-p_1+p_1\tau}} \\
            & + \frac{6c_2a_1^{2p_1}B^{1-p_1}L^{2p_1}}{\lambda^{p_1}n} 
            + \frac{3c_3a_1^{2p_1}L^{2p_1}B^{1-p_1}t}{\lambda^{p_1} n}
            + \frac{3c_4a_2^{2p_2}L^{2p_2}B^{1-p_2}t}{\lambda^{p_2} n} \\
            &+ \left(\frac{1872V t}{n}\right)^{\frac{1}{2-\tau}}
            + \frac{900Bt}{n} + \frac{456B_0t}{n} + \frac{3c_5 t}{n}.
        \end{split}
    \end{equation}
    where $c_i, i = 0,1,2,3,4,5$ are constants independent of $\lambda$, $n$ or $t$ and explicitly given in the proof. 
\end{theorem}
Next, we apply the oracle inequality to the margin-based loss function 
$\phi(y,y',t) = \psi(\mathrm{sgn}(y-y')t)$ with $\psi: \mathbb{R} \to [0,\infty)$, where $\psi$ satisfies the following assumption.

\begin{assumption}\label{admissibleloss}
 The univariate function $\psi: \mathbb{R} \to [0,\infty)$  is convex, differentiable at $t = 0$ with $\psi'(0) < 0$, and the smallest zero of $\psi$ is 1.
\end{assumption} 
Examples of $\psi$ in Assumption \ref{admissibleloss} include the hinge loss $\psi_{\mathrm{hinge}}$, 
square loss $\psi_{\mathrm{square}}$ and the the $r$-norm hinge loss with $1 \leq r<\infty$ defined by $\psi_r(t):=(\psi_{\mathrm{hinge}}(t))^r$. 
Due to Lemma \ref{calibration 1} in Subsection \ref{Subsection: Calibration Inequality for Pairwise Ranking}, 
if $\psi$ satisfies Assumption \ref{admissibleloss}, the margin-based loss $\phi(y,y',t) = \psi(\mathrm{sgn}(y-y')t)$ can be truncated at $M=1$ and the Bayes $\phi$-ranking rule $f^*_\phi$ can be taken as skew-symmetric. 
Now we can give the oracle inequality for the estimators generated by algorithm \eqref{fzk} with the margin-based loss.
\begin{theorem}\label{Oracle inequality of pairwise ranking with a margin-base loss} Let ${\mathbf z} =\left\{(X_i,Y_i)\right\}_{i=1}^n$ be an i.i.d. sample of a probability distribution $P$ on 
    $\mathcal{X}\times \mathcal{Y}$. Given a margin-based loss $\phi(y,y',t) = \psi(\mathrm{sgn}(y-y')t)$ with $\psi: \mathbb{R} \to [0,\infty)$ satisfying Assumption \ref{admissibleloss} 
    and a measurable positive semi-definite kernel $K: \mathcal{X}^2 \times \mathcal{X}^2 \to \mathbb{R}$ with the associated RKHS $\mathcal{H}_K$ 
    consisting of skew-symmetric functions, the estimator $f_{\mathbf z}$ is defined by \eqref{fzk} with $\lambda>0$. 
    Assume that there exist constants $L>0$, $B > 0, \tau \in [0,1]$ and $V \geq B^{2-\tau}$ such that $\psi$ is locally $L$-Lipschitz and
    $\psi(t) \leq B$ over $t \in [-1,1]$, and the variance bound \eqref{variance bound} of Assumption \ref{assumption 1} holds with truncation parameter $M=1$. 
    If there exists some $f_0 \in \mathcal{H}_K$ satisfying $\|\phi_{f_0}\|_\infty \leq B_0$ for some constant $B_0 \geq 0$ and $\mathcal{H}_K$ satisfies Assumption \ref{Capacity condition} with 
    $0 < p_1, p_2 < 1/2$, $ a_1 \geq B$ and $a_2 \geq B$, 
    then the conclusion of Theorem \ref{Oracle inequality of pairwise ranking with a general loss} holds true.
\end{theorem}
\begin{remark}\label{a general learning rate}
    Under the assumptions of Theorem \ref{Oracle inequality of pairwise ranking with a margin-base loss}, we further assume that $p_1 = p_2 = p$ and there exist constants $c > 0$ and $\gamma \in (0, 1]$ such that,
    for all $\lambda > 0$, the approximation error $\mathcal{A}(\lambda)$ is bounded as  
    \[\mathcal{A}(\lambda):=\inf_{f \in \mathcal{H}_{K}}\lambda\|f\|^2_{K} + \mathcal{R}^\phi(f) - \mathcal{R}^\phi(f^*_\phi) \leq c\lambda^{\gamma}.\]
    Then by choosing 
    \[\lambda = n^{-b}, \ b =\frac{1}{(2 - \tau - p + p\tau)\gamma + p},\]
    we have 
   \[\mathcal{R}^{\phi}(\pi(f_\mathbf{z})) -  \mathcal{R}^{\phi}(f^*_\phi) \lesssim n^{-\frac{\gamma}{(2- \tau - p + p\tau)\gamma + p}}.\]
    This learning rate significantly improves the existing results for pairwise ranking with margin-based losses. See Section \ref{Section: Discussions on Related Work} for a detailed comparison.
\end{remark}

We will prove Theorem \ref{Oracle inequality of pairwise ranking with a general loss} and 
Theorem \ref{Oracle inequality of pairwise ranking with a margin-base loss} in Section \ref{Section: Proofs of the General Oracle Inequalities}. 
Having derived the oracle inequality, we now focus on the pairwise ranking scheme \eqref{fz} with pairwise Gaussian kernel $K^\sigma$ and two typical margin-base losses: 
\[\phi_{\mathrm{hinge}}(y,y',t) := \psi_{\mathrm{hinge}}(\mathrm{sgn}(y-y') t) = \max\{ 1 - \mathrm{sgn}(y-y') t, 0\}\] and 
\[\phi_{\mathrm{square}}(y,y',t) := \psi_{\mathrm{square}}(\mathrm{sgn}(y-y') t) = (1-\mathrm{sgn}(y-y') t)^2.\]
Recall that the excess ranking risk $\mathcal{E}(f)$ and the excess $\phi$-ranking risk $\mathcal{E}^{\phi}(f)$ are defined by \eqref{excessrankingrisk} and \eqref{excessphirankingrisk}, respectively. 
To derive learning rates for excess ranking risk using Theorem \ref{Oracle inequality of pairwise ranking with a margin-base loss}, the following problems remain to be solved:
\begin{itemize}
    \item Verify Assumption \ref{assumption 1} for $\phi_{\mathrm{hinge}}$, $\phi_{\mathrm{square}}$;
    
    \item Construct an $f_0 \in \mathcal{H}_{K^\sigma}$ to bound the approximation error $\mathcal{A}(\lambda)$;
    
    \item Verify the capacity condition (Assumption \ref{Capacity condition}) for the RKHS $\mathcal{H}_{K^\sigma}$;
    
    \item Establish calibration inequality for ranking, i.e., bounding the excess ranking risk
    $\mathcal{E}(f)$ by the excess $\phi$-ranking risk $\mathcal{E}^{\phi}(f)$.
\end{itemize} 
In Section \ref{Section: Deriving the Learning Rates for Gaussian Ranking Estimators}, we will solve these problems and establish the corresponding estimates under proper regularity assumptions. 
In the rest part of this section, We introduce these assumptions to characterize the distribution, the smoothness of the target function and the intrinsic dimension of the input domain. 
Then we present our results on fast learning rates of Gaussian ranking estimator \eqref{fz} equipped with  $\phi_{\mathrm{hinge}}$ and $\phi_{\mathrm{square}}$.

We first introduce two noise conditions to describe the distribution. Recall that $P$ is a probability distribution on $\mathcal{X}\times \mathcal{Y}$. 
For distinct $(x,y)$ and $(x',y')$, define the posterior probabilities of $P$, denoted by $\eta_+, \eta_-, \eta_=: \mathcal{X}^2 \to [0,1]$, as 
\begin{align*}
    \eta_+(x,x') &:= P(y >  y' \vert x,x'), \\
    \eta_-(x,x') &:= P(y < y' \vert x,x'), \\
    \eta_=(x,x') &:= P(y = y' \vert x,x').
\end{align*}
Then partition $\mathcal{X}^2$ into 
\begin{equation}\label{seperation}
\begin{split}
\mathcal{X}^2_+ &:= \{(x,x') \in \mathcal{X}^2: \eta_+(x,x') > \eta_-(x,x')\}, \\ 
    \mathcal{X}^2_- &:= \{(x,x') \in \mathcal{X}^2: \eta_+(x,x') < \eta_-(x,x')\}, \\
    \mathcal{X}^2_= &:= \{(x,x') \in \mathcal{X}^2: \eta_+(x,x') = \eta_-(x,x')\}.
\end{split}
\end{equation}

Following Proposition 1 of \cite{Clemencon2008Ranking}, the \emph{Bayes ranking rule} $f^{*}_{\mathrm{rank}}: \mathcal{X}^2 \to \mathbb{R}$ which minimizes the ranking risk 
$\mathcal{R}(f)$ in \eqref{rankingrisk} over all measurable functions is given by
\[f^{*}_{\mathrm{rank}}(x,x') = \mathrm{sgn}\bigl(\eta_+(x,x') - \eta_-(x,x')\bigr) =
\begin{cases}
    1 &\mbox{ if } (x,x') \in \mathcal{X}^2_+,  \\
    0 & \mbox{ if } (x,x') \in \mathcal{X}^2_=,\\
    -1 &\mbox{ if }(x,x') \in \mathcal{X}^2_-.
\end{cases}\] Then the minimal ranking risk is $\mathcal{R}(f^{*}_{\mathrm{rank}})=\mathbb{E}\min\left\{\eta_+(X,X'),\eta_-(X,X')\right\}$. 

Now we introduce a condition analogous to Tsybakov's noise condition in binary classification, describing the amount of noise in the labels. In order to motivate this condition, we note by 
\begin{equation*}
        \min\left\{\eta_+(x,x'),\eta_-(x,x')\right\} =\frac{\eta_+(x,x')+\eta_-(x,x')}{2}-\frac{\bigl\vert\eta_+(x,x')-\eta_-(x,x')\bigr\vert}{2}
\end{equation*} 
that the function $\bigl\vert\eta_+(x,x') - \eta_-(x,x')\bigr\vert$ can be used to describe the noise in comparing the labels of a distribution $P$. 
Indeed, in regions where this function is close to $1$, there is only a small amount of noise in comparing $y$ and $y'$, 
whereas function values close to $0$ only occur in regions with a high level of noise in comparing $y$ and $y'$. 
The following condition in which we use the convention $t^{\infty}:=0$ for $t \in(0,1)$ describes the size of the latter regions. 
Recall that $P_{\mathcal{X}}$ is the marginal distribution of $P$ on $\mathcal{X}$. Let $P^2_{\mathcal{X}}:=P_{\mathcal{X}}\otimes P_{\mathcal{X}}$.
\begin{assumption}\label{Tsybakov's noise condition}
    There exist constants $C_*>0$ and $q \in [0,\infty]$ such that 
    \begin{equation}\label{noise 1}
        P^2_{\mathcal{X}}\bigl(\bigl\{ (x,x') \in \mathcal{X}^2 :
        \bigl\vert\eta_+(x,x') - \eta_-(x,x')\bigr\vert \leq t\bigr\}\bigr) \leq C_*t^q
    \end{equation}
    for all $t > 0$. 
\end{assumption} Obviously, $P^2_{\mathcal{X}}$ has the noise exponent $q>0$ if and only if $\bigl\vert\eta_+(x,x') - \eta_-(x,x')\bigr\vert^{-1}$ belongs to the Lorentz space 
$\mathcal{L}_{q,\infty}(P^2_{\mathcal{X}})$ (see \cite{Lorentz1966Approximation}). 
It is also easy to see that all distributions satisfy \eqref{noise 1} with $C_*=1$ and $q=0$. 
Furthermore, the extreme case $q = \infty$ requires that $\bigl\vert\eta_+(x,x') - \eta_-(x,x')\bigr\vert$ is lower bounded by some positive constant for almost every $(x, x') \in \mathcal{X}^2$, 
including the case of a noise-free distribution in which $\bigl\vert\eta_+(x,x') - \eta_-(x,x')\bigr\vert = 1$ almost surely. Under Assumption \ref{Tsybakov's noise condition}, 
we will demonstrate the variance bound \eqref{variance bound} for $\phi_{\mathrm{hinge}}$. 
We also leverage this assumption to establish a refined calibration inequality to bound $\mathcal{E}(f)$ by $\mathcal{E}^{\phi}(f)$. 
See Proposition \ref{improved calibration inequality} and Lemma \ref{hinge loss variance bound}. 
The work of \cite{Clemencon2008Ranking} and \cite{Clemencon2011Minimax} introduced another two noise conditions for bipartite ranking problems. 
Their noise conditions are stronger than ours which leads to a better variance bound exponent $\tau$ in \eqref{variance bound}. 
We will make detailed comparisons of these noise conditions in Subsection \ref{Subsection: Comparison of Noise Conditions}.

The second noise condition is a geometric condition for distributions that describes the location of the noise sets and allows us to estimate the approximation error for Gaussian kernels. 
According to our discussion in Subsection \ref{Subsection: Calibration Inequality for Pairwise Ranking}, the Bayes ranking rule of $\phi_{\mathrm{hinge}}$ can be defined pointwisely, i.e., given $(x,x')\in \mathcal{X}^2$,
\[f^*_{\mathrm{hinge}}(x,x'):=\mathop{\arg\min}_{t\in \mathbb{R}} \eta_+(x,x')\psi_{\mathrm{hinge}}(t) + \eta_-(x,x')\psi_{\mathrm{hinge}}(-t).\] Then $f^*_{\mathrm{hinge}}$ can be explicitly given by 
\begin{equation}\label{fhinge}
    f^*_{\mathrm{hinge}}(x,x') = \mathrm{sgn}\bigl(\eta_+(x,x') - \eta_-(x,x')\bigr),\ \forall (x,x')\in \mathcal{X}^2,
\end{equation}
which is exactly the optimal ranking rule $f^{*}_{\mathrm{rank}}$. Note that $f^*_{\mathrm{hinge}}$ is typically a step function. 
Therefore, employing the classical smoothness assumption to describe the regularity of $f^*_{\mathrm{hinge}}$ seems rather restrictive. 
Inspired by the theoretical analysis of binary classification with hinge loss (cf. \cite{Steinwart2008Support}), 
we introduce the second noise condition which is called the margin-noise condition. 
It relates the noise to the distance to the decision boundary and is applied to bound the approximation error with respect to hinge loss and Gaussian kernels.
See Proposition \ref{hinge loss approximation}.
\begin{assumption}\label{Margin-noise condition}
    There exist constant $C_{**}>0$ and $\beta \geq 0$ such that for all $t \geq 0$, 
    \begin{equation}\label{noise 2}
        \int_{\{(x,x') \in \mathcal{X}^2 : \Delta(x,x') < t\}} \vert\eta_+(x,x') - \eta_-(x,x')\vert dP_\mathcal{X}(x)dP_\mathcal{X}(x') 
        \leq C_{**}t^\beta.
    \end{equation}
    Here, $\Delta(x,x')$ is the distance to the decision boundary, defined as 
    \begin{equation}\label{Delta}
    \Delta(x,x') := 
    \begin{cases}
        \mathrm{dist}((x,x'), \mathcal{X}^2_- \cup \mathcal{X}^2_= ) & (x,x') \in \mathcal{X}^2_+  \\
        \mathrm{dist}((x,x'), \mathcal{X}^2_+ \cup \mathcal{X}^2_= ) & (x,x') \in \mathcal{X}^2_-  \\
        0 & (x,x') \in \mathcal{X}^2_=,
    \end{cases}
    \end{equation}
    where $\text{dist}(x, A) := \inf_{y \in A}\|x-y\|_2.$
\end{assumption}
 
Note that in condition \eqref{noise 2} we neither impose any kind of smoothness assumption nor require that $P^2_{\mathcal{X}}$ is absolutely continuous with respect to the Lebesgue measure. 
If $P^2_{\mathcal{X}}$ has only a low concentration near the decision boundary $\mathcal{X}^2_=$ or it is particularly noisy in this region, 
Assumption \ref{Margin-noise condition} is satisfied for a large exponent $\beta$. 
For instance, we can select arbitrary large values for $\beta$ in the extreme case where $\mathcal{X}^2_+$ and $\mathcal{X}^2_-$ have positive distance.

Next, we introduce the box-counting dimension condition which takes full advantage of the low-dimensional intrinsic structure of the input data. 
Assumption \ref{Capacity condition} requires upper bounds of the entropy number of RKHS $\mathcal{H}_{K}$ with respect to $\mathcal{L}_2$-seminorm 
$\| \cdot \|_{\mathcal{L}_2(P^n_{\mathbf{x}} \otimes P_{\mathcal{X}})}$
and $\| \cdot \|_{\mathcal{L}_2(P^n_{\mathbf{x}^2})}$, both of which are dominated by the sup-norm $\|\cdot\|_\infty$.
The covering number of Gaussian RKHS $\mathcal{H}_{K^\sigma}$ with respect to the sup-norm has been intensively studied in the literature, cf. \cite{Kuhn2011Coveringa, Steinwart2021closer}, 
some of which have related it to the uniform covering number $\mathcal{N}(\mathcal{X},\| \cdot \|_\infty, \sigma)$ of the input space. 
Since $\mathcal{N}(\mathcal{X},\| \cdot \|_\infty, \sigma)$ exponentially depends on the dimension of the input space $\mathcal{X} \subset \mathbb{R}^d$. 
When $d$ is large, the learning rate suffers from the curse of dimensionality. This phenomenon is usually inevitable for high-dimensional data. 
However, we can derive much faster rates of convergence when the intrinsic dimension of the data is much smaller than the dimension of its ambient space. 
We describe the low-dimensional intrinsic structure of the input domain through the box-counting dimension which is introduced to describe the low-dimensional intrinsic fractal structure in fractal geometry, 
see, e.g., \cite{Falconer2004Fractal}. 
For a metric space $(\mathcal{T},\mathrm{d})$ and a subset $\mathcal{S} \subset \mathcal{T}$, if the limit 
    \[\lim_{\varepsilon \to 0} \frac{\log\mathcal{N}(\mathcal{S},\mathrm{d},\varepsilon)}{\log\frac{1}{\varepsilon}}\]
    exists, it is called the \emph{box-counting dimension} of $\mathcal{S}$. 
    If not, we define the upper box-counting dimension
    \[\limsup_{\varepsilon \to 0} \frac{\log\mathcal{N}(\mathcal{S},\mathrm{d},\varepsilon)}{\log\frac{1}{\varepsilon}},\]
    and lower box-counting dimension
    \[\liminf_{\varepsilon \to 0} \frac{\log\mathcal{N}(\mathcal{S},\mathrm{d},\varepsilon)}{\log\frac{1}{\varepsilon}}.\]
\begin{assumption}\label{Upper box-counting dimension condition}
    There exists a constant $\varrho > 0$ such that
    \begin{equation}\label{low 1}
        \limsup_{\varepsilon \to 0} 
        \frac{\log \mathcal{N}(\mathcal{X}, \|\cdot\|_\infty, \varepsilon)}{\log \frac{1}{\varepsilon}}
        \leq \varrho,
    \end{equation}
    or, equivalently, there exist constants $C_\mathcal{X} \geq 1$ and $\varrho > 0$ such that 
    \begin{equation}\label{low 2}
        \mathcal{N}(\mathcal{X}, \|\cdot\|_\infty, \varepsilon) \leq C_\mathcal{X} \varepsilon^{-\varrho}.
    \end{equation}
\end{assumption}
The infimum over all $\varrho$ satisfying \eqref{low 1} coincides with the upper box-counting
dimension of $\mathcal{X}$. This assumption is rather general because it captures the intrinsic dimension
of low-dimensional smooth submanifold $\mathcal{M} \subset \mathbb{R}^d$. When a bounded $\mathcal{X} \subset \mathbb{R}^d$ has a non-empty interior, the assumption is fulfilled precisely for $\varrho = d$.
Under Assumption \ref{Upper box-counting dimension condition}, we will show that the learning rate of the concerned algorithms  mainly depends on 
the intrinsic dimension $\varrho$ of the input domain.

\begin{remark}\label{remark2}
    In Subsection \ref{Subsection: Bounding the Approximation Error}, we define a convolution operator $\mathcal{K}^{\sigma} * \cdot : \mathcal{L}_2(\mathbb{R}^{2d}) \to \mathcal{H}_{K^{\sigma}}$
    and show that acting the convolution operator on the Bayes $\phi$-ranking rule $f^*_{\phi}$ can lead to the desired approximator. 
    Hence, a proper extension of $f^*_{\phi}: \mathcal{X}^2 \to \mathbb{R}$ to $\mathcal{L}_2(\mathbb{R}^{2d})$ should be taken into consideration. 
    In many cases, it is enough to use the zero extension. However, If $\mathcal{X}$ has a low intrinsic dimension $\varrho < d$, 
    $\mathcal{X}$ may have a zero Lebesgue measure and the trivial zero extension of $f^*_{\phi}$ only yields the zero function in $\mathcal{L}_2(\mathbb{R}^{2d})$. 
    For the case of square loss, we can fix this issue with the help of Whitney’s extension theorem (see Theorem 2.3.6 of \cite{Hormander2015Analysis}), 
    in which we directly regard $f^*_{\mathrm{square}}$ (defined in \eqref{fsquare}) as a function in $\mathcal{L}_2(\mathbb{R}^{2d})$, 
    and impose Besov smoothness assumption  on $f^*_{\mathrm{square}}$.  As is explained in \cite{Hamm2021Adaptive}: if $\mathcal{X}^2$ is a compact $C^k$-manifold, 
    by Whitney’s extension theorem, any $f \in C^k(\mathcal{X}^2)$ has an extension to a function $\widetilde{f} \in C^k(\mathcal{X}^2_{+\delta})$
    where $\mathcal{X}^2_{+\delta}$ is the $\delta$-neighbourhood of $\mathcal{X}^2$ in $\mathbb{R}^{2d}$. 
    However, for the case of hinge loss, the situation is more challenging because we need to extend $f^*_{\mathrm{hinge}}$
    to $\mathcal{L}_2(\mathbb{R}^{2d})$ without violating the margin-noise condition \eqref{Margin-noise condition} which is necessary to bound the approximation error. 
    There is no trivial way to fix this issue to the best of our knowledge. In this work, we will give a construction of this extension in the proof of Proposition \ref{hinge loss approximation}. 
\end{remark}

Under the noise conditions, i.e., Assumption \ref{Tsybakov's noise condition} and Assumption \ref{Margin-noise condition}, and Assumption \ref{Upper box-counting dimension condition}, 
we derive fast learning rates for of Gaussian ranking estimator \eqref{fz} equipped with  $\phi_{\mathrm{hinge}}$. Let $\mathcal{B}_{\mathbb{R}^d}$ denote the unit ball in $\mathbb{R}^d$.

\begin{theorem}\label{Oracle inequality of hinge loss}
    Let ${\mathbf z} =\left\{(X_i,Y_i)\right\}_{i=1}^n$ be an i.i.d. sample of a probability distribution $P$ on 
    $\mathcal{X} \times \mathcal{Y}$. The estimator $f_\mathbf{z}$ is defined by \eqref{fz} with $\lambda > 0, \sigma \in (0, 1)$, and  $\phi_{\mathrm{hinge}}(y, y', t) = \max\{1-\mathrm{sgn}(y - y')t, 0\}$. 
    Further, let $P$ satisfy Assumption \ref{Tsybakov's noise condition} for constant $C_*$ and $q$ and Assumption \ref{Margin-noise condition} for constants $C_{**}$ and $\beta$. 
    Assume that there exits some $r>0$ such that $\mathcal{X} \subset r\mathcal{B}_{\mathbb{R}^d}$ and  Assumption \ref{Upper box-counting dimension condition} is satisfied for constants $C_{\mathcal{X}}$ and $\varrho\in (0,d]$. 
    Then for all  $n \geq 2, t>0$, and $p \in (0,  1/4]$, 
    with probability at least $1-(c_0 + 5)\exp(-t)$, 
    there holds
    \begin{equation}
        \begin{split}
            &\lambda\|f_\mathbf{z}\|^2_{K^\sigma} + \mathcal{R}^{\phi_{\mathrm{hinge}}}(\pi(f_\mathbf{z})) - \mathcal{R}^{\phi_{\mathrm{hinge}}}(f^*_{\mathrm{hinge}}) \\
            \leq{}& \frac{2^{3d+4}r^{2d}}{\Gamma(d)}\lambda \sigma^{-2d} 
            + \frac{2^{\beta/2+4}C_{**}\Gamma(d + \frac{\beta}{2})}{\Gamma(d)}\sigma^\beta 
            + 36c_6 \left(\frac{C^*_\mathcal{X}  C_*^{\frac{1-p}{q+1}}}
            {\lambda^{p} p^{2d+1} \sigma^{2\varrho} n}\right)^{\frac{q+1}{q-p+2}} \\
            & + \frac{12C^*_\mathcal{X}c_6(t + 1)}{\lambda^{p} p^{2d+1} \sigma^{2\varrho} n}
            + \left(\frac{11232C_*^{\frac{1}{q+1}} t}{n}\right)^{\frac{q+1}{q+2}}
            + \frac{(2712+ 3c_6)t}{n},
        \end{split}
    \end{equation}
    where \[C^*_{\mathcal{X}}:= 12C^2_{\mathcal{X}}\binom{4e + 2d}{2d}\frac{(2d+1)^{2d+1}}{2^{2d+1}e^{4d+1}}\]
    and $c_0, c_6$ are constants independent of $\lambda, \sigma, n, t$ or $p$ which are explicitly given in the proof. 
    In particular,  choose 
    \begin{align*}
        &\sigma = n^{-a}, \quad a = \frac{q+1}{\beta(q+2) + 2\varrho(q+1)};\\
        &\lambda = n^{-b}, \quad b \geq \frac{(2d+\beta)(q+1)}{\beta(q+2) + 2\varrho(q+1)}, \quad
        p = \frac{\log 2}{4 \log n}.
    \end{align*}
    Then for all $n \geq 2$ and $t \geq 1$, we have 
    \[\mathcal{E}(\pi(f_\mathbf{z})) \leq 
    \mathcal{E}^{\phi_{\mathrm{hinge}}}(\pi(f_\mathbf{z}))
    \lesssim t n^{-\frac{\beta (q+1)}{\beta(q+2) + 2\varrho(q+1)}} \log^{2d+1}n\]
    with probability at least $1-(c_0 + 5)\exp(-t)$.
\end{theorem}  

For the case of square loss, the Bayes ranking rule of $\phi_{\mathrm{square}}$ is defined pointwise by
\[f^*_{\mathrm{square}}(x,x'):=\mathop{\arg\min}_{t\in \mathbb{R}} \eta_+(x,x')\psi_{\mathrm{square}}(t) + \eta_-(x,x')\psi_{\mathrm{square}}(-t)\] which can be explicitly expressed as 
\begin{equation}\label{fsquare}
    f^*_{\mathrm{square}}(x,x'):= \frac{\eta_+(x,x') - \eta_-(x,x') }{\eta_+(x,x')  + \eta_-(x,x')},\ \forall (x,x')\in \mathcal{X}^2.
\end{equation}
The function $f^*_{\mathrm{square}}$ enjoys a smoothness property inherited from $\eta_+(x,x')$ and $\eta_-(x,x')$, allowing us to derive an approximation error bound more directly. 
Here we introduce the notation of Besov smoothness which is also adopted in \cite{Kerkyacharian1992Density, Suzuki2019Adaptivity, Tsuji2021Estimation, Hamm2021Adaptive}. 
Given a function $f : \mathbb{R}^{2d} \to \mathbb{R}, h \in \mathbb{R}^{2d}$ and $s \in \mathbb{N}$, 
define the $s$-fold application of the difference operator
\[\Delta^s_hf(x) := \sum_{j=0}^s (-1)^{s-j} \binom{s}{j} f(x+jh).\]
Given a measure $\mu$ on $\mathcal{X}^2 \subset \mathbb{R}^{2d}$, define the $s$-th modulus of smoothness 
\[\omega_{s, \mathcal{L}_{2}(\mu)}(f, t) := \sup _{\|h\|_2 \leq t}\left\|\Delta_{h}^{s} f\right\|_{\mathcal{L}_{2}(\mu)}.\]
Finally, given an $\alpha > 0$ and set $s = \left\lfloor \alpha \right\rfloor + 1$, define the semi-norm
\[\vert f \vert_{\mathcal{B}_{2, \infty}^{\alpha}(\mu)} := \sup _{t>0} t^{-\alpha} \omega_{s, \mathcal{L}_{2}(\mu)}(f, t).\]

Under Assumption \ref{Tsybakov's noise condition}, Assumption \ref{Upper box-counting dimension condition} and the Besov smoothness assumption of $f^*_{\mathrm{square}}$, 
we derive fast learning rates for of Gaussian ranking estimator \eqref{fz} equipped with $\phi_{\mathrm{square}}$. 

\begin{theorem}\label{Oracle inequality of square loss}
    Let ${\mathbf z} =\left\{(X_i,Y_i)\right\}_{i=1}^n$ be an i.i.d. sample of a probability distribution $P$ on 
    $\mathcal{X} \times \mathcal{Y}$. The estimator $f_\mathbf{z}$ is defined by \eqref{fz} with $\lambda > 0, \sigma \in (0, 1)$, and $\phi_{\mathrm{square}}(y, y', t) =(1-\mathrm{sgn}(y - y')t)^2$. 
    Assume that Assumption \ref{Upper box-counting dimension condition} is satisfied for constants $C_{\mathcal{X}}$ and $\varrho\in (0,d]$. 
    Further assume that $f^*_{\mathrm{square}}\in \mathcal{L}_2(\mathbb{R}^{2d})$ and $\vert f^*_{\mathrm{square}} \vert_{\mathcal{B}_{2, \infty}^{\alpha}(P^2_{\mathcal{X}})} <\infty$ for some $\alpha>0$. 
    Then for all  $n \geq 2, t>0$, and $p \in (0,  1/4]$, with probability at least $1-(c_0 + 5)\exp(-t)$, there holds
    \begin{equation}
        \begin{split}
            &\lambda\|f_\mathbf{z}\|^2_{K^\sigma} + \mathcal{R}^{\phi_{\mathrm{square}}}(\pi(f_\mathbf{z})) - \mathcal{R}^{\phi_{\mathrm{square}}}(f^*_{\mathrm{square}}) \\
            \leq{}& \frac{2^{2s+3}\|f^*_{\mathrm{square}}\|^2_{\mathcal{L}_2(\mathbb{R}^{2d})}}{\pi^d}\lambda\sigma^{-2d} 
            + 2^{3-\alpha}\left(\frac{\Gamma\left(d + \frac{\alpha}{2}\right)}{\Gamma(d)}\right)^2
            \vert f^*_{\mathrm{square}} \vert^2_{\mathcal{B}_{2, \infty}^{\alpha}(P^2_\mathcal{X})} \sigma^{2\alpha} \\
            &+ \frac{C^*_\mathcal{X}(96c_6 + 72c_6t)}{\lambda^{p}p^{2d+1}\sigma^{2\varrho} n}
            + \frac{(33552 + 1824 \cdot 2^{2s} + 3c_6) t}{n},
        \end{split}
    \end{equation}
    where \[C^*_{\mathcal{X}}:= 12C^2_{\mathcal{X}}\binom{4e + 2d}{2d}\frac{(2d+1)^{2d+1}}{2^{2d+1}e^{4d+1}}\]
    and $c_0, c_6$ are constants independent of $\lambda, \sigma, n, t$ or $p$ which are explicitly given in the proof. In particular, choose 
    \begin{align*}
        &\sigma = n^{-a}, \quad a = \frac{1}{2\alpha + 2\varrho};\\
        &\lambda = n^{-b}, \quad b \geq \frac{\alpha + d}{\alpha + \varrho}, \quad p = \frac{\log 2}{4 \log n}.
    \end{align*}
    Then for all $n \geq 2, t \geq 1$, we have 
    \[\mathcal{E}(\pi(f_\mathbf{z})) \lesssim 
    \sqrt{\mathcal{E}^{\phi_{\mathrm{square}}}(\pi(f_\mathbf{z}))}
    \lesssim \sqrt{t} n^{-\frac{\alpha}{2(\alpha + \varrho)}} \log^{d+\frac{1}{2}}n\] with probability at least $1-(c_0 + 5)\exp(-t)$. 
    Furthermore, if $P$ additionally satisfies Assumption \ref{Tsybakov's noise condition} with $q>0$, 
    then with probability at least $1-(c_0 + 5)\exp(-t)$, there holds
    \[\mathcal{E}(\pi(f_\mathbf{z})) \lesssim 
    t^{\frac{q+1}{q+2}} n^{-\frac{(q+1)\alpha}{(q+2)(\alpha + \varrho)}}
    \log^{\frac{(q+1)(2d+1)}{q+2}}n.\]
\end{theorem}

Finally, we would like to point out that, in light of Proposition 3.2 in \cite{Hamm2021Adaptive}, 
the class of functions $f$ satisfying $\vert f \vert_{\mathcal{B}_{2, \infty}^{\alpha}(P^2_{\mathcal{X}})} <\infty$ is indeed very rich, 
even when the upper box-counting dimension of $\mathcal{X}$ is strictly less than $d$. For instance, if $f$ is the H\"{o}lder $\alpha$-continuous function on $\mathcal{X}^2$, 
then $\vert f \vert_{\mathcal{B}_{2, \infty}^{\alpha}(\mu)} <\infty$ for every measure $\mu$ with $\mbox{supp} \mu \subset \mathcal{X}^2$. 
The proofs of Theorem \ref{Oracle inequality of hinge loss} and Theorem \ref{Oracle inequality of square loss} are postponed to Section \ref{Section: Deriving the Learning Rates for Gaussian Ranking Estimators} 
after establishing some preliminary results.

\section{Discussions on Related Work}\label{Section: Discussions on Related Work}

In this section, we compare our convergence analysis with some existing results in the literature. Our paper contributes to several rapidly growing literature, 
we give only those citations of particular relevance. More references can be found within these works. 

In this paper, we consider the pairwise ranking model which aims to learn a bivariate ranking rule $f: \mathcal{X} \times \mathcal{X} \to \mathbb{R}$. 
Our model is more general compared with the score-based ranking model, in which we learn a scoring function $s: \mathcal{X} \to \mathbb{R}$ and construct the ranking rule as $f(x,x') = s(x) - s(x')$.
Some popular ranking algorithms including RankSVM in \cite{Joachims2002Optimizing}, RankNet in \cite{Burges2005Learning} and 
RankRLS in \cite{Pahikkala2007Learning, Cortes2007Magnitude-preserving, Pahikkala2009efficient} are closely related to the score-based ranking model. 
For the error analysis of the score-based ranking model, \cite{Agarwal2009Generalization} proved generalization bounds via algorithmic stability for the bipartite ranking problem where the label $Y$ is binary. 
\cite{Chen2012convergence} derived a capacity-independent learning rate for RankRLS from the viewpoint of operator approximation.  
Whereafter, \cite{Zhao2017Learning} developed a capacity-dependent generalization analysis for RankRLS by virtue of covering numbers and Hoeffding's decomposition for U-statistics. 

In some learning tasks, for example, learning binary relations between two objects, 
we can assign a vertex to each object and an edge to represent the binary relations between two vertices, which naturally forms a graph structure.
Denote $\mathcal{X}$ the vertex set and $\mathcal{X}^2$ the edge set.
A positive semi-definite kernel $G: \mathcal{X} \times \mathcal{X} \to \mathbb{R}$ induces an RKHS $\mathcal{H}_{G}$ in which we can learn a function $f: \mathcal{X} \to \mathbb{R}$ defined on the vertex set.
However, in such graph learning tasks, if we want to figure out the binary relations, pairwise learning is more applicable, in which we need to learn a function 
$f: \mathcal{X}^2 \to \mathbb{R}$ defined on the edge set.
In fact, the kernel function $G$ defined between vertices can induce the so-called Kronecker product pairwise kernel $K: \mathcal{X}^2 \times \mathcal{X}^2 \to \mathbb{R}$
by Kronecker product $K((x,x'),(u',u')):= G(x,u)G(x',u')$. 
Kronecker kernel ridge regression (KKRR) is a pairwise learning algorithm based on the Kronecker product pairwise kernel $K$, where the induced regularized empirical  minimization problem is solved in the 
RKHS $\mathcal{H}_{K}$. 
In Subsection \ref{Subsection: Entropy Number Estimate of Pairwise Gaussian Kernel Spaces}, 
we see that the RKHS $\mathcal{H}_{K}$ consists of functions on $\mathcal{X}^2$ and we can decompose any $f \in \mathcal{H}_{K}$ into a direct sum of its symmetric part and skew-symmetric part,  
which corresponds to a direct sum of two subspaces of $\mathcal{H}_{K}$ and leads to a decomposition of the pairwise kernel $K$ as well. 
The pairwise Gaussian kernel $K^\sigma$ is indeed the skew-symmetric part of the traditional Gaussian kernel $\widetilde{K}^\sigma$ on $\mathcal{X}^2$ defined as 
$\widetilde{K}^{\sigma}((x,x'),(u,u')) := \exp\left(-\|(x,x') - (u,u')\|^2_2/\sigma^2\right)$,
and hence the corresponding RKHS $\mathcal{H}_{K^\sigma}$ is one subspace in the direct sum of $\mathcal{H}_{\widetilde{K}^\sigma}$.
Although $\mathcal{H}_{K^\sigma}$ is a smaller hypothesis space, 
the universality of the kernel can be still maintained when only 
 learning skew-symmetric relations. That is to say, if the Kronecker product pairwise kernel $K$ can approximate any continuous function on $\mathcal{X}^2$ arbitrarily well, 
then the skew-symmetric part of the pairwise kernel can also approximate any continuous skew-symmetric function arbitrarily well, cf. Theorem III.4. of \cite{Waegeman2012Kernel-Based}. 
Therefore, when we use pairwise learning algorithms including KKRR to learn skew-symmetric relations between data, we only need to restrict the estimators in the hypothesis spaces $\mathcal{H}_{K^\sigma}$. 
A similar observation also holds when using the symmetric part of the kernel to learn symmetric relations. 
One can refer to \cite{Ben-Hur2005Kernel, Waegeman2012Kernel-Based, Pahikkala2013Efficient, Wang2021Regression} for more details.

As the pairwise ranking model is more general, some theoretical frameworks have been established to analyze its convergence behaviors. 
The main distinction in convergence analysis between ranking and classification or regression is that in ranking problems, 
the stochastic part of the convergence analysis involves a second-order $U$-process rather than summations of i.i.d. random variables. 
\cite{Clemencon2008Ranking} and \cite{Rejchel2012Ranking} used Hoeffding's decomposition to the unregularized ranking algorithm and obtained generalization bounds with faster rates than $\frac{1}{\sqrt{n}}$. 
Hoeffding's decomposition breaks the sample error of a ranking algorithm into an empirical term based on a sum of i.i.d. random variables and a degenerate U-process. 
In our work, we make full use of Talagrand's inequality, local Rademacher analysis, 
and capacity information of the hypothesis space to derive tight bounds on the empirical term as well as the degenerate part (see Section \ref{Section: Proofs of the General Oracle Inequalities}). 
For regularized pairwise ranking, \cite{Chen2014Norm} considered $\ell_1$-norm regularized SVM ranking and established a learning rate under a similar noise condition as Assumption \ref{Tsybakov's noise condition}. 
\cite{Rejchel2017Oracle} proved an oracle inequality for parametric pairwise ranking with the Lasso penalty in high-dimensional settings.  
In contrast with our work, the work of \cite{Chen2014Learning, Lei2018Generalization, Wang2022Error} 
considered regularized pairwise ranking with a fixed kernel and derive learning rates by assuming a $\gamma$-decaying rate of the approximation error. 
We further emphasize the differences between our results and the work of  \cite{Clemencon2008Ranking, Rejchel2012Ranking, Chen2014Learning, Lei2018Generalization, Wang2022Error}.
\cite{Clemencon2008Ranking} formulated the question if one can get generalization bounds with fast rates for the excess risk in ranking. 
They gave a positive answer for unregularized ranking estimators  with non-convex $0-1$ loss, i.e., $\phi(y, y', t) = \psi_{\mathrm{0-1}}(\mathrm{sgn}(y-y')t)$ where $\psi_{\mathrm{0-1}}(t):=\mathbb{I}_{[0, \infty)}(t)$, 
and left the case of convex risk minimization to future study. 
The work of \cite{Rejchel2012Ranking} advanced this line of research, in which it established generalization bounds with better rates than $\frac{1}{\sqrt{n}}$ for the excess ranking risk with convex margin-based losses. 
Both of these two papers constructed ranking estimators in a general hypothesis space and required the optimal ranking rule to reside in this function space. 
Hence they did not need to consider the analysis of the approximation error. Moreover, a noise condition was also proposed by \cite{Clemencon2008Ranking}, 
we postpone the comparison of different noise conditions to Subsection \ref{Subsection: Comparison of Noise Conditions}. 
The work of \cite{Chen2014Learning, Wang2022Error} also developed capacity-dependent empirical process technique to analyze regularized ranking estimators, 
in which the capacity condition is based on the sup-norm $\|\cdot\|_\infty$, which is more restrictive than the empirical $\mathcal{L}_2$-seminorm adopted in this paper. 
Moreover, \cite{Chen2014Learning} only focused on hinge loss and derived a rate of $n^{-\frac{\gamma}{(2 - \tau + p)\gamma + p}}$ 
which is slower than the rate of $n^{-\frac{\gamma}{(2- \tau - p + p\tau)\gamma + p}}$ derived in Remark \ref{a general learning rate}.  
The work of \cite{Lei2018Generalization} considered a general loss function similar to ours and derived a rate of $n^{-\frac{\gamma}{\gamma + 1}}$ 
without any capacity condition of the hypothesis spaces nor Bernstein conditions on the bias and variance. 
We note that even for the case $\tau = 0$ in which the variance bound \eqref{variance bound} is trivially satisfied, our rate given by $n^{-\frac{\gamma}{(2 - p)\gamma + p}}$ 
is still faster than $n^{-\frac{\gamma}{\gamma + 1}}$ provided $\gamma < 1$. \cite{Wang2022Error} considered the least square loss $\phi_{\mathrm{ls}}(y,y',t) := ( y- y' - t)^2$ in a regression setting. 
We can apply our analysis developed for the square loss to derive the rate of $n^{-\frac{\gamma}{\gamma + p}}$ for the case of $\phi_{\mathrm{ls}}(y, y', t)$ loss, 
which is much faster than $n^{-\frac{2\gamma}{(p+1)\gamma + 4}}$ obtained by \cite{Wang2022Error}. We also note that the rate $n^{-\frac{\gamma}{\gamma + p}}$
actually achieves the well-known minimax lower rate derived by \cite{Stone1982Optimal} when approximating functions on $\mathbb{R}^{2d}$. 
From this point of view, the oracle inequalities established in Theorem \ref{Oracle inequality of pairwise ranking with a general loss} and 
Theorem \ref{Oracle inequality of pairwise ranking with a margin-base loss} are more general than existing results but can lead to more tight bounds on excess risk. 
Besides, \cite{Smale2003ESTIMATING} shows that the $\gamma$-decaying assumption of the approximation error is indeed a very restrictive assumption when describing the approximation ability of Gaussian RKHS 
and hence the setting of ranking with varying Gaussian kernels can not be handled by the approaches developed in the work of \cite{Chen2014Learning, Lei2018Generalization, Wang2022Error}, 
in which they only focused on the regularized ranking with fixed pairwise kernels.

The intrinsic dimension of the data can be utilized to improve the dependence on the dimension of learning rates. To describe the intrinsic dimension structure, 
the probably most popular notion is to assume that $\mathcal{X} \subset \mathbb{R}^d$ is a low-dimensional submanifold, see, e.g., \cite{Yang2016Bayesian, Ye2008Learning, Ye2009SVM}. 
The more general notion adopted in our paper is based on the box-counting dimension which considerably generalizes the manifold assumption. 
To the best of our knowledge,  our paper is the first to consider the learning behavior of regularized Gaussian ranking estimators 
under the assumption that the gap between the intrinsic dimension of the data and the dimension of its ambient space is large. 

To sum up, in this paper we consider a regularized pairwise ranking problem with a general convex loss
under substantially general assumptions. The oracle inequality established in our work can lead to an elegant framework of convergence analysis which significantly improves the learning rates of existing work. 
It also enables us to derive fast learning rates of Gaussian ranking estimators which can avoid the curse of dimensionality by employing a low intrinsic dimensional assumption of the data.

\section{Proofs of the General Oracle Inequalities}\label{Section: Proofs of the General Oracle Inequalities}

In this section, we provide detailed proofs of the general oracle inequality established in Theorem \ref{Oracle inequality of pairwise ranking with a general loss} 
and its variant in Theorem \ref{Oracle inequality of pairwise ranking with a margin-base loss}. Following a standard error decomposition process, 
we use Hoeffding's decomposition and concentration estimates with U-process and local Rademacher analysis to bound the stochastic parts of the error terms.

\subsection{Error Decomposition and Hoeffding's Decomposition}\label{Subsection: Error Decomposition and Hoeffding's Decomposition}
Given an $f_0 \in \mathcal{H}_{K}$, by the definition of $f_{\mathbf{z}}$ we have
$\mathcal{R}^{\phi}_{\mathbf{z}}(f_{\mathbf{z}}) + \lambda\|f_{\mathbf{z}}\|^2_{K} 
\leq \mathcal{R}^{\phi}_{\mathbf{z}}(f_0) + \lambda\|f_0\|^2_{K}.$
Then
\begin{align*}
    &\lambda\|f_{\mathbf{z}}\|^2_{K} + \mathcal{R}^{\phi}(f_{\mathbf{z}}) - \mathcal{R}^{\phi}(f^*_\phi) \\
    \leq{}& \big(\lambda\|f_0\|^2_{K} + \mathcal{R}^{\phi}(f_0) - \mathcal{R}^{\phi}(f^*_\phi)\big)  
    + \big(\mathcal{R}^{\phi}_{\mathbf{z}}(f_0) - \mathcal{R}^{\phi}_{\mathbf{z}}(f^*_\phi) - \mathcal{R}^{\phi}(f_0) + \mathcal{R}^{\phi}(f^*_\phi)\big) \\
    &+ \big(\mathcal{R}^{\phi}(f_{\mathbf{z}}) - \mathcal{R}^{\phi}(f^*_\phi) 
    - \mathcal{R}^{\phi}_{\mathbf{z}}(f_{\mathbf{z}}) + \mathcal{R}^{\phi}_{\mathbf{z}}(f^*_\phi)\big).
\end{align*}
For a loss function $\phi$ that can be truncated at some $M>0$, 
since $\phi(y, y', \pi(t)) \leq \phi(y, y', t)$, we have 
$\mathcal{R}^{\phi}_{\mathbf{z}}(\pi(f_\mathbf{z})) \leq \mathcal{R}^{\phi}_{\mathbf{z}}(f_{\mathbf{z}})$ which implies $\mathcal{R}^{\phi}_{\mathbf{z}}(\pi(f_\mathbf{z})) + \lambda\|f_{\mathbf{z}}\|^2_{K} 
\leq \mathcal{R}^{\phi}_{\mathbf{z}}(f_0) + \lambda\|f_0\|^2_{K}$ for any $f_0\in \mathcal{H}_K$. 
Hence we can bound $\lambda\|f_{\mathbf{z}}\|^2_{K} + \mathcal{R}^{\phi}(\pi(f_\mathbf{z})) - \mathcal{R}^{\phi}(f^*_\phi)$ by the following error decomposition 
\begin{equation}\label{error decomposition}
   \lambda\|f_{\mathbf{z}}\|^2_{K} + \mathcal{R}^{\phi}(\pi(f_\mathbf{z})) - \mathcal{R}^{\phi}(f^*_\phi)
   \leq \mathcal{A}(\lambda, f_0) + \mathcal{S}_{\bf z}(f_0) +\mathcal{S}_{\bf z}(\pi(f_{\bf z}))
\end{equation}
where 
\begin{equation*}
    \begin{split}
        \mathcal{A}(\lambda, f_0)&:=\lambda\|f_0\|^2_{K} + \mathcal{R}^{\phi}(f_0) - \mathcal{R}^{\phi}(f^*_\phi),\\
        \mathcal{S}_{\bf z}(f_0)&:= \mathcal{R}^{\phi}_{\mathbf{z}}(f_0) - \mathcal{R}^{\phi}_{\mathbf{z}}(f^*_\phi) 
        - \mathcal{R}^{\phi}(f_0) + \mathcal{R}^{\phi}(f^*_\phi),\\
        \mathcal{S}_{\bf z}(\pi(f_{\bf z}))&:=\mathcal{R}^{\phi}(\pi(f_\mathbf{z})) - \mathcal{R}^{\phi}(f^*_\phi) - \mathcal{R}^{\phi}_{\mathbf{z}}(\pi(f_\mathbf{z})) + \mathcal{R}^{\phi}_{\mathbf{z}}(f^*_\phi).
    \end{split}
\end{equation*}
The first term $\mathcal{A}(\lambda, f_0)$ reflects the approximation error of using $f_0$ to approximate $f^*_\phi$. 
The other two terms $\mathcal{S}_{\bf z}(f_0)$ and $\mathcal{S}_{\bf z}(\pi(f_{\bf z}))$ constitute the stochastic parts of the error decomposition.

Under Assumption \ref{assumption 1}, we can bound $\mathcal{S}_{\bf z}(f_0)$ by Bernstein's inequality of $U$-statistics, cf. \cite{Arcones1995Bernstein-type}, stated below. 
\begin{lemma}\label{bernsteinineq}
    Let $\mathbf{z} = \{Z_i\}_{i=1}^n$ be an i.i.d. sample of a probability distribution $P$ on a measurable space $\mathcal{Z}$, 
    $h: \mathcal{Z}\times \mathcal{Z} \to \mathbb{R}$ be a symmetric measurable function with $\mathbb{E}[h(Z,Z')] = 0$ and 
    $\|h\|_\infty = b$. 
    The $U$-statistics with kernel $h$ is defined as
    \[U_\mathbf{z} (h) = \frac{1}{n(n-1)}\sum_{i=1}^n\sum_{j \neq i}h(Z_i,Z_j).\]
    Let $f: \mathcal{Z} \to \mathbb{R}$ be defined as
    $f(z) := \mathbb{E}[h(z,Z')]$ and $\zeta^2 := \mathbb{E}[f(Z)^2]$. 
    Then for all $t > 0$ we have 
    \[P\big(\sqrt{n}\vert U_\mathbf{z}(h)\vert \geq t\big) \leq  
    4\exp\left(-\frac{t^2}{8\zeta^2 + \left(\frac{64}{\sqrt{n-1}} + \frac{1}{3\sqrt{n}}\right)bt}\right),\]
    which implies that 
    \begin{equation}\label{Bernstein's inequality}
       P\left(U_\mathbf{z}(h) \geq \sqrt{\frac{8\zeta^2t}{n}}+\frac{150bt}{n}\right) \leq 2\exp(-t).
    \end{equation}
\end{lemma}

Recall that $\mathcal{S}_{\bf z}(f_0) = \mathcal{R}^{\phi}_{\mathbf{z}}(f_0) - \mathcal{R}^{\phi}_{\mathbf{z}}(f^*_\phi) 
        - \mathcal{R}^{\phi}(f_0) + \mathcal{R}^{\phi}(f^*_\phi)$.

\begin{proposition}\label{proposition 1}
    Given $f_0 \in \mathcal{H}_{K}$ satisfying $\|\phi_{f_0}\|_\infty \leq B_0$ for some constant $B_0 \geq 0$. 
    If Assumption \ref{assumption 1} holds, then with probability at least $1 - 4\exp(-t)$, 
    \begin{equation}\label{boundfors1}
    \mathcal{S}_{\bf z}(f_0)
    \leq \mathcal{R}^{\phi}(f_0) - \mathcal{R}^{\phi}(f^*_\phi) + \left(\frac{8V t}{n}\right)^{\frac{1}{2-\tau}}  
    + \frac{300Bt}{n} + \frac{152B_0t}{n}.
    \end{equation}
\end{proposition}
\begin{proof}
    Define $v_f:\mathcal{Z} \times \mathcal{Z} \to \mathbb{R}$ as $v_f:= \phi_f - \phi_{f^*_\phi}$, then 
    \[\mathcal{S}_{\bf z}(f_0) = \mathcal{R}^{\phi}_{\mathbf{z}}(f_0) - \mathcal{R}^{\phi}_{\mathbf{z}}(f^*_\phi) 
    - \mathcal{R}^{\phi}(f_0) + \mathcal{R}^{\phi}(f^*_\phi) = \mathbb{E}_\mathbf{z}v_{f_0} - \mathbb{E}v_{f_0}.\]
    We further decompose $\mathbb{E}_\mathbf{z}v_{f_0} - \mathbb{E}v_{f_0}$ into two terms by introducing the truncation $\pi(f_0)$, which is given by
    \[\mathbb{E}_\mathbf{z}v_{f_0} - \mathbb{E}v_{f_0} 
    = [\mathbb{E}_\mathbf{z}(v_{f_0}-v_{\pi(f_0)}) - \mathbb{E}(v_{f_0}-v_{\pi(f_0)})] + 
    (\mathbb{E}_\mathbf{z}v_{\pi(f_0)} - \mathbb{E}v_{\pi(f_0)}).\]
    Each term above is a difference between a $U$-statistics and its expectation, hence we can apply Bernstein's inequality in Lemma \ref{bernsteinineq} to bound them. 
    To this end, we need to establish the sup-norm bound and variance bound for the related U-statistics.

    We first bound $\mathbb{E}_\mathbf{z}(v_{f_0}-v_{\pi(f_0)}) - \mathbb{E}(v_{f_0}-v_{\pi(f_0)})$.
    Since \[\phi_{f_0}(z,z')-\phi_{\pi(f_0)}(z,z') = 
    \phi(y, y', f_0(x,x')) - \phi(y, y', \pi(f_0)(x,x')) \geq 0\]
    for all $(z,z') \in \mathcal{Z} \times \mathcal{Z}$, we have \[v_{f_0}-v_{\pi(f_0)} = \phi_{f_0}-\phi_{\pi(f_0)}:\mathcal{Z} \times \mathcal{Z} \to [0, B_0]\] and
    \[\|v_{f_0}-v_{\pi(f_0)} - \mathbb{E}(v_{f_0}-v_{\pi(f_0)})\|_\infty \leq B_0.\]
    Define $g:\mathcal{Z} \to \mathbb{R}$ as 
    \[g(z) := \mathbb{E}[v_{f_0}(z,Z') - v_{\pi(f_0)}(z,Z') - \mathbb{E}(v_{f_0}-v_{\pi(f_0)})].\]
    Then \[\mathbb{E}[g(Z)^2] \leq \mathbb{E}\bigg[\big(\mathbb{E}[v_{f_0}(Z,Z') - v_{\pi(f_0)}(Z,Z')]\big)^2\bigg] 
    \leq B_0\mathbb{E}(v_{f_0} - v_{\pi(f_0)}).\]
    Now we apply Bernstein's inequality \eqref{Bernstein's inequality} to the zero-mean symmetric function 
    $v_{f_0}-v_{\pi(f_0)} - \mathbb{E}(v_{f_0}-v_{\pi(f_0)})$, 
    showing that 
    \[P\left(\mathbb{E}_\mathbf{z}(v_{f_0}-v_{\pi(f_0)}) - \mathbb{E}(v_{f_0}-v_{\pi(f_0)}) 
    \geq \sqrt{\frac{8B_0\mathbb{E}(v_{f_0} - v_{\pi(f_0)})t}{n}} +\frac{150B_0t}{n}\right) \leq 2\exp(-t).\]
    By basic inequality $\sqrt{ab} \leq (a+b)/2,$ we have
    \[\sqrt{\frac{8B_0\mathbb{E}(v_{f_0} - v_{\pi(f_0)})t}{n}} \leq \mathbb{E}(v_{f_0} - v_{\pi(f_0)}) + \frac{2B_0t}{n}.\]
    Hence with probability at least $1-2\exp(-t)$, there holds
    \begin{equation}\label{Bernstein's inequality 1}
  \mathbb{E}_\mathbf{z}(v_{f_0}-v_{\pi(f_0)}) - \mathbb{E}(v_{f_0}-v_{\pi(f_0)}) \leq  \mathbb{E}(v_{f_0} - v_{\pi(f_0)}) +\frac{152B_0t}{n}.
    \end{equation}

    To bound the second term $\mathbb{E}_\mathbf{z} v_{\pi(f_0)} - \mathbb{E}v_{\pi(f_0)}$, define 
    $u:\mathcal{Z} \to \mathbb{R}$ as 
    \[u(z):= \mathbb{E}[v_{\pi(f_0)}(z,Z') - \mathbb{E}v_{\pi(f_0)}].\]
    By \eqref{sup-norm bound} and \eqref{variance bound} of Assumption \ref{assumption 1}, we have 
    \[v_{\pi(f_0)} = \phi_{\pi(f_0)} - \phi_{f^*_\phi} \in [-B,B],\]
    \[\|v_{\pi(f_0)} - \mathbb{E}v_{\pi(f_0)}\|_\infty \leq 2B,\] and
    \[\mathbb{E}[u(Z)^2] \leq \mathbb{E}\bigg[\big(\mathbb{E}[v_{\pi(f_0)}(Z,Z')]\big)^2\bigg]
    = \mathbb{E}\big(Q\phi_{\pi(f_0)}(Z) - Q\phi_{f^*_\phi}(Z)\big)^2 
    \leq  V (\mathbb{E}v_{\pi(f_0)})^\tau.\]
    We apply Bernstein's inequality to the zero-mean symmetric function $v_{\pi(f_0)} - \mathbb{E}v_{\pi(f_0)}$, showing that
    \[P\left(\mathbb{E}_\mathbf{z}(v_{\pi(f_0)}) - \mathbb{E}v_{\pi(f_0)}
    \geq \sqrt{\frac{8V (\mathbb{E}v_{\pi(f_0)})^\tau t}{n}}+\frac{300Bt}{n}\right) \leq 2\exp(-t).\]
    If $\tau > 0$, we use Young's inequality $ab \leq a^q/q + b^p/p$ by setting 
    \[a = \sqrt{\frac{2^{3-\tau}\tau^\tau V t}{n}}, \quad 
    b = \left(\frac{2\mathbb{E}v_{\pi(f_0)}}{\tau}\right)^{\tau/2}, \quad 
    q = \frac{2}{2 - \tau}, \quad p = \frac{2}{\tau} ,\]
    to yield 
    \[\sqrt{\frac{8V (\mathbb{E}v_{\pi(f_0)})^\tau t}{n}} = ab \leq \frac{2-\tau}{2} \cdot 
    \left(\frac{2^{3-\tau}\tau^\tau V t}{n}\right)^{\frac{1}{2-\tau}} + \mathbb{E}v_{\pi(f_0)} 
    \leq \left(\frac{8V t}{n}\right)^{\frac{1}{2-\tau}} + \mathbb{E}v_{\pi(f_0)}.\]
    Since $\mathbb{E}v_{\pi(f_0)} \geq 0$, this inequality also holds for $\tau = 0$.
    Hence with probability at least $1-2\exp(-t)$, there holds
    \begin{equation}\label{Bernstein's inequality 2}
      \mathbb{E}_\mathbf{z}(v_{\pi(f_0)}) - \mathbb{E}v_{\pi(f_0)}
        \leq \mathbb{E}v_{\pi(f_0)} + \left(\frac{8V t}{n}\right)^{\frac{1}{2-\tau}}  
        + \frac{300Bt}{n} .
    \end{equation}
    Consequently, combining \eqref{Bernstein's inequality 1} and \eqref{Bernstein's inequality 2}, we derive the desired bound \eqref{boundfors1}. This completes the proof.
\end{proof}

It is more involved to bound $\mathcal{S}_{\bf z}(\pi(f_{\bf z}))$. Recall that  
\[\mathcal{S}_{\bf z}(\pi(f_{\bf z}))=\mathcal{R}^\phi(\pi(f_\mathbf{z})) - \mathcal{R}^\phi(f^*_\phi) - \mathcal{R}^\phi_\mathbf{z}(\pi(f_\mathbf{z})) + \mathcal{R}^\phi_\mathbf{z}(f^*_\phi).\] 
We apply Hoeffding's decomposition of $U$-statistics (cf. \cite{Pena1999Decoupling:}) to $\mathcal{S}_{\bf z}(\pi(f_{\bf z}))$, which is given by
\begin{equation}\label{Hoeffding's decomposition}
\begin{split}
    &\mathcal{R}^\phi(\pi(f_\mathbf{z})) - \mathcal{R}^\phi(f^*_\phi) - \mathcal{R}^\phi_\mathbf{z}(\pi(f_\mathbf{z})) + \mathcal{R}^\phi_\mathbf{z}(f^*_\phi)\\
    &=2Q_\mathbf{z}[\mathcal{R}^\phi(\pi(f_\mathbf{z})) - \mathcal{R}^\phi(f^*_\phi) - Q\phi_{\pi(f_\mathbf{z})} + Q\phi_{f^*_\phi}] - U_\mathbf{z}(h_{\pi(f_\mathbf{z})} - h_{f^*_\phi})\\
    &=:2\mathcal{S}_{{\bf z}, 1}(\pi(f_{\bf z}))-\mathcal{S}_{{\bf z}, 2}(\pi(f_{\bf z})). 
\end{split}
\end{equation}
Here the functionals $Q_\mathbf{z}, U_\mathbf{z}$ and function $h_f:\mathcal{Z} \times \mathcal{Z} \to \mathbb{R}$ are defined by
\begin{equation*}
\begin{split}
    Q_\mathbf{z}(g)& := \frac{1}{n}\sum_{i=1}^ng(Z_i), \forall g: \mathcal{Z} \to \mathbb{R},\\
    U_\mathbf{z}(g')& :=  \frac{1}{n(n-1)}\sum_{i=1}^n\sum_{j \neq i}g'(Z_i,Z_j), \forall g':\mathcal{Z} \times \mathcal{Z} \to \mathbb{R},\\
    h_f(z,z')& :=  \phi_f(z,z') - Q\phi_f(z) - Q\phi_f(z') + \mathcal{R}^\phi(f).
\end{split}
\end{equation*}
Hoeffding's decomposition breaks $\mathcal{S}_{\bf z}(\pi(f_{\bf z}))$ into a sum of i.i.d. random variables $\mathcal{S}_{{\bf z}, 1}(\pi(f_{\bf z}))$ called \emph{empirical term} 
and a degenerate $U$-statistics $\mathcal{S}_{{\bf z}, 2}(\pi(f_{\bf z}))$ called \emph{degenerate term} which has zero conditional expectation
\[\mathbb{E}[h_{\pi(f_\mathbf{z})}(Z,Z') - h_{f^*_\phi}(Z,Z')\vert Z = z] = 0.\]
We derive bounds for the empirical term and degenerate term in the following two subsections respectively.

\subsection{Bounding the Empirical Term}\label{Subsection: Bounding the Empirical Term}
In this subsection, we estimate the empirical term 
\[\mathcal{S}_{{\bf z}, 1}(\pi(f_{\bf z}))=Q_\mathbf{z}[\mathcal{R}^\phi(\pi(f_\mathbf{z})) - \mathcal{R}^\phi(f^*_\phi) - Q\phi_{\pi(f_\mathbf{z})} + Q\phi_{f^*_\phi}]\] 
which can be further bounded by the supremum norm of an empirical process indexed by $\mathcal{H}_K$, i.e.,  
\[\mathcal{S}_{{\bf z}, 1}(\pi(f_{\bf z})) \leq \sup_{f \in \mathcal{H}_{K}}Q_\mathbf{z}[\mathcal{R}^\phi(\pi(f)) - \mathcal{R}^\phi(f^*_\phi) - Q\phi_{\pi(f)} + Q\phi_{f^*_\phi}].\] 
In the remaining part of this subsection, we will use Talagrand’s concentration inequality
together with local Rademacher averages to derive a tight bound for the supremum norm of the empirical process. The idea of our estimates  can be found in
Chapter 7 of \cite{Steinwart2008Support}, in which the authors established refined oracle inequalities for classification with support vector machines.

Define $s_{f},g_{f,r}:\mathcal{Z} \to \mathbb{R}$ as $s_{f}:= Q\phi_{f} - Q\phi_{f^*_\phi}$ with $Q\phi_{f}$ given by \eqref{Qphif} and 
\[g_{f,r} := \frac{\mathbb{E}s_{\pi(f)} - s_{\pi(f)}}{\lambda\|f\|^2_{K} + \mathbb{E}s_{\pi(f)} + r}\] 
for all $f \in \mathcal{H}_{K}$ and $r>0.$ Under Assumption \ref{assumption 1}, we apply Talagrand's inequality to obtain the following lemma. 
\begin{lemma}\label{lemma 1}
    If Assumption \ref{assumption 1} holds, then for all $t > 0$, with probability at least $1 - \exp(-t)$,
    \[\sup_{f \in \mathcal{H}_{K}} \vert Q_{\mathbf{z}}(g_{f,r})\vert 
    \leq \frac{5}{4} \mathbb{E} \sup_{f \in \mathcal{H}_{K}} \vert Q_{\mathbf{z}}(g_{f,r})\vert 
    + \sqrt{\frac{2V t}{nr^{2-\tau}}} + \frac{28Bt}{3nr}.\]
\end{lemma}
\begin{proof}
    Analogous to the proof of Proposition \ref{proposition 1},
    by \eqref{sup-norm bound} of Assumption \ref{assumption 1}, obviously $\|g_{f,r}\|_\infty \leq 2Br^{-1}$.
    For the variance bound, if $\tau > 0$, using Young's inequality, we can prove that $(qa)^{2/q}(pb)^{2/p} \leq (a+b)^2$, 
    then setting $a = r , b = \mathbb{E}s_{\pi(f)}, q=2/(2-\tau), p = 2/\tau$ yields
    \[\mathbb{E}g^2_{f,r} \leq \frac{\mathbb{E}s^2_{\pi(f)}}{(\mathbb{E}s_{\pi(f)} + r)^2} 
    \leq \frac{(2-\tau)^{2-\tau}\tau^\tau V}{4r^{2-\tau}} \leq V r^{\tau - 2}\] which also holds for $\tau = 0$.
    Hence we can apply Talagrand's inequality, cf. Theorem 7.5 of \cite{Steinwart2008Support}, 
    to $(g_{f,r})_{f \in \mathcal{H}_{K}}$, to derive the desired bound. This completes the proof.
\end{proof}

In order to derive an upper bound of $\sup_{f \in \mathcal{H}_{K}} \vert Q_{\mathbf{z}}(g_{f,r})\vert$, 
we leverage standard tools including peeling, symmetrization and Dudley's chaining to estimate $\mathbb{E}\sup_{f \in \mathcal{H}_{K}} \vert Q_{\mathbf{z}}(g_{f,r})\vert $. 
Define 
\begin{align*}
    r^* &:= \inf\{\lambda\|f\|^2_{K} + \mathbb{E}s_{\pi(f)} : f \in \mathcal{H}_{K}\},\\
    \mathcal{F}_r &:= \{f \in \mathcal{H}_{K}:\lambda\|f\|^2_{K} + \mathbb{E}s_{\pi(f)} \leq r\} ,\ \forall r > r^*,\\
    \mathcal{S}_r &:= \{s_{\pi(f)}: f \in \mathcal{F}_r\},\ \forall r > r^*,
\end{align*}
and the \emph{empirical Rademacher complexity} of $\mathcal{S}_r$
\[\text{Rad}(\mathcal{S}_r, \mathbf{z}) 
:= \mathbb{E}_{\varepsilon} \sup_{f \in \mathcal{F}_r} \bigg\vert\frac{1}{n}\sum_{i=1}^n \varepsilon_i s_{\pi(f)}(Z_i)\bigg\vert,\]
where $\{\varepsilon_i\}_{i=1}^n$ is the Rademacher sequence, i.e., $\{\varepsilon_i\}_{i=1}^n$ are i.i.d. random variables  uniformly chosen from $\{-1,1\}$. 
The following two lemmas present peeling and symmetrization arguments to handle $\mathbb{E}\sup_{f \in \mathcal{F}_r} \vert Q_{\mathbf{z}}(\mathbb{E}s_{\pi(f)} - s_{\pi(f)})\vert$, 
which can be proved according to Theorem 7.7 and Proposition 7.10 in \cite{Steinwart2008Support}.

\begin{lemma}\label{lemma 2}
    Let $\varphi : (r^*,\infty) \to [0,\infty)$ such that $\varphi(4r) \leq 2\varphi(r)$ and 
    \[\mathbb{E}\sup_{f \in \mathcal{F}_r} \vert Q_{\mathbf{z}}(\mathbb{E}s_{\pi(f)} - s_{\pi(f)})\vert \leq \varphi(r)\]
    for all $r > r^*$. Then for all $r > r^*$, we have
    \[\mathbb{E}\sup_{f \in \mathcal{H}_{K}} \vert Q_{\mathbf{z}}(g_{f,r})\vert \leq \frac{4\varphi(r)}{r}.\]
\end{lemma}

\begin{lemma} \label{lemma 3}
    \[\mathbb{E}\sup_{f \in \mathcal{F}_r} \vert Q_{\mathbf{z}}(\mathbb{E}s_{\pi(f)} - s_{\pi(f)})\vert 
    \leq 2\mathbb{E} \operatorname{Rad}(\mathcal{S}_r,\mathbf{z}).\]
\end{lemma}

We then apply Theorem 7.16 of \cite{Steinwart2008Support} to derive the following lemma.
\begin{lemma}\label{lemma 4}
    If Assumption \ref{assumption 1} and Assumption \ref{Capacity condition} hold, 
    then there exist constants $k_1 > 0$ and $k_2 > 0$ such that 
    \begin{align*}
        \mathbb{E} \operatorname{Rad}(\mathcal{S}_r,\mathbf{z}) 
        \leq{} \max\bigg\{& 2^{p_1}k_1 a_1^{p_1} \left(\frac{r}{\lambda}\right)^{\frac{p_1}{2}} 
        (V r^{\tau})^{\frac{1-p_1}{2}} L^{p_1} n^{-\frac{1}{2}}, \\
        &4^{\frac{p_1}{1+p_1}}k_2a_1^{\frac{2p_1}{1+p_1}} \left(\frac{r}{\lambda}\right)^{\frac{p_1}{1+p_1}} 
        B^{\frac{1-p_1}{1+p_1}} L^{\frac{2p_1}{1+p_1}} n^{-\frac{1}{1+p_1}}\bigg\}.
    \end{align*}
\end{lemma}
\begin{proof} We first establish the sup-norm bound and variance bound for $s_{\pi(f)} \in \mathcal{S}_r.$ By Assumption \ref{assumption 1} we have $\|s_{\pi(f)}\|_\infty \leq B$ and
    \[\mathbb{E}(s^2_{\pi(f)}) = \mathbb{E}(Q\phi_{\pi(f)} - Q\phi_{f^*_\phi})^2 \leq 
    V\big(\mathbb{E}(Q\phi_{\pi(f)}  - Q\phi_{f^*_\phi})\big)^\tau 
    = V (\mathbb{E}s_{\pi(f)})^{\tau} \leq V r^{\tau}.\]
    For $\mathbf{z} = \{(X_i,Y_i)\}_{i=1}^n$, the empirical norm $\| \cdot \|_{\mathcal{L}_2(P^n_{\mathbf{z}})}$ on $\mathcal{S}_r$ is defined as
    \[\|s_{\pi(f)}\|_{\mathcal{L}_2(P^n_{\mathbf{z}})} := Q_{\mathbf{z}}^{\frac{1}{2}}(s^2_{\pi(f)}) =
    \bigg(\frac{1}{n}\sum_{i=1}^n s_{\pi(f)}(Z_i)^2\bigg)^{\frac{1}{2}}.\]
    According to Theorem 7.16 of \cite{Steinwart2008Support}, it remains to bound 
    $\mathbb{E}e_i(\textrm{id}: \mathcal{S}_r \to \mathcal{L}_2(P^n_{\mathbf{z}})).$ 
    For all $(y,y') \in \mathcal{Y} \times \mathcal{Y}$, function $\phi(y,y',\cdot)$ is locally $L$-Lipschitz over $[-M,M]$.
    Hence for all $s_{\pi(f)},s_{\pi(f')} \in \mathcal{S}_r$ and $\mathbf{z} = \{(X_i,Y_i)\}_{i=1}^n$ we have
    \begin{align*}
        \|s_{\pi(f)} - s_{\pi(f')}\|^2_{\mathcal{L}_2(P^n_{\mathbf{z}})} 
        &= \frac{1}{n}\sum_{i=1}^n\big(s_{\pi(f)}(Z_i) - s_{\pi(f')}(Z_i)\big)^2 \\
        &= \frac{1}{n}\sum_{i=1}^n\bigg(\mathbb{E}\bigg[\phi(Y_i,Y',\pi(f)(X_i,X')) - \phi(Y_i,Y',\pi(f')(X_i,X'))\bigg]\bigg)^2 \\
        &\leq \frac{L^2}{n}\sum_{i=1}^n\bigg(\mathbb{E}\bigg\vert\pi(f)(X_i,X') - \pi(f')(X_i,X')\bigg\vert\bigg)^2 \\
        &\leq \frac{L^2}{n}\sum_{i=1}^n\mathbb{E}\bigg\vert\pi(f)(X_i,X') - \pi(f')(X_i,X')\bigg\vert^2 \\
        &\leq  L^2\|f - f'\|^2_{\mathcal{L}_2(P^n_{\mathbf{x}} \otimes P_{\mathcal{X}})}.
    \end{align*}
    Therefore, an $\varepsilon$-net of $(\mathcal{F}_r, \| \cdot \|_{\mathcal{L}_2(P^n_{\mathbf{x}} \otimes P_{\mathcal{X}})})$ induces an 
    $L\varepsilon$-net of $(\mathcal{S}_r,\| \cdot \|_{\mathcal{L}_2(P^n_{\mathbf{z}})})$.
    Besides, notice that $\mathcal{F}_r \subset (r/\lambda)^{1/2}\mathcal{B}_{\mathcal{H}_{K}}$.
    Combine with Assumption \ref{Capacity condition} we have 
    \begin{align*}
        \mathbb{E}e_i(\textrm{id}: \mathcal{S}_r \to \mathcal{L}_2(P^n_{\mathbf{z}})) &\leq 
        L\mathbb{E}e_i(\textrm{id}: \mathcal{F}_r \to \mathcal{L}_2(P^n_{\mathbf{x}} \otimes P_{\mathcal{X}})) \\
        & \leq2L(r/\lambda)^{1/2}\mathbb{E}e_i(\textrm{id}: \mathcal{H}_{K} \to \mathcal{L}_2(P^n_{\mathbf{x}} \otimes P_{\mathcal{X}})) \\
        & \leq2L(r/\lambda)^{1/2}a_1i^{-\frac{1}{2p_1}}.
    \end{align*}
    Now we apply Theorem 7.16 of \cite{Steinwart2008Support} and derive the desired bound. The proof is then finished.
\end{proof}

With the help of the preceding lemmas, we can now establish an upper bound for empirical term $\mathcal{S}_{{\bf z}, 1}(\pi(f_{\bf z}))$.

\begin{proposition}\label{proposition 2}
    If Assumption \ref{assumption 1} and Assumption \ref{Capacity condition} hold, 
    there exist constants 
    \[c_1 := \frac{\big(2^{p_1} \cdot 270k_1\big)^\frac{2}{2-\tau-p_1+p_1\tau}}{3}, 
    \quad c_2 := \frac{2^{1 + 3p_1}\big(135k_2\big)^{1+p_1}}{3}\]
    such that for all $t > 0$ , $n \geq 72t $ and $f_0 \in \mathcal{H}_{K}$, 
    with probability at least $1 - \exp(-t)$,
    \begin{align*}
        \mathcal{S}_{{\bf z}, 1}(\pi(f_{\bf z}))&\leq \frac{1}{3}\left(\lambda\|f_\mathbf{z}\|^2_{K} + \mathcal{R}^\phi(\pi(f_\mathbf{z})) - \mathcal{R}^\phi(f^*_\phi)\right)
        + \frac{1}{3}\left(\lambda\|f_0\|^2_{K} + \mathcal{R}^\phi(f_0) - \mathcal{R}^\phi(f^*_\phi)\right) \notag \\
        &+  c_1 \left(\frac{a_1^{2p_1}V^{1-p_1}L^{2p_1}}{\lambda^{p_1} n}\right)^{\frac{1}{2-\tau-p_1+p_1\tau}} 
        +  \frac{c_2a_1^{2p_1}B^{1-p_1}L^{2p_1}}{\lambda^{p_1}n}
        + \left(\frac{24V t}{n}\right)^{\frac{1}{2-\tau}}.
    \end{align*}
\end{proposition}
\begin{proof}
    By Lemma \ref{lemma 4}  we have
    \begin{align*}
        \mathbb{E} \operatorname{Rad}(\mathcal{S}_r,\mathbf{z}) 
        \leq{} \max\bigg\{& 2^{p_1}k_1 a_1^{p_1} \left(\frac{r}{\lambda}\right)^{\frac{p_1}{2}} 
        (V r^{\tau})^{\frac{1-p_1}{2}} L^{p_1} n^{-\frac{1}{2}}, \\
        &4^{\frac{p_1}{1+p_1}}k_2a_1^{\frac{2p_1}{1+p_1}} \left(\frac{r}{\lambda}\right)^{\frac{p_1}{1+p_1}} 
        B^{\frac{1-p_1}{1+p_1}} L^{\frac{2p_1}{1+p_1}} n^{-\frac{1}{1+p_1}}\bigg\}.
    \end{align*}
    Define 
    \begin{align*}
        \varphi(r) := \max\bigg\{& 2^{p_1+1}k_1 a_1^{p_1} \left(\frac{r}{\lambda}\right)^{\frac{p_1}{2}} 
        (V r^{\tau})^{\frac{1-p_1}{2}} L^{p_1} n^{-\frac{1}{2}}, \\
        &2^{\frac{1 + 3p_1}{1+p_1}}k_2a_1^{\frac{2p_1}{1+p_1}} \left(\frac{r}{\lambda}\right)^{\frac{p_1}{1+p_1}} 
        B^{\frac{1-p_1}{1+p_1}} L^{\frac{2p_1}{1+p_1}} n^{-\frac{1}{1+p_1}}\bigg\}.
    \end{align*}
    One can verify that $\varphi(4r) \leq 2\varphi(r)$.
    By Lemma \ref{lemma 1}, Lemma \ref{lemma 2} and Lemma \ref{lemma 3}, for all $t > 0$ we have 
    \[P\left(\sup_{f \in \mathcal{H}_{K}} \vert Q_{\mathbf{z}}(g_{f,r})\vert
    \leq \frac{5\varphi(r)}{r}
    + \sqrt{\frac{2V t}{r^{2-\tau}n}} + \frac{28Bt}{3nr}\right) \geq 1 - \exp(-t).\]
    For $f_\mathbf{z} \in \mathcal{H}_{K}$, by the definition of $g_{f_\mathbf{z},r}$, we have, 
    with probability at least $1 - \exp(-t)$, 
    \begin{align*}
        &Q_\mathbf{z}[\mathcal{R}^\phi(\pi(f_\mathbf{z})) - \mathcal{R}^\phi(f^*_\phi) - Q\phi_{\pi(f_\mathbf{z})} + Q\phi_{f^*_\phi}] \\
        \leq{} &(\lambda\|f_\mathbf{z}\|^2_{K} + \mathcal{R}^\phi(\pi(f_\mathbf{z})) - \mathcal{R}^\phi(f^*_\phi) + r) 
        \left(\frac{5\varphi(r)}{r} + \sqrt{\frac{2V t}{r^{2-\tau}n}} + \frac{28Bt}{3nr}\right). 
    \end{align*}
    If there exists an 
    \[r \geq \max\bigg\{135\varphi(r), \left(\frac{72V t}{n}\right)^{\frac{1}{2-\tau}}, r^*\bigg\},\]
    then using simple algebra and the assumption $n \geq 72t$ we have 
    \[\frac{5\varphi(r)}{r} \leq \frac{1}{27}, \quad 
    \sqrt{\frac{2V t}{r^{2-\tau}n}} \leq \frac{1}{6},\quad
    \frac{28Bt}{3nr} \leq \frac{28V^{\frac{1}{2-\tau}}t}{3nr} \leq \frac{7}{54}.\]
    Hence with probability at least $1 - \exp(-t)$,
    \begin{equation}\label{oracle 2}
        \begin{split}
            &Q_\mathbf{z}[\mathcal{R}^\phi(\pi(f_\mathbf{z})) - \mathcal{R}^\phi(f^*_\phi) - Q\phi_{\pi(f_\mathbf{z})} + Q\phi_{f^*_\phi}]\\ 
            \leq{}& \frac{1}{3}(\lambda\|f_\mathbf{z}\|^2_{K} + \mathcal{R}^\phi(\pi(f_\mathbf{z})) - \mathcal{R}^\phi(f^*_\phi) + r).
        \end{split}
        \end{equation}
    It remains to find such $r$.
    By definition of $\varphi(r)$, if 
    \begin{align*}
        r &\geq \left(\frac{4^{p_1}\big(270k_1\big)^2a_1^{2p_1}V^{1-p_1}L^{2p_1}}{\lambda^{p_1} n}\right)^{\frac{1}{2-\tau-p_1+p_1\tau}} \\
    &\quad+  \frac{2^{1+3p_1}\big(135k_2\big)^{1+p_1}a_1^{2p_1}B^{1-p_1}L^{2p_1}}{\lambda^{p_1}n},
    \end{align*}
    we have $r \geq 135\varphi(r)$.
    Besides, notice that $r^* \leq \lambda\|f_0\|^2_{K} + \mathcal{R}^\phi(f_0) - \mathcal{R}^\phi(f^*_\phi)$, 
    it is sufficient to choose
    \begin{align*}
    r ={} &\lambda\|f_0\|^2_{K} + \mathcal{R}^\phi(f_0) - \mathcal{R}^\phi(f^*_\phi)
    +  \left(\frac{4^{p_1}\big(270k_1\big)^2a_1^{2p_1}V^{1-p_1}L^{2p_1}}{\lambda^{p_1} n}\right)^{\frac{1}{2-\tau-p_1+p_1\tau}} \\
    &+  \frac{2^{1 + 3p_1}\big(135k_2\big)^{1+p_1}a_1^{2p_1}B^{1-p_1}L^{2p_1}}{\lambda^{p_1}n}
    + \left(\frac{72V t}{n}\right)^{\frac{1}{2-\tau}}.       
    \end{align*}
    Substitute this $r$ into \eqref{oracle 2} and choose 
    \[c_1 = \frac{\big(2^{p_1} \cdot 270k_1\big)^\frac{2}{2-\tau-p_1+p_1\tau}}{3}, 
    \quad c_2 = \frac{2^{1 + 3p_1}\big(135k_2\big)^{1+p_1}}{3}.\]
    The proof of Proposition \ref{proposition 2} is then finished.
\end{proof}
Here we remark that the condition $n \geq 72t$ in Proposition \ref{proposition 2} is not essential and we will leave the case $n < 72t$ to the final proof of the general oracle inequality in the next subsection.

\subsection{Bounding the Degenerate Term}\label{Subsection: Bounding the Degenerate Term}

In this subsection, we use arguments of symmetrization and Dudley's chaining to bound the degenerate term
\[\mathcal{S}_{{\bf z}, 2}(\pi(f_{\bf z}))= U_\mathbf{z}(h_{\pi(f_\mathbf{z})} - h_{f^*_\phi}).\]

If Assumption \ref{assumption 1} holds, 
since $\mathcal{R}^\phi_\mathbf{z}(f_\mathbf{z}) + \lambda\|f_\mathbf{z}\|^2_{K} \leq \mathcal{R}^\phi_\mathbf{z}(0) \leq B$, we have
$\|f_\mathbf{z}\|_{K} \leq \sqrt{B/\lambda}$. Recalling the Rademacher sequence $\{\varepsilon_i\}_{i=1}^n$, define 
\begin{align*}
    \mathcal{F} &:= \{f \in \mathcal{H}_{K}:\|f\|_{K} \leq \sqrt{B/\lambda}\}, \\ 
    S_\mathbf{z}(h_{\pi(f)} - h_{f^*_\phi}) &:= \frac{1}{n(n-1)} \sum_{i=1}^n\sum_{j \neq i} \varepsilon_i \varepsilon_j 
    \big(h_{\pi(f)}(Z_i,Z_j) - h_{f^*_\phi}(Z_i,Z_j)\big).
\end{align*} Then $f_\mathbf{z} \in \mathcal{F}$.

\begin{lemma}\label{lemma 6}
    There exist constants $C_1>0, C_2>0$ such that, for all $\xi > 0$
    \begin{align*}
    &\mathbb{E}\exp \left(\xi \sqrt{\sup_{f \in \mathcal{F}} \vert(n-1)U_\mathbf{z}(h_{\pi(f)} - h_{f^*_\phi})\vert}\right) \\
    \leq{}& 
    C_1\mathbb{E}\mathbb{E}_{\varepsilon}\exp \left(C_2\xi \sqrt{\sup_{f \in \mathcal{F}} \vert(n-1)S_\mathbf{z}(h_{\pi(f)} - h_{f^*_\phi})\vert}\right)       
    \end{align*}
\end{lemma}
\begin{proof}
    Define $\varphi_1(t) := \exp(\xi \sqrt{t})$ and 
    \[\varphi_2(t) :=
			\begin{cases}
				e &0 \leq t \leq \xi^{-2} \\
				\exp(\xi \sqrt{t}) &t \geq \xi^{-2}. \\
			\end{cases}
	\]
    Then $\varphi_2$ is convex, non-decreasing, and the following inequality holds:
    \begin{equation}\label{sandwich inequality}
        \varphi_2(t)e^{-1} \leq \varphi_1(t) \leq \varphi_2(t).
    \end{equation}
    We can apply the symmetrization argument in the theory of $U$-processes with $\varphi_2$, cf. Remarks 3.5.4 of \cite{Pena1999Decoupling:},  to show that there exist constants $C'_1>0$ and $C'_2>0$ such that 
    \[\mathbb{E}\varphi_2(\sup_{f \in \mathcal{F}} \vert(n-1)U_\mathbf{z}(h_{\pi(f)} - h_{f^*_\phi})\vert) 
    \leq C'_1\mathbb{E}\mathbb{E}_{\varepsilon}\varphi_2\big(C'_2\sup_{f \in \mathcal{F}} \vert(n-1)S_\mathbf{z}(h_{\pi(f)} - h_{f^*_\phi})\vert\big).\]
    Combining this with the inequality \eqref{sandwich inequality} finishes the proof.
\end{proof}

\begin{lemma}\label{lemma 7}
    There exist constants $C_3>0, C_4>0$ such that, for all $\xi > 0$,
    \begin{align*}
    &C_1\mathbb{E}\mathbb{E}_{\varepsilon}\exp \left(C_2\xi \sqrt{\sup_{f \in \mathcal{F}} \vert(n-1)S_\mathbf{z}(h_{\pi(f)} - h_{f^*_\phi})\vert}\right) \\
    \leq{}& C_3\mathbb{E}\exp \left(C_4\xi^2 \mathbb{E}_{\varepsilon}\sup_{f \in \mathcal{F}} \vert(n-1)S_\mathbf{z}(h_{\pi(f)} - h_{f^*_\phi})\vert\right).  
    \end{align*}
\end{lemma}
\begin{proof}
    Define a Rademacher chaos of order two $R_1 := C_2^2\xi^2(n-1)S_\mathbf{z}(h_{\pi(f)} - h_{f^*_\phi})$, 
    and \[\|R_1\| := \sup_{f \in \mathcal{F}}C_2^2\xi^2(n-1)\vert S_\mathbf{z}(h_{\pi(f)} - h_{f^*_\phi})\vert.\]
    By formula (3.5) of \cite{Arcones1994U-processes}, there exist constants $C'_3 > 0, C'_4 >0$ such that 
    \begin{equation}\label{oracle 3}
        \mathbb{E}_{\varepsilon} \exp(\|R_1\|^{1/2}) \leq C'_3 \exp(C'_4 (\mathbb{E}_{\varepsilon}\|R_1\|^2)^{1/2}).
    \end{equation}
    Besides, by Hölder's inequality and formula (3.4) of \cite{Arcones1994U-processes} we have
    \begin{align*}
    (\mathbb{E}_{\varepsilon}\|R_1\|^2)^{1/2} &= (\mathbb{E}_{\varepsilon}\|R_1\|^{1/2} \|R_1\|^{3/2})^{1/2} \\
    &\leq (\mathbb{E}_{\varepsilon}\|R_1\|)^{1/4}(\mathbb{E}_{\varepsilon}\|R_1\|^3)^{1/4} \\
    &\leq 2^{3/4}(\mathbb{E}_{\varepsilon}\|R_1\|)^{1/4}(\mathbb{E}_{\varepsilon}\|R_1\|^2)^{3/8},      
    \end{align*}
    which is equivalent to
    \[(\mathbb{E}_{\varepsilon}\|R_1\|^2)^{1/2} \leq 8\mathbb{E}_{\varepsilon}\|R_1\|.\]
    Combining this with \eqref{oracle 3} we complete the proof.
\end{proof}
Next we bound $\mathbb{E}_{\varepsilon}\sup_{f \in \mathcal{F}} \vert(n-1)S_\mathbf{z}(h_{\pi(f)} - h_{f^*_\phi})\vert$ by the chaining argument.

\begin{lemma}\label{lemma 8}
    If Assumption \ref{assumption 1} and Assumption \ref{Capacity condition} hold, 
    then there exists a constant $C_5 > 0$ such that
    \begin{align*}
        &\mathbb{E}_\varepsilon\sup_{f \in \mathcal{F}} \vert (n-1)S_\mathbf{z}(h_{\pi(f)} - h_{f^*_\phi})\vert \\
        \leq{}& (\log 4)C_5B \bigg[  \bigg(\frac{3a_1^2L^2}{2\lambda B} \bigg)^{p_1}\frac{4^{1-2p_1}}{1-2p_1}
        +  \bigg(\frac{3a_2^2L^2}{4\lambda B}\bigg)^{p_2}\frac{4^{1-2p_2}}{1-2p_2} \bigg] 
        +  \frac{(1+\log 4\sqrt{3})C_5B}{4}.
    \end{align*}
\end{lemma}
\begin{proof}
    Define
    \[\mathcal{G} := \{h_{\pi(f)} - h_{f^*_\phi}: f \in \mathcal{F}\} \cup \{ 0:\mathcal{Z} \times \mathcal{Z} \to \mathbb{R} \}.\]
    By \eqref{sup-norm bound} of Assumption \ref{assumption 1}, $\|h\|_{\infty} \leq 4B$ for all $h \in \mathcal{G}$.
    Now given $\mathbf{z} = \{(X_i,Y_i)\}_{i=1}^n$, define a stochastic process $J_\mathbf{z}$ and 
    a pseudometric $\rho$ on $\mathcal{G}$ as 
    \begin{align*}
        J_\mathbf{z}(h) &:= \frac{1}{4Bn}\sum_{i=1}^n\sum_{j \neq i}\varepsilon_i\varepsilon_j h(Z_i,Z_j), \quad h \in \mathcal{G},\\
        \rho(h,h') &:= \frac{1}{4B}\sqrt{\frac{1}{n(n-1)}\sum_{i=1}^n\sum_{j \neq i}\big(h(Z_i,Z_j) - h'(Z_i,Z_j)\big)^2}, \quad h,h' \in \mathcal{G}.
    \end{align*}
    
    We will apply chaining lemma for $U$-processes to $\{J_\mathbf{z}(h): h \in \mathcal{G}\}$, cf. Lemma 5 of \cite{Nolan1987U-Processes:a}.
    To this end we need to find a convex, strictly increasing function $\varphi: [0, \infty) \to [0, \infty)$ such that
    $0 \leq \varphi(0) \leq 1$ and for all $\rho(h,h') > 0$, 
    \[\mathbb{E}_\varepsilon \varphi\left(\frac{\vert J_\mathbf{z}(h) - J_\mathbf{z}(h')\vert}{\rho(h,h')}\right) \leq 1.\]
    Define a Rademacher chaos of order two $R_2 := J_\mathbf{z}(h) - J_\mathbf{z}(h').$
    By Corollary 3.2.6 of \cite{Pena1999Decoupling:}, there exists a constant $\kappa > 0$ which is independent of 
    $h, h', n$ 
    such that 
    \begin{equation}\label{oracle 4}
        \mathbb{E}_\varepsilon \exp\left(\frac{\vert R_2\vert }{\kappa (\mathbb{E}_\varepsilon \vert R_2\vert ^2)^{1/2}}\right) \leq e.
    \end{equation}
    Notice that 
    \begin{align}\label{oracle 5}
        \mathbb{E}_\varepsilon \vert R_2\vert ^2 ={}& \frac{1}{16B^2n^2}  
        \mathbb{E}_\varepsilon \notag
        \bigg(\sum_{i \neq j} \varepsilon_i \varepsilon_j \big(h(Z_i,Z_j) - h'(Z_i,Z_j)\big)\bigg)^2 \notag\\
        ={}& \frac{1}{16B^2n^2}
        \bigg(\sum_{i = k \neq j = l}\big(h(Z_i,Z_j) - h'(Z_i,Z_j)\big)\big(h(Z_k,Z_l) - h'(Z_k,Z_l)\big)\notag \\
        &+ \sum_{i = l \neq j = k}\big(h(Z_i,Z_j) - h'(Z_i,Z_j)\big)\big(h(Z_k,Z_l) - h'(Z_k,Z_l)\big)\bigg) \notag\\
        ={}& \frac{1}{8B^2n^2} \sum_{i \neq j}\big(h(Z_i,Z_j) - h'(Z_i,Z_j)\big)^2 \notag\\
        \leq{}& 2\rho(h,h')^2.
    \end{align}
    Hence if we choose $\varphi(t) = \exp(\frac{t}{\sqrt{2}\kappa} -1)$, by \eqref{oracle 4} and \eqref{oracle 5}, we have 
    \begin{align*}
    \mathbb{E}_\varepsilon \varphi\left(\frac{\vert J_\mathbf{z}(h) - J_\mathbf{z}(h')\vert }{\rho(h,h')}\right) 
    &\leq \mathbb{E}_\varepsilon \varphi\left(\frac{\sqrt{2}\vert R_2\vert }{(\mathbb{E}_\varepsilon \vert R_2\vert ^2)^{1/2}}\right) \\
    &= \mathbb{E}_\varepsilon \exp\left(\frac{\vert R_2\vert }{\kappa(\mathbb{E}_\varepsilon \vert R_2\vert ^2)^{1/2}} - 1\right) \leq 1.        
    \end{align*}

    Fix $h'=0 \in \mathcal{G}$, obviously $\sup_{h \in \mathcal{G}} \rho(h,0) \leq 1$, and
    \[4B\mathbb{E}_\varepsilon\sup_{h \in \mathcal{G}}\vert J_\mathbf{z}(h) - J_\mathbf{z}(0)\vert  
    = \mathbb{E}_\varepsilon\sup_{f \in \mathcal{F}} \vert (n-1)S_\mathbf{z}(h_{\pi(f)} - h_{f^*_\phi})\vert .\]
    Therefore, by Lemma 5 of \cite{Nolan1987U-Processes:a}, we obtain 
    \begin{equation}\label{Ustatchaining}
        \begin{split}
            \mathbb{E}_\varepsilon\sup_{f \in \mathcal{F}} \vert (n-1)S_\mathbf{z}(h_{\pi(f)} - h_{f^*_\phi})\vert  
            &\leq 32B\int_{0}^{\frac{1}{4}} \varphi^{-1}(\mathcal{N}(\mathcal{G},\rho,\delta)) \,d\delta \\
            &\leq C_5B\int_{0}^{\frac{1}{4}} \log\mathcal{N}(\mathcal{G},\rho,\delta)\,d\delta,
        \end{split}
    \end{equation}
    where $C_5 = 64\sqrt{2}\kappa$.
    It remains to bound the metric entropy $\log\mathcal{N}(\mathcal{G},\rho,\delta)$.
    We write 
    \begin{align*}
        &16B^2\rho^2(h,h') \\
        ={}& \frac{1}{n(n-1)}\sum_{i=1}^n\sum_{j \neq i}\big(h(Z_i,Z_j) - h'(Z_i,Z_j)\big)^2 \\
        \leq{}& \frac{4}{n(n-1)}\sum_{i=1}^n\sum_{j \neq i}\bigg[\big(\phi_{\pi(f)}(Z_i,Z_j) - \phi_{\pi(f')}(Z_i,Z_j)\big)^2 \\
        &+\big(Q\phi_{\pi(f)}(Z_i) - Q\phi_{\pi(f')}(Z_i)\big)^2
        +\big(Q\phi_{\pi(f)}(Z_j) - Q\phi_{\pi(f')}(Z_j)\big)^2\\
        &+\big(\mathcal{R}^{\phi}(\pi(f)) - \mathcal{R}^{\phi}(\pi(f'))\big)^2 \bigg] \\
        \leq{}& \frac{4}{n(n-1)}\sum_{i=1}^n\sum_{j \neq i}\bigg[L^2 \big(\pi(f)(X_i,X_j) - \pi(f')(X_i,X_j)\big)^2 \\
        &+ L^2 \mathbb{E}\big(\pi(f)(X_i, X') - \pi(f')(X_i, X')\big)^2
        + L^2 \mathbb{E}\big(\pi(f)(X_j, X') - \pi(f')(X_j, X')\big)^2 \\
        &+ \big(\mathcal{R}^{\phi}(\pi(f)) - \mathcal{R}^{\phi}(\pi(f'))\big)^2 \bigg] \\
        \leq{}& 4L^2\|f-f'\|^2_{\mathcal{L}_2(P^n_{\mathbf{x}^2})} + 8L^2\|f-f'\|^2_{\mathcal{L}_2(P^n_{\mathbf{x}} \otimes P_{\mathcal{X}})} 
        + 4\big(\mathcal{R}^{\phi}(\pi(f)) - \mathcal{R}^{\phi}(\pi(f'))\big)^2,
    \end{align*}
    where the second inequality follows from the assumption that $\phi$ is $L$-Lipschitz with respect to the third argument.
    This implies that
    \begin{align*}
        \mathcal{N}(\mathcal{G},\rho,\delta) \leq{}& 
        \mathcal{N}\bigg(\mathcal{F}, \| \cdot \|_{\mathcal{L}_2(P^n_{\mathbf{x}^2})} ,\frac{2B\delta}{\sqrt{3}L}\bigg)
        \cdot\mathcal{N}\bigg(\mathcal{F}, \| \cdot \|_{\mathcal{L}_2(P^n_{\mathbf{x}} \otimes P_{\mathcal{X}})} , \frac{\sqrt{2}B\delta}{\sqrt{3}L}\bigg) 
        \cdot\mathcal{N}\bigg([0,B], \vert \cdot \vert , \frac{2B\delta}{\sqrt{3}}\bigg) \\
        ={}& \mathcal{N}\bigg(\mathrm{id} : \mathcal{H}_{K} \to \mathcal{L}_2(P^n_{\mathbf{x}^2}) , 2\sqrt{\frac{\lambda B}{3}}\frac{\delta}{L} \bigg)
        \cdot\mathcal{N}\bigg(\mathrm{id} : \mathcal{H}_{K} \to \mathcal{L}_2(P^n_{\mathbf{x}} \otimes P_{\mathcal{X}}) , \sqrt{\frac{2\lambda B}{3}}\frac{\delta}{L} \bigg) \\
        &\cdot\mathcal{N}\bigg([0,B], \vert \cdot \vert , \frac{2B\delta}{\sqrt{3}}\bigg).
    \end{align*}
    By Assumption \ref{Capacity condition} we have 
    \[\log\mathcal{N}(\mathcal{G},\rho,\delta) \leq \log 4 \cdot \bigg(\frac{3a_1^2L^2}{2\lambda B \delta^2}\bigg)^{p_1}
    + \log 4 \cdot \bigg(\frac{3a_2^2L^2}{4\lambda B \delta^2}\bigg)^{p_2}
    + \log \frac{\sqrt{3}}{\delta}.\]
    Since $p_1, p_2 \in (0, 1/2)$, we integrate the right-hand side and yield the desired bound by \eqref{Ustatchaining}. The proof is then finished.
\end{proof}

With the help of the preceding lemmas, we can now establish an upper bound for degenerate term $\mathcal{S}_{{\bf z}, 2}(\pi(f_{\bf z}))$.

\begin{proposition}\label{proposition 3}
    If Assumption \ref{assumption 1} and Assumption \ref{Capacity condition} hold, 
    there exist constants
    \begin{align*}
        c_0 := C_3 , \quad 
        c_3 :=  \frac{(64 \log 2) C_4C_5 3^{p_1}}{32^{p_1}(1-2p_1)}, \quad 
        c_4 :=  \frac{(64 \log 2) C_4C_5 3^{p_2}}{64^{p_2}(1-2p_2)}, \quad 
        c_5 :=  (2 + \log 48)C_4 C_5,
    \end{align*}
    such that for all $t>0$ and $n \geq 2$, with probability at least $1-c_0\exp(-t)$ we have 
    \begin{align*}
        \vert\mathcal{S}_{{\bf z}, 2}(\pi(f_{\bf z}))\vert  \leq 
        \frac{c_3a_1^{2p_1}L^{2p_1}B^{1-p_1}t}{\lambda^{p_1} n} + 
        \frac{c_4a_2^{2p_2}L^{2p_2}B^{1-p_2}t}{\lambda^{p_2} n} +
        \frac{c_5 t}{n}.
    \end{align*}
\end{proposition}
\begin{proof} Recall that $\mathcal{S}_{{\bf z}, 2}(\pi(f_{\bf z}))= U_\mathbf{z}(h_{\pi(f_\mathbf{z})} - h_{f^*_\phi})$.
    Denote 
    \[\omega := (\log 4)C_5B \bigg[  \bigg(\frac{3a_1^2L^2}{2\lambda B} \bigg)^{p_1}\frac{4^{1-2p_1}}{1-2p_1}
    +  \bigg(\frac{3a_2^2L^2}{4\lambda B}\bigg)^{p_2}\frac{4^{1-2p_2}}{1-2p_2} \bigg] 
    +  \frac{(1+\log 4\sqrt{3})C_5B}{4}.\]
    Combine Lemma \ref{lemma 6}, Lemma \ref{lemma 7} and Lemma \ref{lemma 8} we have 
    \begin{align*}
        &\mathbb{E}\exp \left(\xi \sqrt{\vert (n-1)U_\mathbf{z}(h_{\pi(f_\mathbf{z})} - h_{f^*_\phi})\vert }\right) \\
        \leq{}& C_1\mathbb{E}\mathbb{E}_{\varepsilon}\exp \left(C_2\xi \sqrt{\sup_{f \in \mathcal{F}} 
        \vert (n-1)S_\mathbf{z}(h_{\pi(f)} - h_{f^*_\phi})\vert }\right) \\
        \leq{}& C_3\mathbb{E}\exp \left(C_4\xi^2 \mathbb{E}_{\varepsilon}\sup_{f \in \mathcal{F}} 
        \vert (n-1)S_\mathbf{z}(h_{\pi(f)} - h_{f^*_\phi})\vert \right) \leq C_3\exp \left(C_4 \xi^2 \omega\right). 
    \end{align*} 
    By Markov's inequality we conclude that for all $t'>0$ and $\xi>0$, we have
    \begin{equation*}
        P\left(\vert (n-1)U_\mathbf{z}(h_{\pi(f_\mathbf{z})} - h_{f^*_\phi})\vert \geq t'\right) \leq C_3\exp (C_4 \xi^2 \omega - \xi\sqrt{t'}).
    \end{equation*}  
    Let $t' = 4C_4\omega t$ and $\xi = \frac{\sqrt{t'}}{2C_4\omega}$, we have 
    \begin{equation*}
    \begin{split}
    &P\left(\vert U_\mathbf{z}(h_{\pi(f_\mathbf{z})} - h_{f^*_\phi})\vert  \geq \frac{8C_4\omega t}{n} \right) \\
    \leq{}& P\left(\vert (n-1)U_\mathbf{z}(h_{\pi(f_\mathbf{z})} - h_{f^*_\phi})\vert  \geq 4C_4\omega t\right) \\
    \leq{}& C_3\exp(-t).
    \end{split}
    \end{equation*}
This completes the proof of Proposition \ref{proposition 3}.
\end{proof}

Now we are in the position to prove Theorem \ref{Oracle inequality of pairwise ranking with a general loss} 
which presents the general oracle inequality for excess $\phi$-ranking risk and its variant in Theorem \ref{Oracle inequality of pairwise ranking with a margin-base loss} for margin-based losses.

\begin{proof}[Proof of Theorem \ref{Oracle inequality of pairwise ranking with a general loss}] 
    By error decomposition \eqref{error decomposition} and Hoeffding's decomposition \eqref{Hoeffding's decomposition} we have 
    \begin{align*}
        &\lambda\|f_\mathbf{z}\|^2_{K} + \mathcal{R}^\phi(\pi(f_\mathbf{z})) - \mathcal{R}^\phi(f^*_\phi) \\
        \leq{}& \big(\lambda\|f_0\|^2_{K} + \mathcal{R}^\phi(f_0) - \mathcal{R}^\phi(f^*_\phi)\big)  
        + \big(\mathcal{R}^\phi_\mathbf{z}(f_0) - \mathcal{R}^\phi_\mathbf{z}(f^*_\phi) - \mathcal{R}^\phi(f_0) + \mathcal{R}^\phi(f^*_\phi)\big) \\
        &+ 2Q_\mathbf{z}[\mathcal{R}^\phi(\pi(f_\mathbf{z})) - \mathcal{R}^\phi(f^*_\phi) - Q\phi_{\pi(f_\mathbf{z})} + Q\phi_{f^*_\phi}] - U_\mathbf{z}(h_{\pi(f_\mathbf{z})} - h_{f^*_\phi})\\
        ={}&\mathcal{A}(\lambda, f_0) + \mathcal{S}_{\bf z}(f_0) + 2\mathcal{S}_{{\bf z}, 1}(\pi(f_{\bf z}))-\mathcal{S}_{{\bf z}, 2}(\pi(f_{\bf z})).
    \end{align*} 

    If $n < 72t$, by \eqref{sup-norm bound} of Assumption \ref{assumption 1} 
    and $\|\phi_{f_0}\|_{\infty} \leq B_0$, we have
    \begin{align*}
        &\lambda\|f_\mathbf{z}\|^2_{K} + \mathcal{R}^\phi(\pi(f_\mathbf{z})) - \mathcal{R}^\phi(f^*_\phi) \\ 
        \leq{}& \big(\lambda\|f_0\|^2_{K} + \mathcal{R}^\phi(f_0) - \mathcal{R}^\phi(f^*_\phi)\big) + (B_0 + B) + (2B + 2B) + 4B\\
        ={}& \big(\lambda\|f_0\|^2_{K} + \mathcal{R}^\phi(f_0) - \mathcal{R}^\phi(f^*_\phi)\big) +  9B + B_0\\
        \leq{}& 8\big(\lambda\|f_0\|^2_{K} + \mathcal{R}^\phi(f_0) - \mathcal{R}^\phi(f^*_\phi)\big) + \frac{900Bt}{n} + \frac{456B_0t}{n}.
    \end{align*}
    The conclusion then follows immediately.

    If $n \geq 72t$,
    by Proposition \ref{proposition 1}, with probability at most $4\exp(-t)$, there holds
    \[
    \mathcal{S}_{\bf z}(f_0)
    \geq \mathcal{R}^\phi(f_0) - \mathcal{R}^\phi(f^*_\phi) + \left(\frac{8V t}{n}\right)^{\frac{1}{2-\tau}}  
    + \frac{300Bt}{n} + \frac{152B_0t}{n}.
    \]
    By Proposition \ref{proposition 2}, with probability at most $\exp(-t)$, there holds
    \begin{align*}
        \mathcal{S}_{{\bf z}, 1}(\pi(f_{\bf z}))\geq{}& \frac{1}{3}\left(\lambda\|f_\mathbf{z}\|^2_{K} + \mathcal{R}^\phi(\pi(f_\mathbf{z})) - \mathcal{R}^\phi(f^*_\phi)\right)
        + \frac{1}{3}\left(\lambda\|f_0\|^2_{K} + \mathcal{R}^\phi(f_0) - \mathcal{R}^\phi(f^*_\phi)\right) \notag \\
        &+  c_1 \left(\frac{a_1^{2p_1}V^{1-p_1}L^{2p_1}}{\lambda^{p_1} n}\right)^{\frac{1}{2-\tau-p_1+p_1\tau}} 
        +  \frac{c_2a_1^{2p_1}B^{1-p_1}L^{2p_1}}{\lambda^{p_1}n}
        + \left(\frac{24V t}{n}\right)^{\frac{1}{2-\tau}}.
    \end{align*}
    By Proposition \ref{proposition 3}, with probability at most $c_0\exp(-t)$, there holds
    \begin{align*}
        \vert \mathcal{S}_{{\bf z}, 2}(\pi(f_{\bf z}))\vert  \geq 
        \frac{c_3a_1^{2p_1}L^{2p_1}B^{1-p_1}t}{\lambda^{p_1} n} + 
        \frac{c_4a_2^{2p_2}L^{2p_2}B^{1-p_2}t}{\lambda^{p_2} n} +
        \frac{c_5 t}{n}
    \end{align*}
    Finally, combing all the estimates above, with probability at least $1-(c_0 + 5)\exp(-t)$, we have
    \begin{align*}
        &\lambda\|f_\mathbf{z}\|^2_{K} + \mathcal{R}^\phi(\pi(f_\mathbf{z})) - \mathcal{R}^\phi(f^*_\phi) \\
        \leq{}& \big(\lambda\|f_0\|^2_{K} + \mathcal{R}^\phi(f_0) - \mathcal{R}^\phi(f^*_\phi)\big)  
        + \mathcal{R}^\phi(f_0) - \mathcal{R}^\phi(f^*_\phi) + \left(\frac{8V t}{n}\right)^{\frac{1}{2-\tau}}  
        + \frac{300Bt}{n} + \frac{152B_0t}{n} \\
        &+ \frac{2}{3}\left(\lambda\|f_\mathbf{z}\|^2_{K} + \mathcal{R}^\phi(\pi(f_\mathbf{z})) - \mathcal{R}^\phi(f^*_\phi)\right)
        + \frac{2}{3}\left(\lambda\|f_0\|^2_{K} + \mathcal{R}^\phi(f_0) - \mathcal{R}^\phi(f^*_\phi)\right) \\
        &+  2c_1 \left(\frac{a_1^{2p_1}V^{1-p_1}L^{2p_1}}{\lambda^{p_1} n}\right)^{\frac{1}{2-\tau-p_1+p_1\tau}} 
        +  \frac{2c_2a_1^{2p_1}B^{1-p_1}L^{2p_1}}{\lambda^{p_1}n} + \left(\frac{96V t}{n}\right)^{\frac{1}{2-\tau}} \\
        &+ \frac{c_3a_1^{2p_1}L^{2p_1}B^{1-p_1}t}{\lambda^{p_1} n} + 
        \frac{c_4a_2^{2p_2}L^{2p_2}B^{1-p_2}t}{\lambda^{p_2} n} +
        \frac{c_5 t}{n}.
    \end{align*}
    The proof is then finished.
\end{proof}
\begin{proof}[Proof of Theorem \ref{Oracle inequality of pairwise ranking with a margin-base loss}]
    Using the conditions on $\psi$ and the analysis conducted in subsection 
    \ref{Subsection: Calibration Inequality for Pairwise Ranking}, we see that Assumption \ref{assumption 1}
    is satisfied and hence we apply Theorem
    \ref{Oracle inequality of pairwise ranking with a general loss} to yield the conclusion.
\end{proof}

\section{Deriving the Learning Rates for Gaussian Ranking Estimators}\label{Section: Deriving the Learning Rates for Gaussian Ranking Estimators}

In this section, we apply the oracle inequality to derive fast learning rates of Gaussian ranking estimator \eqref{fz} with hinge loss and square loss. 
We first estimate the capacity of pairwise Gaussian kernel space $\mathcal{H}_{K^\sigma}$ under the assumption that 
the marginal of the data generating distribution on the input space $\mathcal{X}\subset \mathbb{R}^d$ is supported on a set of upper box-counting dimension $\varrho \in(0, d]$. 
Then, we derive approximation error bounds for the hinge loss under noise conditions and the square loss under the Besov smoothness. 
Finally, we apply the well-established oracle inequality and calibration inequality to derive fast learning rates for the excess ranking risk. 
We also make some comparisons of different noise conditions at the end of this section.

\subsection{Entropy Number Estimate of Pairwise Gaussian Kernel Spaces}\label{Subsection: Entropy Number Estimate of Pairwise Gaussian Kernel Spaces}

In this subsection, we write $\mathcal{H}_{K}(\mathcal{X}^2)$ to emphasize that the RKHS induced by the pairwise reproducing kernel $K$ consists of functions defined on $\mathcal{X}^2=\mathcal{X} \times \mathcal{X}$, 
or equivalently, $K$ is defined or restricted on $\mathcal{X}^2 \times \mathcal{X}^2$, where $\mathcal{X}$ will be specified in the context. 
The traditional Gaussian kernel on $\mathbb{R}^{2d} \times \mathbb{R}^{2d}$ with variance $\sigma > 0$ is defined as
\begin{equation}\label{traditionalgaussian}
\widetilde{K}^{\sigma}((x,x'),(u,u')) := \exp\left(-\frac{\|(x,x') - (u,u')\|^2_2}{\sigma^2}\right)
\end{equation}
and we denote by $\mathcal{H}_{\widetilde{K}^\sigma}(\mathcal{X}^2)$ the induced RKHS restricted on $\mathcal{X}^2$, cf. \cite{Aronszajn1950Theory}. 
Note that we require the ranking estimators to be skew-symmetric. 
It is not suitable to choose $\mathcal{H}_{\widetilde{K}^\sigma}(\mathcal{X}^2)$ as the hypothesis space directly. 
However, we can decompose $\widetilde{f} \in \mathcal{H}_{\widetilde{K}^\sigma}(\mathcal{X}^2)$  into a symmetric part 
$(\widetilde{f} - f)$ and a skew-symmetric part $f$:
\begin{equation}\label{orthogonal decomposition}
    \widetilde{f}(x,x') = 
    \frac{\widetilde{f}(x,x') + \widetilde{f}(x',x)}{2} + \frac{\widetilde{f}(x,x') - \widetilde{f}(x',x)}{2}
    =: (\widetilde{f} - f) + f
\end{equation}
which leads to  an orthogonal decomposition of $\mathcal{H}_{\widetilde{K}^\sigma}(\mathcal{X}^2)$ in $\mathcal{L}_2(P^2_{\mathcal{X}})$ and also a decomposition of the traditional Gaussian kernel given by 
\begin{align*}
    \widetilde{K}^{\sigma}((x,x')(u,u'))&=\frac{\widetilde{K}^{\sigma}((x,x'),(u,u')) + \widetilde{K}^{\sigma}((x',x),(u,u'))}{2} \\
    &\qquad + \frac{\widetilde{K}^{\sigma}((x,x'),(u,u')) - \widetilde{K}^{\sigma}((x',x),(u,u'))}{2}.
\end{align*}
The skew-symmetric part of $\widetilde{K}^{\sigma}((x,x')(u,u'))$ yields the pairwise Gaussian kernel $K^{\sigma}((x,x'),(u,u'))$ of the form \eqref{pairwisegaussian}
considered in this paper. 
According to the properties of reproducing kernels, cf. \cite{Aronszajn1950Theory}, $\mathcal{H}_{K^\sigma}(\mathcal{X}^2)$ is a subspace of $\mathcal{H}_{\widetilde{K}^\sigma}(\mathcal{X}^2)$ 
with $\|f\|_{K^\sigma} = \|f\|_{\widetilde{K}^\sigma}$ for all $f \in \mathcal{H}_{K^\sigma}(\mathcal{X}^2)$. 
Therefore, we can carry out the entropy number estimate of $\mathcal{H}_{K^\sigma}(\mathcal{X}^2)$ by using the existing estimates developed for 
$\mathcal{H}_{\widetilde{K}^\sigma}(\mathcal{X}^2)$. Concretely,
\cite{Kuhn2011Coveringa} showed that 
\[\log \mathcal{N}(\mathrm{id} : \mathcal{H}_{\widetilde{K}^\sigma}([0,1]^{2d}) \to C([0,1]^{2d}), \varepsilon)
\asymp \frac{\left(\log \frac{1}{\varepsilon}\right)^{2d+1}}{{\left(\log \log \frac{1}{\varepsilon}\right)^{2d}}}, 
\quad \varepsilon \to 0.\]
A refined analysis in \cite{Steinwart2021closer} clarified the  crucial dependency of constants on $\sigma$
and the underlying space $\mathcal{X} \subset \mathbb{R}^d$:
\begin{equation}\label{covering number estimate of Gaussian kernel}
\begin{split}
    &\log \mathcal{N}(\mathrm{id}: \mathcal{H}_{\widetilde{K}^\sigma}(\mathcal{X}^2) \to C(\mathcal{X}^2) ,\varepsilon)\\
    {}& \leq \mathcal{N}(\mathcal{X}^2, \|\cdot\|_\infty, \sigma) \cdot 
    \binom{4e + 2d}{2d}e^{-2d}\frac{\left(\log \frac{1}{\varepsilon}\right)^{2d+1}}{{\left(\log \log \frac{1}{\varepsilon}\right)^{2d}}}.
\end{split}
\end{equation}

Based on these results above, we obtain the following proposition under the intrinsic dimension assumption of $\mathcal{X}$.
\begin{proposition}\label{entropy number estimate of pairwise Gaussian kernel}
    If Assumption \ref{Upper box-counting dimension condition} holds, then $\mathcal{H}_{K^\sigma}(\mathcal{X}^2)$ satisfies Assumption \ref{Capacity condition} 
    for all $p_1 = p_2 = p \in (0,1/2)$ and 
    $a_1 = a_2 = a:= (C^*_\mathcal{X} p^{-2d-1}\sigma^{-2\varrho})^{\frac{1}{2p}}.$
\end{proposition}
\begin{proof}
    First we show that \[\mathcal{N}(\mathrm{id}: \mathcal{H}_{K^\sigma}(\mathcal{X}^2) \to C(\mathcal{X}^2) ,\varepsilon) \leq 
    \mathcal{N}(\mathrm{id}: \mathcal{H}_{\widetilde{K}^\sigma}(\mathcal{X}^2) \to C(\mathcal{X}^2) ,\varepsilon).\] 
    Recall that $\mathcal{B}_{\mathcal{F}}$ denotes the unit ball of a normed space $(\mathcal{F},\|\cdot\|_{\mathcal{F}})$.
    Given an $\varepsilon$-net $\{\widetilde{f}_1, \ldots , \widetilde{f}_m\}$ of 
    $(\mathcal{B}_{\mathcal{H}_{\widetilde{K}^\sigma}(\mathcal{X}^2)} , \|\cdot\|_\infty)$, 
    by the decomposition \eqref{orthogonal decomposition}, 
    it induces a set of skew-symmetric functions $\{f_1, \ldots , f_m\} \subset \mathcal{B}_{\mathcal{H}_{K^\sigma}(\mathcal{X}^2)}$.
    Let $f \in \mathcal{B}_{\mathcal{H}_{K^\sigma}(\mathcal{X}^2)} \subset \mathcal{B}_{\mathcal{H}_{\widetilde{K}^\sigma}(\mathcal{X}^2)}$.
    There exists a $\widetilde{f}_i \in \{\widetilde{f}_1, \ldots , \widetilde{f}_m\}$ such that 
    $\|\widetilde{f}_i - f\|_\infty \leq \varepsilon$.
    Notice that $f$ is skew-symmetric. Then due to the decomposition \eqref{orthogonal decomposition},
    we have $\|f_i - f\|_\infty \leq \|\widetilde{f}_i - f\|_\infty \leq \varepsilon$.
    Hence $\{f_1, \ldots , f_m\}$ is a $\varepsilon$-net of $(B_{\mathcal{H}_{K^\sigma}(\mathcal{X}^2)} , \|\cdot\|_\infty)$. 
    Then we have the following estimate for all $p \in (0, 1/2)$ by \eqref{low 2} and \eqref{covering number estimate of Gaussian kernel}.
    \begin{align*}
        &\log \mathcal{N}(\mathrm{id}: \mathcal{H}_{K^\sigma}(\mathcal{X}^2) \to C(\mathcal{X}^2),\varepsilon) \\
        \leq{}& \log \mathcal{N}(\mathrm{id}: \mathcal{H}_{\widetilde{K}^\sigma}(\mathcal{X}^2) \to C(\mathcal{X}^2) ,\varepsilon) \\
        \leq{}& \mathcal{N}(\mathcal{X}^2, \|\cdot\|_\infty, \sigma) \cdot 
        \binom{4e + 2d}{2d}e^{-2d}\frac{\left(\log \frac{4}{\varepsilon}\right)^{2d+1}}{{\left(\log \log \frac{4}{\varepsilon}\right)^{2d}}} \\
        \leq{}& C^2_\mathcal{X}\sigma^{-2\varrho}
        \binom{4e + 2d}{2d}e^{-2d}\frac{\left(\log \frac{4}{\varepsilon}\right)^{2d+1}}{{\left(\log \log \frac{4}{\varepsilon}\right)^{2d}}} \\
        \leq{}& C^2_\mathcal{X}\sigma^{-2\varrho}
        \binom{4e + 2d}{2d}e^{-2d}\left(\log \frac{4}{\varepsilon}\right)^{2d+1} \\
        \leq{}& C^2_\mathcal{X}\sigma^{-2\varrho}
        \binom{4e + 2d}{2d}e^{-2d} \left(\frac{2d+1}{2p}\right)^{2d+1} 4^{2p} e^{-2d-1} \varepsilon^{-2p} \\
        \leq{}& 4C^2_\mathcal{X}\binom{4e + 2d}{2d}\frac{(2d+1)^{2d+1}}{2^{2d+1}e^{4d+1}}
        p^{-2d-1}\sigma^{-2\varrho}\varepsilon^{-2p}.
    \end{align*}
    We then use the above estimate to establish a bound for the entropy number, cf. Exercise 6.8 of \cite{Steinwart2008Support}, which is given by
    \[e_i(\mathrm{id} : \mathcal{H}_{K^\sigma}(\mathcal{X}^2) \to C(\mathcal{X}^2)) \leq 
    (C^*_\mathcal{X} p^{-2d-1}\sigma^{-2\varrho})^{\frac{1}{2p}} i^{-\frac{1}{2p}}\]
    where 
    \[C^*_{\mathcal{X}}:= 12C^2_{\mathcal{X}}\binom{4e + 2d}{2d}\frac{(2d+1)^{2d+1}}{2^{2d+1}e^{4d+1}}.\]
    Notice that $\mathcal{L}_2$-seminorm 
    $\| \cdot \|_{\mathcal{L}_2(P^n_{\mathbf{x}} \otimes P_{\mathcal{X}})}$
    and $\| \cdot \|_{\mathcal{L}_2(P^n_{\mathbf{x}^2})}$ are both dominated by the sup-norm $\|\cdot\|_\infty$. Therefore, Assumption \ref{Capacity condition} holds for all $p_1 = p_2 = p \in (0,1/2)$ and 
    $a_1 = a_2 = a :=(C^*_\mathcal{X} p^{-2d-1}\sigma^{-2\varrho})^{\frac{1}{2p}}.$ The proof is then finished.
\end{proof}

\subsection{Bounding the Approximation Error}\label{Subsection: Bounding the Approximation Error}

In this subsection, we aim to construct a suitable $f_0 \in \mathcal{H}_{K^\sigma}$ and bound the approximation error
$\lambda\|f_0\|^2_{K^{\sigma}} + \mathcal{R}^\phi(f_0) - \mathcal{R}^\phi(f^*_\phi)$ for $\phi=\phi_{\mathrm{hinge}}$ and $\phi=\phi_{\mathrm{square}}$. 
Inspired by the work of \cite{Steinwart2008Support, Xiang2009Classification}, define
\[\widetilde{\mathcal{K}}^\sigma(x,x'):= \left(\frac{2}{\sigma\sqrt{\pi}}\right)^{d}\exp\left(-\frac{2\|(x,x')\|^2_2}{\sigma^2}\right)\] 
and an operator $\widetilde{\mathcal{K}}^\sigma * \cdot$ on $\mathcal{L}_2(\mathbb{R}^{2d})$ as 
\[\widetilde{\mathcal{K}}^\sigma * f(x,x') := \int_{\mathbb{R}^{2d}} 
\widetilde{\mathcal{K}}^\sigma(u,u')f((x,x')+(u,u')) d(u,u'), \forall f \in \mathcal{L}_2(\mathbb{R}^{2d}).\]

One can verify that 
\begin{align*}
    &\widetilde{\mathcal{K}}^\sigma * f(x,x') \\
    &=  \left(\frac{1}{\sigma\sqrt{\pi}}\right)^{d}\int_{\mathbb{R}^{2d}} 
    \widetilde{K}^{\sigma}((x,x'),(u,u')) f\bigg(\bigg(1- \frac{1}{\sqrt{2}}\bigg)(x,x') + \frac{1}{\sqrt{2}}(u,u')\bigg) d(u,u')
\end{align*}
where $\widetilde{K}^{\sigma}$ is the traditional Gaussian kernel given by \eqref{traditionalgaussian}. 
Therefore, $\widetilde{\mathcal{K}}^\sigma * \cdot$ is essentially a convolution operator which is a metric surjection from $\mathcal{L}_2(\mathbb{R}^{2d})$ to $\mathcal{H}_{\widetilde{K}^\sigma}$, 
i.e., $\widetilde{\mathcal{K}}^\sigma * f\in \mathcal{H}_{\widetilde{K}^\sigma}$ and $\|\widetilde{\mathcal{K}}^\sigma * f\|_{\widetilde{K}^{\sigma}} \leq \|f\|_{\mathcal{L}_2(\mathbb{R}^{2d})}$. 
This motivates us to consider the following operator on $\mathcal{L}_2(\mathbb{R}^{2d})$ defined by
\[\mathcal{K}^{\sigma}*f(x,x') := \frac{1}{2}\int_{\mathbb{R}^{2d}}\widetilde{\mathcal{K}}^\sigma(u,u')\bigg(f((x,x') + (u,u')) - f((x',x) + (u,u'))\bigg) d(u,u').\]
In fact, for all $f \in \mathcal{L}_2(\mathbb{R}^{2d})$, $\mathcal{K}^\sigma * f(x,x')$ is skew-symmetric, i.e., $\mathcal{K}^\sigma * f(x,x')=-\mathcal{K}^\sigma * f(x',x)$, 
and if $f$ is skew-symmetric, $\mathcal{K}^\sigma * f(x,x') = \widetilde{\mathcal{K}}^\sigma * f(x,x').$ Moreover, if $f$ is symmetric, we have $\mathcal{K}^\sigma * f(x,x') = 0$.
For all $\widetilde{f} \in \mathcal{L}_2(\mathbb{R}^{2d})$, by decomposition \eqref{orthogonal decomposition},  
we write 
$\widetilde{f} = (\widetilde{f} - f) + f$ where $f$ is skew-symmetric while $\widetilde{f} - f$ is symmetric. Then we have 
\[\|\mathcal{K}^\sigma * \widetilde{f}\|_{K^\sigma} 
= \|\mathcal{K}^\sigma *  f\|_{K^\sigma}
= \|\widetilde{\mathcal{K}}^\sigma *  f\|_{\widetilde{K}^\sigma}
\leq \|f\|_{\mathcal{L}_2(\mathbb{R}^{2d})} \leq \|\widetilde{f}\|_{\mathcal{L}_2(\mathbb{R}^{2d})}.\]
Hence $\mathcal{K}^\sigma * \cdot$ is  a metric surjection from $\mathcal{L}_2(\mathbb{R}^{2d})$ to $\mathcal{H}_{K^\sigma}$. 
Note that the Bayes ranking rule $f^*_\phi$ of $\phi=\phi_{\mathrm{hinge}}$ and $\phi=\phi_{\mathrm{square}}$, which we need to approximate, are skew-symmetric. 
In the followings, we will act the convolution operator $\widetilde{\mathcal{K}}^\sigma * \cdot$ on skew-symmetric functions to construct efficient approximations of $f^*_\phi$ in $\mathcal{H}_{K^{\sigma}}$ 
and then carry out the approximation analysis.

We first deal with hinge loss. Under margin-noise condition Assumption \ref{Margin-noise condition}, 
we establish the following proposition to bound the approximation error.
\begin{proposition}\label{hinge loss approximation}
    Denote $\mathcal{B}_{\mathbb{R}^{2d}}$ the unit closed ball in $\mathbb{R}^{2d}$. 
    Suppose that $\mathcal{X}^2 \subset r\mathcal{B}_{\mathbb{R}^{2d}}$ for $r>0$ and Assumption \ref{Margin-noise condition} holds with $C_{**}>0$ and $\beta > 0$. 
    Then for all $\lambda > 0, \sigma > 0$, there exists an 
    $f_0 \in \mathcal{H}_{K^\sigma}$ such that $\|f_0\|_\infty \leq 1$ and 
    \begin{align*}
    \lambda\|f_0\|^2_{K^\sigma} + \mathcal{R}^{\phi_\mathrm{hinge}}(f_0) - \mathcal{R}^{\phi_\mathrm{hinge}}(f^*_\mathrm{hinge})\leq \frac{2^{2d+1}r^{2d}}{\Gamma(d)}\lambda \sigma^{-2d} 
    + \frac{2^{\beta/2+1}C_{**}\Gamma(d + \frac{\beta}{2})}{\Gamma(d)}\sigma^\beta.
    \end{align*}
\end{proposition}
\begin{proof}
    Recall that the Bayes ranking rule for hinge loss $f^*_{\mathrm{hinge}} : \mathcal{X}^2 \to \mathbb{R}$ takes the form 
    \[f^*_{\mathrm{hinge}}(x, x') = \mathrm{sgn}(\eta_+(x,x') - \eta_-(x,x')).\]
    Denote $\mathrm{vol}(\cdot)$ the Lebesgue measure on Euclidean space. 
    Note that if $\mathrm{vol}(\mathcal{X}) = 0$, for example, when $\mathcal{X}$ is a submanifold with upper box-counting dimension $\varrho < d$, 
    the trivial zero extension of $f^*_{\phi}$ only yields the zero function in $\mathcal{L}_2(\mathbb{R}^{2d})$. Hence, before we construct the approximator using $\widetilde{\mathcal{K}}^\sigma * f^*_{\mathrm{hinge}}$,
    we need to extend $f^*_{\mathrm{hinge}}$ to a domain with a positive measure in a nontrivial way, while still maintaining the margin-noise condition.

    To this end, we extend the posterior probabilities $\eta_+, \eta_-, \eta_=: \mathcal{X}^2 \to [0,1]$
    to $2r\mathcal{B}_{\mathbb{R}^{2d}}$.
    For all $(x,x') \in \mathcal{X}^2$, with a slight abuse of notation, denote $\mathcal{B}_{\mathbb{R}^{2d}}((x,x'), \Delta(x,x')/2)$ the open ball
    centered at $(x,x')$ with radius $\Delta(x,x')/2$, where $\Delta(x,x')$ is defined as \eqref{Delta}. Since $\mathcal{X}^2 \subset r\mathcal{B}_{\mathbb{R}^{2d}}$, we have $\Delta(x,x') \leq 2r$ and 
    $\mathcal{B}_{\mathbb{R}^{2d}}((x,x'), \Delta(x,x')/2) \subset 2r\mathcal{B}_{\mathbb{R}^{2d}}$.
    For all $(u,u') \in 2r\mathcal{B}_{\mathbb{R}^{2d}} \backslash \mathcal{X}^2$, define the extension as 
    \begin{align*}
        &\eta_+(u,u') = 1,\eta_-(u,u') = \eta_=(u,u') =0,\\
        &\qquad \forall (u,u') \in \bigcup_{(x,x') \in \mathcal{X}^2_+}\mathcal{B}_{\mathbb{R}^{2d}}((x,x'), \Delta(x,x')/2), \\
        &\eta_-(u,u') = 1,\eta_+(u,u') = \eta_=(u,u') =0,\\
        &\qquad \forall (u,u') \in \bigcup_{(x,x') \in \mathcal{X}^2_-}\mathcal{B}_{\mathbb{R}^{2d}}((x,x'), \Delta(x,x')/2), \\
        &\eta_=(u,u') = 1,\eta_+(u,u') = \eta_-(u,u') =0, \\
        &\qquad \forall (u,u') \in 2r\mathcal{B}_{\mathbb{R}^{2d}} \backslash \bigcup_{(x,x') \in \mathcal{X}^2_+ \cup \mathcal{X}^2_-}\mathcal{B}_{\mathbb{R}^{2d}}((x,x'), \Delta(x,x')/2).
    \end{align*}
    Since
    $\bigcup_{(x,x') \in \mathcal{X}^2_+}\mathcal{B}_{\mathbb{R}^{2d}}((x,x'), \Delta(x,x')/2)$ and $\bigcup_{(x,x') \in \mathcal{X}^2_-}\mathcal{B}_{\mathbb{R}^{2d}}((x,x'), \Delta(x,x')/2)$
    have an empty intersection by definition of $\Delta(x,x')$ and triangle inequality,
    the extension is well-defined. 
    Using extended $\eta_+, \eta_-, \eta_=: 2r\mathcal{B}_{\mathbb{R}^{2d}} \to [0,1]$, define
    \begin{align*}
    (2r\mathcal{B}_{\mathbb{R}^{2d}})_+ &:= \{(x,x') \in 2r\mathcal{B}_{\mathbb{R}^{2d}}: \eta_+(x,x') > \eta_-(x,x')\}, \\ 
    (2r\mathcal{B}_{\mathbb{R}^{2d}})_- &:= \{(x,x') \in 2r\mathcal{B}_{\mathbb{R}^{2d}}: \eta_+(x,x') < \eta_-(x,x')\}, \\
    (2r\mathcal{B}_{\mathbb{R}^{2d}})_= &:= \{(x,x') \in 2r\mathcal{B}_{\mathbb{R}^{2d}}: \eta_+(x,x') = \eta_-(x,x')\},
    \end{align*} 
    and a new distant function $\widetilde{\Delta}: \mathcal{X}^2 \to \mathbb{R}$ as 
    \[\widetilde{\Delta}(x,x') := 
    \begin{cases}
        \mathrm{dist}((x,x'), (2r\mathcal{B}_{\mathbb{R}^{2d}})_- \cup (2r\mathcal{B}_{\mathbb{R}^{2d}})_= ) & (x,x') \in \mathcal{X}^2_+  \\
        \mathrm{dist}((x,x'), (2r\mathcal{B}_{\mathbb{R}^{2d}})_+ \cup (2r\mathcal{B}_{\mathbb{R}^{2d}})_= ) & (x,x') \in \mathcal{X}^2_-  \\
        0 & (x,x') \in \mathcal{X}^2_=.
    \end{cases}\] Here $\mathcal{X}^2_+, \mathcal{X}^2_-$ and $\mathcal{X}^2_=$ are defined in \eqref{seperation} with original $\eta_+, \eta_-$ and $\eta_=$.

    We claim that with the redefined $\widetilde{\Delta}$, the margin-noise condition holds with the power index $\beta$ and constant $2^{\beta}C_{**}$.
    To see this, one may notice that for all $(x,x') \in \mathcal{X}^2_+$, if $\widetilde{\Delta}(x,x') < \Delta(x,x')/2$, 
    then there exists a $(u,u') \in (2r\mathcal{B}_{\mathbb{R}^{2d}})_-$ or $(u,u') \in (2r\mathcal{B}_{\mathbb{R}^{2d}})_=$ such that $\mathrm{dist}((x,x'),(u,u')) < \Delta(x,x')/2$.
    If $(u,u') \in (2r\mathcal{B}_{\mathbb{R}^{2d}})_-$, then it must be the case that $(u,u') \in \mathcal{B}_{\mathbb{R}^{2d}}((v,v'), \Delta(v,v')/2)$ for some $(v,v') \in \mathcal{X}^2_-$.
    By triangle inequality we have 
    \begin{align*}
        \mathrm{dist}((x,x'),(v,v'))& \leq \mathrm{dist}((x,x'),(u,u')) + \mathrm{dist}((u,u'),(v,v'))\\
        &< \Delta(x,x')/2 + \Delta(v,v')/2,
    \end{align*}
    which is a contradiction to the definition of $\Delta$.
    If $(u,u') \in (2r\mathcal{B}_{\mathbb{R}^{2d}})_=$, then 
    \[(u,u') \in 2r\mathcal{B}_{\mathbb{R}^{2d}} \backslash \bigcup_{(v,v') \in \mathcal{X}^2_+ \cup \mathcal{X}^2_-}\mathcal{B}_{\mathbb{R}^{2d}}((v,v'), \Delta(v,v')/2),\]
    which also contradicts to the fact that $\mathrm{dist}((x,x'),(u,u')) < \Delta(x,x')/2$. Following the same argument, one can
    discuss the case when $(x,x') \in \mathcal{X}^2_-$.
    Hence we conclude that $\widetilde{\Delta}(x,x') \geq \Delta(x,x')/2$ for all $(x,x') \in \mathcal{X}^2$.
    Then the margin-noise condition holds for $\widetilde{\Delta}$, given by
    \begin{equation}\label{new margin-noise condition}
        \begin{split}
            &\int_{\{(x,x') \in \mathcal{X}^2 : \widetilde{\Delta}(x,x') < t\}} \vert\eta_+(x,x') - \eta_-(x,x')\vert dP_\mathcal{X}^2(x,x') \\
            \leq{}& \int_{\{(x,x') \in \mathcal{X}^2 : \Delta(x,x') < 2t\}} \vert\eta_+(x,x') - \eta_-(x,x')\vert dP_\mathcal{X}^2(x,x') \\
            \leq{}& 2^{\beta}C_{**}t^\beta.
        \end{split}
    \end{equation}

    Now we extend the original $f^*_{\mathrm{hinge}}$ as 
    \[f^*_{\mathrm{hinge}}(x,x') = \mathbb{I}_{2r\mathcal{B}_{\mathbb{R}^{2d}}}(x,x')\mathrm{sgn}(\eta_+(x,x') - \eta_-(x,x')).\]
    One can verify that $f^*_{\mathrm{hinge}}$ maintains skew-symmetric and $f^*_{\mathrm{hinge}} \in \mathcal{L}_2(\mathbb{R}^{2d})$ with 
    $\|f^*_{\mathrm{hinge}}\|_{\mathcal{L}_2(\mathbb{R}^{2d})} \leq \sqrt{\mathrm{vol}(2r\mathcal{B}_{\mathbb{R}^{2d}})} = (2r)^d\sqrt{\mathrm{vol}(\mathcal{B}_{\mathbb{R}^{2d}})}.$ 
    We choose $f_0 \in \mathcal{H}_{K^\sigma}$ as 
    \[f_0(x,x') = (\pi \sigma^2)^{-d/2} \widetilde{\mathcal{K}}^\sigma * f^*_{\mathrm{hinge}}(x,x').\]
    Since $\widetilde{\mathcal{K}}^\sigma * \cdot$ is a metric surjection, 
    \begin{equation}\label{hinge 1}
        \|f_0\|_{K^\sigma} \leq (\pi \sigma^2)^{-d/2}\|f^*_{\mathrm{hinge}}\|_{\mathcal{L}_2(\mathbb{R}^{2d})} \leq 
        (\pi \sigma^2)^{-d/2}(2r)^d\sqrt{\mathrm{vol}(\mathcal{B}_{\mathbb{R}^{2d}})}.
    \end{equation}    
    Besides, by Young's convolution inequality, there holds
    \begin{equation}\label{hinge 2}
        \|f_0\|_\infty \leq \frac{2^d}{\sigma^{2d}\pi^d} \|f^*_{\mathrm{hinge}}\|_\infty
        \int_{\mathbb{R}^{2d}}
        \exp\left(-\frac{2\|(x,x')\|^2_2}{\sigma^2}\right)d(x,x') = 1.
    \end{equation}

    It remains to bound the excess risk $\mathcal{R}^{\phi_\mathrm{hinge}}(f_0) - \mathcal{R}^{\phi_\mathrm{hinge}}(f^*_{\mathrm{hinge}})$.
    Fix a $(x,x') \in \mathcal{X}^2_+ \subset (2r\mathcal{B}_{\mathbb{R}^{2d}})_+$.
    For all $(u,u') \in \mathcal{B}_{\mathbb{R}^{2d}}((x,x'), \widetilde{\Delta}(x,x'))$, we have $(u,u') \in (2r\mathcal{B}_{\mathbb{R}^{2d}})_+$.
    Hence \[\mathcal{B}_{\mathbb{R}^{2d}}((x,x'), \widetilde{\Delta}(x,x')) \subset  (2r\mathcal{B}_{\mathbb{R}^{2d}})_+, \quad  
    (2r\mathcal{B}_{\mathbb{R}^{2d}})_- \subset \mathbb{R}^{2d} \backslash \mathcal{B}_{\mathbb{R}^{2d}}((x,x'), \widetilde{\Delta}(x,x'))\]
    and
    \begin{align*}
        &f_0(x,x') \\
        ={}&\left(\frac{2}{\pi \sigma^{2}}\right)^d \int_{\mathbb{R}^{2d}} 
        \exp\left(-\frac{2\|(x,x') - (u,u')\|^2_2}{\sigma^2}\right) f^*_{\mathrm{hinge}}(u, u')d(u,u') \\
        ={}&\left(\frac{2}{\pi \sigma^{2}}\right)^{d}\left(\int_{(2r\mathcal{B}_{\mathbb{R}^{2d}})_+} 
        \exp\left(-\frac{2\|(x,x') - (u,u')\|^2_2}{\sigma^2}\right)d(u,u')\right.\\
        &\left. -\int_{(2r\mathcal{B}_{\mathbb{R}^{2d}})_-} 
        \exp\left(-\frac{2\|(x,x') - (u,u')\|^2_2}{\sigma^2}\right)d(u,u') \right) \\
        \geq{}&\left(\frac{2}{\pi \sigma^{2}}\right)^{d}\left(\int_{\mathcal{B}_{\mathbb{R}^{2d}}((x,x'), \widetilde{\Delta}(x,x'))} 
        \exp\left(-\frac{2\|(x,x') - (u,u')\|^2_2}{\sigma^2}\right)d(u,u')\right.\\
        &\left. -\int_{\mathbb{R}^{2d} \backslash \mathcal{B}_{\mathbb{R}^{2d}}((x,x'), \widetilde{\Delta}(x,x'))} 
        \exp\left(-\frac{2\|(x,x') - (u,u')\|^2_2}{\sigma^2}\right)d(u,u') \right) \\
        ={}&2\left(\frac{2}{\pi \sigma^{2}}\right)^{d}\int_{\mathcal{B}_{\mathbb{R}^{2d}}((x,x'), \widetilde{\Delta}(x,x'))} 
        \exp\left(-\frac{2\|(x,x') - (u,u')\|^2_2}{\sigma^2}\right)d(u,u') -1.
    \end{align*}
    By \eqref{hinge 2} and $f^*_{\mathrm{hinge}}(x,x') = 1$, we derive that
    \begin{align*}
        &\vert f_0(x,x') - f^*_{\mathrm{hinge}}(x,x')\vert  \\
        ={}& 1- f_0(x,x') \\
        \leq{}& 2-2\left(\frac{2}{\pi \sigma^{2}}\right)^{d}\int_{\mathcal{B}_{\mathbb{R}^{2d}}((x,x'), \widetilde{\Delta}(x,x'))} 
        \exp\left(-\frac{2\|(x,x') - (u,u')\|^2_2}{\sigma^2}\right)d(u,u') \\
        ={}&2\left(\frac{2}{\pi \sigma^{2}}\right)^{d}\int_{\mathbb{R}^{2d} \backslash \mathcal{B}_{\mathbb{R}^{2d}}((x,x'), \widetilde{\Delta}(x,x'))} 
        \exp\left(-\frac{2\|(x,x') - (u,u')\|^2_2}{\sigma^2}\right)d(u,u') \\
        ={}&\frac{2^{d+2}}{\sigma^{2d}\Gamma(d)}\int_{\widetilde{\Delta}(x,x')}^\infty \exp(-2\sigma^{-2}\rho^2)\rho^{2d-1}d\rho \\
        ={}&\frac{2}{\Gamma(d)}\int_{2\widetilde{\Delta}(x,x')^2\sigma^{-2}}^\infty e^{-\rho}\rho^{d-1}d\rho.
    \end{align*}
    This inequality also holds for all $(x,x') \in \mathcal{X}^2_-$ by the same argument.
    Note that
    \begin{align*}
        \mathcal{R}^{\phi_\mathrm{hinge}}(f_0) - \mathcal{R}^{\phi_\mathrm{hinge}}(f^*_{\mathrm{hinge}}) &= 
        \mathbb{E}[\mathrm{sgn}(Y-Y') \cdot (f^*_{\mathrm{hinge}}(X,X') - f_0(X,X'))] \\
        &=  \mathbb{E}_{P_\mathcal{X}^2}[(\eta_+(X,X') - \eta_-(X,X')) \cdot (f^*_{\mathrm{hinge}}(X,X') - f_0(X,X'))],
    \end{align*}
    we apply our claim \eqref{new margin-noise condition} and yield that 
    \begin{align*}
        &\mathcal{R}^{\phi_\mathrm{hinge}}(f_0) - \mathcal{R}^{\phi_\mathrm{hinge}}(f^*_{\mathrm{hinge}}) \\
        ={}&\int_{\mathcal{X}^2_- \cup \mathcal{X}^2_+} \vert f_0(x,x') - f^*_{\mathrm{hinge}}(x,x')\vert  \cdot \vert \eta_+(x,x') - \eta_-(x,x')\vert dP^2_{\mathcal{X}}(x,x') \\
        \leq{}& \frac{2}{\Gamma(d)}\int_{\mathcal{X}^2_- \cup \mathcal{X}^2_+}
        \int_{2\widetilde{\Delta}(x,x')^2\sigma^{-2}}^\infty e^{-\rho}\rho^{d-1}
        \vert \eta_+(x,x') - \eta_-(x,x')\vert d\rho dP^2_{\mathcal{X}}(x,x') \\
        ={}& \frac{2}{\Gamma(d)}\int_{0}^\infty e^{-\rho}\rho^{d-1} 
        \int_{\mathcal{X}^2_+ \cup \mathcal{X}^2_-} \mathbb{I}_{[0,\sigma(\rho/2)^{1/2})}(\widetilde{\Delta}(x,x'))\vert \eta_+(x,x') - \eta_-(x,x')\vert  dP^2_{\mathcal{X}}(x,x') d\rho \\
        \leq{}& \frac{2^{1+\beta/2}C_{**}\sigma^{\beta}}{\Gamma(d)}\int_{0}^\infty e^{-\rho}\rho^{d-1+\beta/2}d\rho \\
        ={}& \frac{2^{1+\beta/2} C_{**}\Gamma(d + \frac{\beta}{2})}{\Gamma(d)}\sigma^\beta.
    \end{align*}
    Combine with \eqref{hinge 1} \eqref{hinge 2} we conclude that $f_0 \in \mathcal{H}_{K^\sigma}$, $\|f_0\|_\infty \leq 1$ and 
    \[\lambda\|f_0\|^2_{K^\sigma} + \mathcal{R}^{\phi_\mathrm{hinge}}(f_0) - \mathcal{R}^{\phi_\mathrm{hinge}}(f^*_{\mathrm{hinge}}) 
    \leq \frac{2^{2d+1}r^{2d}}{\Gamma(d)}\lambda \sigma^{-2d} 
    + \frac{2^{1+\beta/2}C_{**}\Gamma(d + \frac{\beta}{2})}{\Gamma(d)}\sigma^\beta.\] The proof is then finished.
\end{proof}

Next, we establish the approximation error estimates for the square loss under the standard Besov smoothness condition on the Bayes ranking rule 
\[f^*_{\mathrm{square}}(x,x') = \frac{\eta_+(x,x') - \eta_-(x,x') }{\eta_+(x,x')  + \eta_-(x,x')}.\] Due to the discussion in Remark \ref{remark2}, we can regard $f^*_{\mathrm{square}}$ as a function in $\mathcal{L}_2(\mathbb{R}^{2d})$, 
or equivalently, $f^*_{\mathrm{square}}$ is the restriction of a function in $\mathcal{L}_2(\mathbb{R}^{2d})$ on $\mathcal{X}^2$. 

\begin{proposition}\label{square loss approximation}
    Suppose that $f^*_{\mathrm{square}}\in \mathcal{L}_2(\mathbb{R}^{2d})$ and $\vert f^*_{\mathrm{square}} \vert_{\mathcal{B}_{2, \infty}^{\alpha}(P^2_{\mathcal{X}})} <\infty$ for some $\alpha>0$. 
    Then for all $\lambda > 0, \sigma > 0$, there exists an 
    $f_0 \in \mathcal{H}_{K^\sigma}$ such that $\|f_0\|_\infty \leq 2^{s}$ and 
    \begin{align*}
        &\lambda\|f_0\|^2_{K^{\sigma}} + \mathcal{R}^{\phi_{\mathrm{square}}}(f_0) - \mathcal{R}^{\phi_{\mathrm{square}}}(f^*_{\mathrm{square}}) \\
        \leq{}& \frac{2^{2s}\|f^*_{\mathrm{square}}\|^2_{\mathcal{L}_2(\mathbb{R}^{2d})}}{\pi^d}\lambda\sigma^{-2d} 
        + \frac{\vert f^*_{\mathrm{square}} \vert^2_{\mathcal{B}_{2, \infty}^{\alpha}(P^2_\mathcal{X})}}{2^{\alpha}}\left(\frac{\Gamma\left(d + \frac{\alpha}{2}\right)}{\Gamma(d)}\right)^2 
        \sigma^{2\alpha}.
    \end{align*}
\end{proposition}
\begin{proof}
    Define the $s$-fold application of the convolution operator $\widetilde{\mathcal{K}}^{\sigma}$ as 
    \[\widetilde{\mathcal{K}}^{\sigma}_s * \cdot
    := \sum_{j=1}^s (-1)^{1-j}\binom{s}{j}(j\sigma\sqrt{\pi})^{-d}\widetilde{\mathcal{K}}^{j\sigma} * \cdot\]
    and we define the approximator $f_0 \in \mathcal{H}_{K^\sigma}$ as 
    \begin{align*}
        &f_0(x,x') := \widetilde{\mathcal{K}}^{\sigma}_s * f^*_{\mathrm{square}}(x,x') \\
        ={}& \int_{\mathbb{R}^{2d}} \sum_{j=1}^{s}(-1)^{1-j}\binom{s}{j}
        \left(\frac{2}{\pi j^2\sigma^2}\right)^{d} \exp \left(- \frac{2\|(u,u')\|^{2}_2}{(j\sigma)^2}\right)
        f^*_{\mathrm{square}}((x,x')+(u,u')) d(u,u') \\
        ={}& \int_{\mathbb{R}^{2d}} \sum_{j=1}^{s}(-1)^{1-j}\binom{s}{j}
        \left(\frac{2}{\pi \sigma^2}\right)^{d} \exp \left(- \frac{2\|(u,u')\|^{2}_2}{\sigma^2}\right)
        f^*_{\mathrm{square}}((x,x') + j(u,u')) d(u,u').
    \end{align*}
    Then we have 
    \begin{align*}
        &f_0(x,x') - f^*_{\mathrm{square}}(x,x') \\
        ={}&  \widetilde{\mathcal{K}}^{\sigma}_s * f^*_{\mathrm{square}}(x,x') - 
        \int_{\mathbb{R}^{2d}}(\sigma\sqrt{\pi})^{-d}
        \mathcal{K}^{\sigma}(u,u') f^*_{\mathrm{square}}(x,x') d(u,u') \\
        ={}& \int_{\mathbb{R}^{2d}} \sum_{j=0}^{s}(-1)^{1-j}\binom{s}{j}
        \left(\frac{2}{\pi \sigma^2}\right)^{d} \exp \left(- \frac{2\|(u,u')\|^{2}_2}{\sigma^2}\right)
        f^*_{\mathrm{square}}((x,x') + j(u,u')) d(u,u') \\
        ={}& \int_{\mathbb{R}^{2d}} \sum_{j=0}^{s}(-1)^{1-j}\binom{s}{j}
        \left(\frac{2}{\pi \sigma^2}\right)^{d} \exp \left(- \frac{2\|(u,u')\|^{2}_2}{\sigma^2}\right)
        f^*_{\mathrm{square}}((x,x') + j(u,u')) d(u,u') \\
        ={}& \int_{\mathbb{R}^{2d}} (-1)^{1-s}
        \left(\frac{2}{\pi \sigma^2}\right)^{d} \exp \left(- \frac{2\|(u,u')\|^{2}_2}{\sigma^2}\right)
        \Delta^s_{(u,u')}f^*_{\mathrm{square}}(x,x') d(u,u').
    \end{align*}
    Notice that by definition of $f^*_{\mathrm{square}}$, 
    \begin{equation}\label{square loss excess risk}
        \begin{split}
            &\mathcal{R}^{\phi_{\mathrm{square}}}(f_0) - \mathcal{R}^{\phi_{\mathrm{square}}}(f^*_{\mathrm{square}}) \\
            ={}& \mathbb{E}[(1-\mathrm{sgn}(Y-Y')f_0(X,X'))^2 - (1-\mathrm{sgn}(Y-Y')f^*_{\mathrm{square}}(X,X'))^2] \\
            ={}& \mathbb{E}_{P_\mathcal{X}^2}
            \bigg[ \big(\eta_+ ( 2 - f_0 - f^*_{\mathrm{square}}) - \eta_- ( 2 + f_0 + f^*_{\mathrm{square}})\big) \cdot
            \big(f^*_{\mathrm{square}} - f_0 \big) \bigg] \\
            ={}& \mathbb{E}_{P_\mathcal{X}^2}
            \bigg[(\eta_+ + \eta_-) \cdot (f^*_{\mathrm{square}} - f_0)^2 \bigg] \\
            \leq{}& \|f^*_{\mathrm{square}} - f_0\|^2_{\mathcal{L}_2(P_\mathcal{X}^2)}.
        \end{split}
    \end{equation}
    By Minkowski's integral inequality, we have 
    \begin{align*}
        &\mathcal{R}^{\phi_{\mathrm{square}}}(f_0) - \mathcal{R}^{\phi_{\mathrm{square}}}(f^*_{\mathrm{square}}) \\
        \leq{}& \|f^*_{\mathrm{square}} - f_0\|^2_{\mathcal{L}_2(P_\mathcal{X}^2)} \\
        ={} & \int_{\mathcal{X}^2} \bigg\vert \int_{\mathbb{R}^{2d}} (-1)^{1-s}
        \left(\frac{2}{\pi \sigma^2}\right)^{d} \exp \left(- \frac{2\|(u,u')\|^{2}_2}{\sigma^2}\right)
        \Delta^s_{(u,u')}f^*_{\mathrm{square}}(x,x') d(u,u') \bigg\vert^2 dP_\mathcal{X}^2(x,x') \\ 
        \leq & \left(  \int_{\mathbb{R}^{2d}} \left(\int_{\mathcal{X}^2} 
        \left(\left(\frac{2}{\pi \sigma^2}\right)^{d} \exp \left(- \frac{2\|(u,u')\|^{2}_2}{\sigma^2}\right)
        \Delta^s_{(u,u')}f^*_{\mathrm{square}}(x,x')\right)^2 dP_\mathcal{X}^2(x,x') \right)^\frac{1}{2} d(u,u')\right)^2 \\
        ={} & \left(  \int_{\mathbb{R}^{2d}}\left(\frac{2}{\pi \sigma^2}\right)^{d} 
        \exp \left(- \frac{2\|(u,u')\|^{2}_2}{\sigma^2} \right)
        \|\Delta^s_{(u,u')}f^*_{\mathrm{square}}\|_{\mathcal{L}_2(P_\mathcal{X}^2)} d(u,u') \right)^2 \\
        \leq{} & \vert f^*_{\mathrm{square}}\vert^2_{\mathcal{B}_{2, \infty}^{\alpha}(P^2_\mathcal{X})} 
        \left(\frac{2}{\pi \sigma^2}\right)^{2d}
        \left(\int_{\mathbb{R}^{2d}}
        \exp \left(- \frac{2\|(u,u')\|^{2}_2}{\sigma^2} \right) \|(u,u')\|^\alpha_2 d(u,u') \right)^2 \\
        ={} & \vert f^*_{\mathrm{square}}\vert^2_{\mathcal{B}_{2, \infty}^{\alpha}(P^2_\mathcal{X})} 2^{-\alpha} 
        \left( \frac{\Gamma(d + \frac{\alpha}{2})}{\Gamma(d)}\right)^2 \sigma^{2\alpha}.
    \end{align*}

    By Proposition 4.46 of \cite{Steinwart2008Support}, for 
    $0 < \sigma_1 < \sigma_2 < \infty$, the space $\mathcal{H}_{\widetilde{K}^{\sigma_2}}$ is
    continuously embedded into $\mathcal{H}_{\widetilde{K}^{\sigma_1}}$ with the operator norm satisfies:
    \[\bigg\| \mathrm{id}:\mathcal{H}_{\widetilde{K}^{\sigma_2}} \to \mathcal{H}_{\widetilde{K}^{\sigma_1}} \bigg\|
    \leq \left(\frac{\sigma_2}{\sigma_1}\right)^d.\]
    Then we can bound $\|f_0\|_{K^\sigma}$ as 
    \begin{align*}
        \|f_0\|_{K^\sigma} &\leq 
        \sum_{j=1}^s \binom{s}{j}(j\sigma\sqrt{\pi})^{-d}
        \|\widetilde{\mathcal{K}}^{j\sigma} * f^*_{\mathrm{square}}\|_{K^\sigma} \\
        &= \sum_{j=1}^s \binom{s}{j}(j\sigma\sqrt{\pi})^{-d}
        \|\widetilde{\mathcal{K}}^{j\sigma} * f^*_{\mathrm{square}}\|_{\widetilde{K}^\sigma} \\
        &\leq \sum_{j=1}^s \binom{s}{j}(\sigma\sqrt{\pi})^{-d}
        \|\widetilde{\mathcal{K}}^{j\sigma} * f^*_{\mathrm{square}}\|_{\widetilde{K}^{j\sigma}} \\
        &\leq 2^s(\sigma\sqrt{\pi})^{-d}
        \|f^*_{\mathrm{square}}\|_{\mathcal{L}_2(\mathbb{R}^{2d})}. \\
    \end{align*}

    Finally, by Young's convolution inequality, we have 
    \begin{align*}
        \|f_0\|_\infty 
        &\leq \|f^*_{\mathrm{square}}\|_{\infty} 
        \left\| \sum_{j=1}^s (-1)^{1-j}\binom{s}{j}(j\sigma\sqrt{\pi})^{-d}\widetilde{\mathcal{K}}^{j\sigma} \right\|_{L_1(\mathbb{R}^{2d})} \\
        &\leq \sum_{j=1}^s \binom{s}{j}(j\sigma\sqrt{\pi})^{-d}\| \widetilde{\mathcal{K}}^{j\sigma} \|_{L_1(\mathbb{R}^{2d})} \\
        &=\sum_{j=1}^s \binom{s}{j} \\
        &\leq 2^{s}.
    \end{align*}
    Combining all these estimates, we complete the proof.
\end{proof}

\subsection{Calibration Inequality for Pairwise Ranking}\label{Subsection: Calibration Inequality for Pairwise Ranking}
In this subsection, we establish calibration inequality for pairwise ranking with the margin-based loss $\phi(y, y', t) = \psi(\mathrm{sgn}(y-y')t)$. 
Our discussion is not only limited to $\phi=\phi_{\mathrm{hinge}}$ and $\phi=\phi_{\mathrm{square}}$, but also can be generalized to study more general margin-based losses. 
So far, our estimates suffice to derive the learning rates for excess $\phi$-ranking risk. 
To bound the excess ranking risk $\mathcal{E}(f)$ by the excess $\phi$-ranking risk, it remains to establish the so-called calibration inequality or comparison theorem. 
The calibration inequality in binary classification has been extensively studied in a vast literature, see, e.g., 
\cite{Zhang2004Statistical, Bartlett2006Convexitya, Cucker2007Learning, Steinwart2008Support} and references therein. 
Motivated by these studies, we in this subsection aim to build similar calibration inequalities in the pairwise ranking setting. 
Concretely, for $\phi=\phi_{\mathrm{hinge}}$, an analogous version of Zhang's inequality for pairwise ranking is established in Proposition \ref{Zhang's inequality}. 
For those $\psi$ satisfying Assumption \ref{admissibleloss} and $\psi''(0) > 0$ which includes $\phi=\phi_{\mathrm{square}}$ as a special case, 
we prove that the excess ranking risk can be bounded by the square root of excess $\phi$- ranking risk (multiplying by some constant), 
and the noise condition in Assumption \ref{Tsybakov's noise condition} can further refine this upper bound.

\begin{proposition}\label{Zhang's inequality}
    For all measurable $f: \mathcal{X}^2 \to \mathbb{R}$ we have 
    \[\mathcal{R}(f) - \mathcal{R}(f^*_{\mathrm{rank}}) \leq 
    \mathcal{R}^{\phi_{\mathrm{hinge}}}(f) - \mathcal{R}^{\phi_{\mathrm{hinge}}}(f^*_{\mathrm{hinge}}).\]
\end{proposition}
\begin{proof}
   Consider the truncated function $\pi(f): \mathcal{X}^2 \to [-1, 1]$.
    Since hinge loss can be truncated at $t = 1$ and the truncated function does not change the 
    ranking risk, it is sufficient to prove the proposition for all measurable $f: \mathcal{X}^2 \to [-1, 1]$.
    Note that for all measurable $f: \mathcal{X}^2 \to [-1 , 1]$, there holds
    \[\mathcal{R}^{\phi_{\mathrm{hinge}}}(f) - \mathcal{R}^{\phi_{\mathrm{hinge}}}(f^*_{\mathrm{hinge}}) = 
    \mathbb{E}_{P_\mathcal{X}^2}[(\eta_+(X,X') - \eta_-(X,X')) \cdot (f^*_{\mathrm{hinge}}(X,X') - f(X,X'))].\]
    Then  
    \begin{align*}
        &\mathcal{R}(f) - \mathcal{R}(f^*_{\mathrm{rank}}) \\
        ={}& \mathbb{E}_{P_\mathcal{X}^2} \big[\eta_+\mathbb{I}_{(-\infty , 0)} (f) + \eta_-\mathbb{I}_{[0, \infty)} (f) 
        - \min\{\eta_+, \eta_-\}\big] \\
        \leq{}& \mathbb{E}_{P_\mathcal{X}^2} \big[(\eta_+ - \eta_-) \cdot (f^*_{\mathrm{hinge}} - f)\big] \\
        ={}& \mathcal{R}^{\phi_{\mathrm{hinge}}}(f) - \mathcal{R}^{\phi_{\mathrm{hinge}}}(f^*_{\mathrm{hinge}}), 
    \end{align*}
    where the inequality can be derived by consider the cases $\eta_+ > \eta_-$, $\eta_+ = \eta_-$ and $\eta_+ < \eta_-$ respectively. The proof is then finished.
\end{proof}

To continue our analysis of general margin-based losses, we further introduce some definitions and notations. Given a $\psi$ satisfying Assumption \ref{admissibleloss},
fix an $(x,x') \in \mathcal{X}^2$, we define function $\Psi = \Psi_{(x,x')} : \mathbb{R} \to [0, \infty)$ as 
\[ \Psi_{(x,x')}(t):= \eta_+(x,x')\psi(t) + \eta_-(x,x')\psi(-t) + \eta_=(x,x')\psi(0).\]
Since $\psi$ is convex, the left derivative $\psi'_-(t) \geq 0$ for $t \in (1, \infty)$ and 
the right derivative  $\psi'_+(t) < 0$ for $t \in (-\infty, 1)$.
Hence we can define
\begin{align*}
    f^*_-(x,x') &:= \sup \big\{t \in \mathbb{R} \mid \Psi'_-(t) < 0 \big\}, \\
    f^*_+(x,x') &:= \inf \big\{t \in \mathbb{R} \mid \Psi'_+(t) > 0 \big\}. 
\end{align*} 
Since $\Psi$ is also a convex function, we have $f^*_-(x,x') \leq f^*_+(x,x')$.
The following lemmas can be proved by the same arguments in Theorem 10.8 and Lemma 10.10 of \cite{Cucker2007Learning}. Recall that $\mathcal{X}^2_+, \mathcal{X}^2_-$ and $\mathcal{X}^2_=$ are defined in \eqref{seperation}.

\begin{lemma}\label{calibration 1}
    Given $\phi(y,y',t) = \psi(\mathrm{sgn}(y-y')t)$ and $(x,x') \in \mathcal{X}^2$. Let $f^*_\phi$ denote the Bayes $\phi$-ranking rule. If $\psi$ satisfies Assumption \ref{admissibleloss}, then the following statements hold.
    \begin{enumerate}
        \item $\Psi_{(x,x')}$ is strictly decreasing on $(-\infty , f^*_-(x,x')]$, strictly increasing on $[f^*_+(x,x') , \infty)$
        and keep constant on $[f^*_-(x,x') , f^*_+(x,x')]$.
        \item $f^*_\phi(x,x')$ is a minimizer of $\Psi_{(x,x')}$, i.e., $f^*_\phi(x,x')=\mathop{\arg\min}_{t\in \mathbb{R}}\Psi_{(x,x')}(t)$, 
        and can take any value in $[f^*_-(x,x') , f^*_+(x,x')]$.
        \item There holds
        $\begin{cases}
            0 \leq f^*_-(x,x') \leq f^*_\phi(x,x'), & (x,x') \in \mathcal{X}^2_+, \\
            f^*_\phi(x,x') \leq f^*_+(x,x') \leq 0, & (x,x') \in \mathcal{X}^2_-, \\
            f^*_-(x,x') \leq 0 \leq f^*_+(x,x'), & (x,x') \in \mathcal{X}^2_=. \\
        \end{cases}$
        \item $f^*_-(x,x') \leq 1$ and $f^*_+(x,x') \geq -1$.
    \end{enumerate}
\end{lemma}
\begin{lemma}\label{calibration 2}
    If $\psi$ satisfies Assumption \ref{admissibleloss} and $\psi''(0) > 0$. Then there exists a constant $C_\psi > 0$
    only depending on $\psi$ such that for all $(x,x') \in \mathcal{X}^2$
    \[(\eta_+(x,x') - \eta_-(x,x'))^2 \leq  C_\psi\big(\Psi_{(x,x')}(0) - \Psi_{(x,x')}(f^*_\phi(x,x'))\big).\]
\end{lemma}
By Lemma \ref{calibration 1} we also see that $\phi$ is truncated at $t=1$ and the Bayes $\phi$-ranking rule can be defined pointwise which can be taken as
skew-symmetric. Then we can establish a general calibration inequality for pairwise ranking.
\begin{proposition}\label{general calibration inequality}
    If $\psi$ satisfies Assumption \ref{admissibleloss} and $\psi''(0) > 0$. 
    Let $f^*_\phi$ be the Bayes $\phi$- ranking rule.
    Then there exists a constant $c_\psi > 0$
    only depending on $\psi$
    such that for all measurable $f : \mathcal{X}^2 \to \mathbb{R}$ we have 
    \[\mathcal{R}(f) - \mathcal{R}(f^*_{\mathrm{rank}}) \leq c_\psi \sqrt{\mathcal{R}^{\phi}(f) - \mathcal{R}^{\phi}(f^*_\phi)}.\]
\end{proposition}
\begin{proof} Given $f : \mathcal{X}^2 \to \mathbb{R}$, define 
\begin{align*}
    \mathcal{X}^2_{f}:=& \big\{(x,x') \in \mathcal{X}^2 \mid (x,x') \in \mathcal{X}^2_+ \mbox{ and } f(x,x') < 0\big\}\\
    &\quad \bigcup \big\{(x,x') \in \mathcal{X}^2 \mid (x,x') \in \mathcal{X}^2_- \mbox{ and }f(x,x') \geq 0 \big\}.
\end{align*} Then by Cauchy-Schwarz inequality and Lemma \ref{calibration 2} we have 
    \begin{align*}
        \mathcal{R}(f) - \mathcal{R}(f^*_{\mathrm{rank}}) &= \int_{\mathcal{X}^2_{f}}
        \vert \eta_+(x,x') - \eta_-(x,x') \vert dP^2_\mathcal{X}(x,x') \\
        &\leq \left(\int_{\mathcal{X}^2_{f}}
        \vert \eta_+(x,x') - \eta_-(x,x') \vert^2 dP^2_\mathcal{X}(x,x')\right)^{\frac{1}{2}} \\
        &\leq \left(C_\psi \int_{\mathcal{X}^2_{f}}
        \Psi_{(x,x')}(0) - \Psi_{(x,x')}(f^*_\phi(x,x'))  dP^2_\mathcal{X}(x,x')\right)^{\frac{1}{2}}.
    \end{align*}
    For $(x,x') \in \mathcal{X}_{f}$, if $(x,x') \in \mathcal{X}^2_+$, then 
    $f(x,x') < 0$ and $f^*_{\mathrm{rank}}(x,x') = 1$. By Lemma \ref{calibration 1}, 
    $\Psi_{(x,x')}$ is strictly decreasing on $(-\infty , 0]$ and hence 
    $\Psi_{(x,x')}(0) \leq \Psi_{(x,x')}(f(x,x'))$. Following the same argument, the inequality also holds for $(x,x') \in \mathcal{X}^2_-$. Thus, let $c_\psi := C_\psi^{\frac{1}{2}}$ we have 
    \begin{align*}
        \mathcal{R}(f) - \mathcal{R}(f^*_{\mathrm{rank}})
        &\leq \left(C_\psi \int_{\mathcal{X}^2_{f}}
        \Psi_{(x,x')}(0) - \Psi_{(x,x')}(f^*_\phi(x,x'))  dP^2_\mathcal{X}(x,x')\right)^{\frac{1}{2}} \\
        &\leq \left(C_\psi \int_{\mathcal{X}^2_{f}}
        \Psi_{(x,x')}(f(x,x')) - \Psi_{(x,x')}(f^*_\phi(x,x'))  dP^2_\mathcal{X}(x,x')\right)^{\frac{1}{2}} \\
        &\leq c_\psi \sqrt{\mathcal{R}^{\phi}(f) - \mathcal{R}^{\phi}(f^*_\phi)}.
    \end{align*} The proof is then finished.
\end{proof}
Furthermore, under the noise condition in Assumption \ref{Tsybakov's noise condition}, a refined oracle inequality can be established.

\begin{proposition}\label{improved calibration inequality}
    Assume that $\psi$ satisfies Assumption \ref{admissibleloss}, $\psi''(0) > 0$ and Assumption \ref{Tsybakov's noise condition} holds with $C_*>0$ and $q\in [0,\infty]$. 
    Let $f^*_\phi$ be the Bayes $\phi$-ranking rule. Then there exists a constant $c_{\psi,q} > 0$ only depending on $\psi, q$ and $C_*$
    such that for all measurable $f : \mathcal{X}^2 \to \mathbb{R}$, 
    \[\mathcal{R}(f) - \mathcal{R}(f^*_{\mathrm{rank}}) \leq 
    c_{\psi,q} (\mathcal{R}^{\phi}(f) - \mathcal{R}^{\phi}(f^*_\phi))^{\frac{q+1}{q+2}}.\]
\end{proposition}
\begin{proof} For $t > 0$ we denote 
    \[U_t := \big\{ (x,x') \in \mathcal{X}^2 : 
    \big\vert \eta_+(x,x') - \eta_-(x,x')\big\vert  \leq t\big\},\]
    \[V_t := \big\{ (x,x') \in \mathcal{X}^2 : 
    \big\vert \eta_+(x,x') - \eta_-(x,x')\big\vert  > t\big\}.\] Recall $\mathcal{X}^2_{f}$ defined in the proof of Proposition \ref{calibration 1}. 
    By Assumption \ref{Tsybakov's noise condition} and analysis from the last proof, we have 
    \begin{align*}
        &\mathcal{R}(f) - \mathcal{R}(f^*_{\mathrm{rank}}) \\
        ={}& \int_{\mathcal{X}^2_{f}}
        \vert \eta_+(x,x') - \eta_-(x,x') \vert dP^2_\mathcal{X}(x,x') \\
        ={}& \int_{\mathcal{X}^2_{f} \cap  U_t} 
        \vert \eta_+(x,x') - \eta_-(x,x') \vert dP^2_\mathcal{X}(x,x') 
        + \int_{\mathcal{X}^2_{f} \cap  V_t} 
        \vert \eta_+(x,x') - \eta_-(x,x') \vert dP^2_\mathcal{X}(x,x') \\
        \leq{}& t P^2_\mathcal{X}(U_t)  
        + \frac{1}{t}\int_{\mathcal{X}^2_{f} \cap  V_t} 
        \vert \eta_+(x,x') - \eta_-(x,x') \vert^2 dP^2_\mathcal{X}(x,x') \\
        \leq{}& C_*t^{q+1}  
        + \frac{C_\psi(\mathcal{R}^{\phi}(f) - \mathcal{R}^{\phi}(f^*_\phi))}{t} .\\
    \end{align*}
    Then the optimal choice
    \[t := \left(\frac{C_\psi(\mathcal{R}^{\phi}(f) - \mathcal{R}^{\phi}(f^*_\phi))}{(q+1)C_*}\right)^{\frac{1}{q+2}}\]
    yields that 
    \begin{align*}
    \mathcal{R}(f) - \mathcal{R}(f^*_{\mathrm{rank}}) 
    &\leq C_*\left(\frac{C_\psi(\mathcal{R}^{\phi}(f) - \mathcal{R}^{\phi}(f^*_\phi))}{(q+1)C_* }\right)^{\frac{q+1}{q+2}}\\
    &\quad + (q+1)^{\frac{1}{q+2}}C_*^{\frac{1}{q+2}}C_\psi^{\frac{q+1}{q+2}}(\mathcal{R}^{\phi}(f) - \mathcal{R}^{\phi}(f^*_\phi))^{\frac{q+1}{q+2}}.
    \end{align*} By taking $c_{\psi,q} := (q+1)^{-\frac{q+1}{q+2}}C_*^{\frac{1}{q+2}}C_{\psi}^{\frac{q+1}{q+2}} 
    + (q+1)^{\frac{1}{q+2}}C_*^{\frac{1}{q+2}}C_{\psi}^{\frac{q+1}{q+2}}$, we complete the proof.
\end{proof}
We see that the case of $q = 0$ in which Assumption \ref{Tsybakov's noise condition} holds trivially, Proposition 
\ref{improved calibration inequality} reduces to Proposition \ref{general calibration inequality}. 
When $q>0$, the calibration inequality is refined compared to the previous one given in Proposition \ref{general calibration inequality}. 
In particular, when $q=\infty$, we obtain a similar inequality as that established for $\phi_{\mathrm{hinge}}$ in Proposition \ref{Zhang's inequality}.

\subsection{Learning Rates of Gaussian Ranking Estimators with Hinge Loss and Square Loss}\label{Subsection: Learning Rates of Gaussian Ranking Estimators with Hinge Loss and Square Loss}
We are now in the position to prove the oracle inequalities of Gaussian ranking estimators with $\phi=\phi_{\mathrm{hinge}}$ and $\phi=\phi_{\mathrm{square}}$, and then derive the learning rates. 
To this end, it only remains to verify Assumption \ref{assumption 1}, particularly the variance bound \eqref{variance bound}. 
For hinge loss, the noise condition in Assumption \ref{Tsybakov's noise condition} can be used to establish the variance bound.
\begin{lemma}\label{hinge loss variance bound}
    If Assumption \ref{Tsybakov's noise condition} holds with $C_*>0$ and $q\in [0,\infty]$, then for all measurable $f:\mathcal{X}^2 \to [-1,1]$ we have 
    \[\mathbb{E}(Q\phi_{\mathrm{hinge}f} - Q\phi_{\mathrm{hinge}f^*_{\mathrm{hinge}}})^2 \leq V
    \big(\mathbb{E}(Q\phi_{\mathrm{hinge}f}  - Q\phi_{\mathrm{hinge}f^*_{\mathrm{hinge}}})\big)^{\tau},\]
    where $V = 2^{\frac{q+2}{q+1}}C_*^{\frac{1}{q+1}}q^{\frac{1}{q+1}}(1+q^{-1})$ and $\tau = \frac{q}{q+1}.$
\end{lemma}
\begin{proof} Recall $U_t$ and $V_t$ defined in the proof of Proposition \ref{improved calibration inequality}, which is given by
\[U_t = \big\{ (x,x') \in \mathcal{X}^2 :  
    \big\vert \eta_+(x,x') - \eta_-(x,x')\big\vert  \leq t\big\},\]
    \[V_t = \big\{ (x,x') \in \mathcal{X}^2 : 
    \big\vert \eta_+(x,x') - \eta_-(x,x')\big\vert  > t\big\}.\]
Then we have 
    \begin{align*}
        &\mathbb{E}(Q\phi_{\mathrm{hinge}f} - Q\phi_{\mathrm{hinge}f^*_{\mathrm{hinge}}})^2 \\
        ={}& \int_{\mathcal{X} \times \mathcal{Y}} \biggl(\int_{\mathcal{X} \times \mathcal{Y}}
        \phi_{\mathrm{hinge}}(y, y', f(x,x')) - \phi_{\mathrm{hinge}}(y, y', f^*_{\mathrm{hinge}}(x,x'))
        dP(x',y')\biggr)^2dP(x,y) \\
        \leq{}& \int_{\mathcal{X} \times \mathcal{Y}} \int_{\mathcal{X} \times \mathcal{Y}}
        \biggl(\phi_{\mathrm{hinge}}(y, y', f(x,x')) - \phi_{\mathrm{hinge}}(y, y', f^*_{\mathrm{hinge}}(x,x'))\biggr)^2
        dP(x',y')dP(x,y) \\
        \leq{}& \int_{\mathcal{X}^2}
        \vert f(x,x') - f^*_{\mathrm{hinge}}(x,x')\vert ^2
        dP^2_\mathcal{X}(x, x') \\
        \leq{}& 2\int_{\mathcal{X}^2}
        \vert f(x,x') - f^*_{\mathrm{hinge}}(x,x')\vert 
        dP^2_\mathcal{X}(x, x')  \\
        ={}& 2\int_{U_t \cup V_t} \vert f(x,x') - f^*_{\mathrm{hinge}}(x,x')\vert 
        dP_\mathcal{X}^2(x,x') \\
        \leq{}& 4P^2_\mathcal{X}(U_t) + \frac{2}{t}\int_{V_t} 
        \vert \eta_+(x,x') - \eta_-(x,x')\vert  \cdot
        \vert f(x,x') - f^*_{\mathrm{hinge}}(x,x')\vert 
        dP_\mathcal{X}^2(x,x') \\
        \leq{} & 4C_*t^q + 2t^{-1}\mathbb{E}(Q\phi_{\mathrm{hinge}f}  - Q\phi_{\mathrm{hinge}f^*_{\mathrm{hinge}}}) \\
        \leq{} & 2^{\frac{q+2}{q+1}}C_*^{\frac{1}{q+1}}q^{\frac{1}{q+1}}(1+q^{-1})
        \big(\mathbb{E}(Q\phi_{\mathrm{hinge}f}  - Q\phi_{\mathrm{hinge}f^*_{\mathrm{hinge}}})\big)^{\frac{q}{q+1}}.
    \end{align*}
    In the last inequality, we choose the optimal 
    \[t:= \left(\frac{\mathbb{E}(Q\phi_{\mathrm{hinge}f}  - Q\phi_{\mathrm{hinge}f^*_{\mathrm{hinge}}})}
    {2qC_*}\right)^{\frac{1}{q+1}}\] to minimize the sum. The proof is then finished.
\end{proof}

Now we can prove Theorem \ref{Oracle inequality of hinge loss} which presents the oracle inequality and learning rates for $\phi=\phi_{\mathrm{hinge}}$.

\begin{proof}[Proof of Theorem \ref{Oracle inequality of hinge loss}]
    Combining Proposition \ref{entropy number estimate of pairwise Gaussian kernel}, 
    Proposition \ref{hinge loss approximation} and
    Proposition \ref{hinge loss variance bound}, we directly apply Theorem
    \ref{Oracle inequality of pairwise ranking with a margin-base loss}
    with $L = 1, B = 2, B_0 = 2, \tau = \frac{q}{q+1},
    V = 2^{\frac{q+2}{q+1}}C_*^{\frac{1}{q+1}}q^{\frac{1}{q+1}}(1+q^{-1}) \leq 6C_*^{\frac{1}{q+1}}, p_1 = p_2 = p,
    a_1 = a_2 = (C^*_\mathcal{X} p^{-2d-1}\sigma^{-2\varrho})^{\frac{1}{2p}}$, which yields
    \begin{align*}
        &\lambda\|f_\mathbf{z}\|^2_{K^\sigma} + \mathcal{R}^{\phi_{\mathrm{hinge}}}(\pi(f_\mathbf{z})) - \mathcal{R}^{\phi_{\mathrm{hinge}}}(f^*_{\mathrm{hinge}}) \\
        \leq{}& \frac{2^{3d+4}r^{2d}}{\Gamma(d)}\lambda \sigma^{-2d} 
        + \frac{2^{\beta/2+4}C_{**}\Gamma(d + \frac{\beta}{2})}{\Gamma(d)}\sigma^\beta 
        + 36c_1 \left(\frac{C^*_\mathcal{X}  C_*^{\frac{1-p}{q+1}}}
        {\lambda^{p} p^{2d+1} \sigma^{2\varrho} n}\right)^{\frac{q+1}{q-p+2}} \\
        & + \frac{6C^*_\mathcal{X}(2c_2 + c_3t + c_4t)}{\lambda^{p} p^{2d+1} \sigma^{2\varrho} n}
        + \left(\frac{11232C_*^{\frac{1}{q+1}} t}{n}\right)^{\frac{q+1}{q+2}}
        + \frac{(2712+ 3c_5)t}{n}.
    \end{align*}
    Recall the proof of Proposition \ref{proposition 3} and 
    notice that constants $c_3$ and $c_4$ depend on $p$. 
    When $p \in (0,  1/4]$, one can verify that $c_3$ and $c_4$ are uniformly upper bounded. Taking $c_6=\sup\{ c_i : 1 \leq i \leq 5, 0 < p \leq 1/4 \}$, we complete the proof of the oracle inequality.

    To derive the learning rate, for all $t \geq 1$, write
    \[\mathcal{R}^{\phi_{\mathrm{hinge}}}(\pi(f_\mathbf{z})) - \mathcal{R}^{\phi_{\mathrm{hinge}}}(f^*_{\mathrm{hinge}})
    \lesssim \lambda\sigma^{-2d} + \sigma^\beta  
    + \left(\frac{t}{\lambda^{p} p^{2d+1} \sigma^{2\varrho} n}\right)^{\frac{q+1}{q-p+2}} 
    + \left(\frac{t}{n}\right)^{\frac{q+1}{q+2}}.\]
    With the choices we specified in the statement of the theorem, we have 
    \[\lambda < 1, \quad \sigma^{2\varrho}n \geq 1, \quad \frac{q+1}{q+2} \leq \frac{q+1}{q-p+2} \leq 1,\]
    which leads to 
    \begin{align*}
        \left(\frac{t}{p^{2d+1}\sigma^{2d}\lambda^p n}\right)^{\frac{q+1}{q-p+2}} 
        &\leq tp^{-2d-1} \lambda^{-p} (\sigma^{2\varrho}n)^{-\frac{q+1}{q+2}} \\
        &\lesssim t\lambda^{-1/\log n} (\sigma^{2\varrho}n)^{-\frac{q+1}{q+2}}\log^{2d+1}n \\
        &= te^b n^{-\frac{\beta (q+1)}{\beta(q+2) + 2\varrho(q+1)}}\log^{2d+1}n,
    \end{align*}
    while other terms have convergence rates faster than this. Then for all $n \geq 2, t \geq 1$, by Proposition \ref{Zhang's inequality}, 
    with probability at least $1-(c_0 + 5)\exp(-t)$, 
    we have 
    \[\mathcal{E}(\pi(f_\mathbf{z})) \leq 
    \mathcal{E}^{\phi_{\mathrm{hinge}}}(\pi(f_\mathbf{z}))
    \lesssim t n^{-\frac{\beta (q+1)}{\beta(q+2) + 2\varrho(q+1)}} \log^{2d+1}n.\] Thus we complete the proof.
\end{proof}
For $\phi=\phi_{\mathrm{square}}$, the variance bound can be proved directly.
\begin{lemma}\label{square loss variance bound}
    For all measurable $f: \mathcal{X}^2 \to [-1,1]$ we have 
    \[\mathbb{E}(Q\phi_{\mathrm{square}f} - Q\phi_{\mathrm{square}f^*_{\mathrm{square}}})^2 \leq 
    16 \mathbb{E}(Q\phi_{\mathrm{square}f}  - Q\phi_{\mathrm{square}f^*_{\mathrm{square}}}).\] 
\end{lemma}
\begin{proof}
    Recall the calculation we did in \eqref{square loss excess risk}, we have 
    \begin{align*}
        &\mathbb{E}(Q\phi_{\mathrm{square}f} -  Q\phi_{\mathrm{square}f^*_{\mathrm{square}}})^2 \\
        ={}& \int_{\mathcal{X} \times \mathcal{Y}} \bigg(\int_{\mathcal{X} \times \mathcal{Y}}
        \phi_{\mathrm{square}}(y, y', f(x,x')) - \phi_{\mathrm{square}}(y, y', f^*_{\mathrm{square}}(x,x'))
        dP(x',y')\bigg)^2dP(x,y) \\
        \leq{}& \int_{\mathcal{X} \times \mathcal{Y}} \int_{\mathcal{X} \times \mathcal{Y}}
        \bigg(\phi_{\mathrm{square}}(y, y', f(x,x')) - \phi_{\mathrm{square}}(y, y', f^*_{\mathrm{square}}(x,x'))\bigg)^2
        dP(x',y')dP(x,y) \\
        ={}& \mathbb{E}_{P^2_\mathcal{X}}\big[\eta_+(2 - f - f^*_{\mathrm{square}})^2(f^*_{\mathrm{square}} - f)^2
        + \eta_-(2 + f + f^*_{\mathrm{square}})^2(f^*_{\mathrm{square}} - f)^2\big] \\
        \leq{}& 16\mathbb{E}_{P^2_\mathcal{X}}\big[(\eta_+ + \eta_-)(f^*_{\mathrm{square}} - f)^2 \big] \\
        ={}& 16(\mathcal{R}^{\phi_{\mathrm{square}}}(f) - \mathcal{R}^{\phi_{\mathrm{square}}}(f^*_{\mathrm{square}})) \\
        ={}& 16\mathbb{E}(Q\phi_{\mathrm{square}f}  - Q\phi_{\mathrm{square}f^*_{\mathrm{square}}}). 
    \end{align*} 
    The proof is then finished.
\end{proof} 

At the end of this subsection, we prove Theorem \ref{Oracle inequality of square loss} which presents the oracle inequality and learning rates for $\phi=\phi_{\mathrm{square}}$.

\begin{proof}[Proof of Theorem \ref{Oracle inequality of square loss}]
    Combining Proposition \ref{entropy number estimate of pairwise Gaussian kernel}, 
    Proposition \ref{square loss approximation} and
    Proposition \ref{square loss variance bound}, we can apply Theorem
    \ref{Oracle inequality of pairwise ranking with a margin-base loss}
    with $L = 4, B = 4, B_0 = 2^{2s+2}, \tau = 1,
    V = 16, p_1 = p_2 = p,
    a_1 = a_2 = (C^*_\mathcal{X} p^{-2d-1}\sigma^{-2\varrho})^{\frac{1}{2p}}$, 
    and the $c_6$ we introduce in proof of Theorem \ref{Oracle inequality of hinge loss}, which yields
    \begin{align*}
        &\lambda\|f_\mathbf{z}\|^2_{K^\sigma} + \mathcal{R}^{\phi_{\mathrm{square}}}(\pi(f_\mathbf{z})) - \mathcal{R}^{\phi_{\mathrm{square}}}(f^*_{\mathrm{square}}) \\
        \leq{}& \frac{2^{2s+3}\|f^*_{\mathrm{square}}\|^2_{\mathcal{L}_2(\mathbb{R}^{2d})}}{\pi^d}\lambda\sigma^{-2d} 
        + 2^{3-\alpha}\left(\frac{\Gamma\left(d + \frac{\alpha}{2}\right)}{\Gamma(d)}\right)^2
        \vert f^*_{\mathrm{square}} \vert^2_{\mathcal{B}_{2, \infty}^{\alpha}(P^2_\mathcal{X})} \sigma^{2\alpha} \\
        &+ \frac{C^*_\mathcal{X}(96c_6 + 72c_6t)}{\lambda^{p}p^{2d+1}\sigma^{2\varrho} n}
        + \frac{(33552 + 1824 \cdot 2^{2s} + 3c_6) t}{n}.
    \end{align*} 
    To derive the learning rate, for all $t \geq 1$, write
    \[\mathcal{R}^{\phi_{\mathrm{square}}}(\pi(f_\mathbf{z})) - \mathcal{R}^{\phi_{\mathrm{square}}}(f^*_{\mathrm{square}})
    \lesssim \lambda\sigma^{-2d} + \sigma^{2\alpha}  
    + \frac{t}{\lambda^{p} p^{2d+1} \sigma^{2\varrho} n} 
    + \frac{t}{n}.\]
    With the choices we specified in the statement of the theorem, we have 
    \[\lambda < 1, \quad \sigma^{2\varrho}n \geq 1, \]
    which leads to 
    \[\frac{t}{p^{2d+1}\sigma^{2\varrho}\lambda^p n}
    \lesssim t\lambda^{-1/\log n} (\sigma^{2\varrho}n)^{-1}\log^{2d+1}n
    = e^bn^{\frac{\alpha}{\alpha + \varrho}}\log^{2d+1}n,\]
    while other terms have convergence rates faster than this.
    Then for all $n \geq 2, t \geq 1$, by Proposition \ref{general calibration inequality}, 
    with probability at least $1-(c_0 + 5)\exp(-t)$, 
    we have
    \[\mathcal{E}(\pi(f_\mathbf{z})) \lesssim 
    \sqrt{\mathcal{E}^{\phi_{\mathrm{square}}}(\pi(f_\mathbf{z}))}
    \lesssim \sqrt{t} n^{-\frac{\alpha}{2(\alpha + \varrho)}} \log^{d+\frac{1}{2}}n.\]
    Furthermore, if $P$ additionally satisfies Assumption \ref{Tsybakov's noise condition}, 
    by Proposition \ref{improved calibration inequality}, 
    with probability at least $1-(c_0 + 5)\exp(-t)$ we have 
    \[\mathcal{E}(\pi(f_\mathbf{z})) \lesssim  
    (\mathcal{E}^{\phi_{\mathrm{square}}}(\pi(f_\mathbf{z})))^{\frac{q+1}{q+2}} \lesssim 
    t^{\frac{q+1}{q+2}} n^{-\frac{(q+1)\alpha}{(q+2)(\alpha + \varrho)}}
    \log^{\frac{(q+1)(2d+1)}{q+2}}n.\] 
    The proof is then finished.
\end{proof}

\subsection{Comparisons of Noise Conditions}\label{Subsection: Comparison of Noise Conditions}
In this subsection, we make some comparisons of different noise conditions. 
Let \textbf{TN($q$)} denote the noise condition in Assumption \ref{Tsybakov's noise condition}.
\cite{Clemencon2008Ranking} and \cite{Clemencon2011Minimax} propose another two noise conditions 
in the context of bipartite ranking, namely the global low noise conditions \textbf{LN($q$)} and \textbf{NA($q$)}, 
which can be reformulated as the following conditions in the pairwise ranking setting.
\begin{assumption}[\textbf{Global low noise condition LN($q$)}]
    There exist constants $C_{\mathrm{LN}} > 0$ and $q \in [0, \infty]$ such that for all $x \in \mathcal{X}$,
    \[\mathbb{E}_{X'}[\vert \eta_+(x,X') - \eta_-(x,X') \vert^{-q}] \leq C_{\mathrm{LN}}.\]
\end{assumption}
\begin{assumption}[\textbf{Global low noise condition NA($q$)}]
    There exists constants $C_{\mathrm{NA}} > 0$ and $q \in [0, \infty]$ such that for all $x \in \mathcal{X}$ and $t > 0$, 
    \[P_{\mathcal{X}}(\{x' \in \mathcal{X} : \vert \eta_+(x,x') - \eta_-(x,x') \vert \leq t \}) \leq C_{\mathrm{NA}} t^{q}.\]
\end{assumption}
Similar to the proof of Proposition 3 of \cite{Clemencon2011Minimax} which describes the connection between \textbf{LN($q$)} and \textbf{NA($q$)}, 
we can show that \textbf{LN($q$)} implies \textbf{NA($q$)} and \textbf{NA($q$)} implies \textbf{LN($q'$)} for all $q' < q$ so \textbf{NA($q$)} can be considered as a slightly weaker condition.

Obviously, \textbf{NA($q$)} implies \textbf{TN($q$)}, and here is an example showing that this implication is sharp. Consider $P_\mathcal{X}$ be a uniform distribution on $\mathcal{X} = [0, 1]$ and 
the label $Y = X + \epsilon$ where $\epsilon$ is a random noise uniformly distributed on $[0, 1]$. The conditional distribution $P_{Y \vert X}$ together with $P_\mathcal{X}$ generate the distribution $P$ on 
$\mathcal{X} \times \mathcal{Y} = [0, 1] \times [0, 2]$. A simple calculation shows that $\vert \eta_+(x,x') - \eta_-(x,x') \vert =  2 \vert x - x' \vert - \vert x - x' \vert^2$ and hence 
the set $\{(x,x') \in \mathcal{X}^2: \vert \eta_+(x,x') - \eta_-(x,x') \vert \leq t\}$ is a belt domain along the diagonal with width and probability measure of $O(t)$ asymptotically when $t \to 0$. 
Thus, $P$ satisfies \textbf{NA($1$)} and \textbf{TN($1$)} but does not satisfy \textbf{TN($q$)} for all $q > 1$. This is also an example that $P$ does not satisfy \textbf{LN($1$)} while 
it satisfies $\textbf{LN($q$)}$ for all $q < 1$.

On the other hand, generally, \textbf{TN($q$)} can not imply \textbf{NA($q'$)} for any $q' \in (0,q]$ because the global low noise condition is stated at every $x \in \mathcal{X}$ while 
\textbf{TN($q$)} only requires a fast decay speed of the product probability measure. One can easily construct a decreasing sequence of sets on $\mathcal{X}^2$ with fast decay measure 
but the measure of their cross-section at some point $x \in \mathcal{X}$ keeps a positive constant.

Having noticed that \textbf{LN($q$)} and \textbf{NA($q$)} are strictly stronger than \textbf{TN($q$)}, one may expect that the global low noise condition can derive a bigger variance bound exponent $\tau$
compared with $\tau = \frac{q}{q+1}$ stated in Lemma \ref{hinge loss variance bound} under \textbf{TN($q$)}. In fact, we have the following result.
\begin{proposition}\label{hinge loss variance bound with global low noise condition}
    The following assertions hold.
    \begin{enumerate}
        \item If $P$ satisfies \textbf{LN($q$)} with $q \in [0, 1]$, then for all measurable $f:\mathcal{X}^2 \to [-1,1]$ we have
        \[\mathbb{E}(Q\phi_{\mathrm{hinge}f} - Q\phi_{\mathrm{hinge}f^*_{\mathrm{hinge}}})^2 \leq 
        V\big(\mathbb{E}(Q\phi_{\mathrm{hinge}f}  - Q\phi_{\mathrm{hinge}f^*_{\mathrm{hinge}}})\big)^{\tau},\]
        where $V = 4C_{\mathrm{LN}}$ and $\tau = q$.
        \item If $P$ satisfies \textbf{NA($q$)} with $q \in [0, \infty]$, then for all measurable $f:\mathcal{X}^2 \to [-1,1]$ we have
        \[\mathbb{E}(Q\phi_{\mathrm{hinge}f} - Q\phi_{\mathrm{hinge}f^*_{\mathrm{hinge}}})^2 \leq 
        V\big(\mathbb{E}(Q\phi_{\mathrm{hinge}f}  - Q\phi_{\mathrm{hinge}f^*_{\mathrm{hinge}}})\big)^{\tau},\]
        where $V = 2^{\frac{2q+3}{2q+1}}C_{\mathrm{NA}}^{\frac{2}{2q+1}}q^{\frac{1}{2q+1}}(2+q^{-1})$ and $\tau = \frac{2q}{2q+1}$.
        \item The best variance exponent $\tau$ we can derive from the global low noise conditions can be listed below:
        \begin{table}[!htb]
            \centering
            \begin{tabular}{|c|c|c|c|} \hline
            $\tau$ & $q \in [0, 1/2]$ & $q \in (1/2, 1]$ & $q \in (1, \infty]$\\ \hline
            \textbf{LN($q$)} & $2q/(2q+1)$ & $q$ & $1$ \\ \hline
            \textbf{NA($q$)} & $2q/(2q+1)$ & $(1/2, q)$ & $1$ \\ \hline
            \end{tabular}
        \end{table}
    \end{enumerate} 
\end{proposition}
\begin{proof}
    If \textbf{LN($q$)} holds with $q \in [0, 1]$, by Cauchy-Schwarz inequality and Jensen's inequality we have 
    \begin{align*}
        & \mathbb{E}(Q\phi_{\mathrm{hinge}f} - Q\phi_{\mathrm{hinge}f^*_{\mathrm{hinge}}})^2 \\
        \leq{}& \mathbb{E}_{X}( \mathbb{E}_{X'} \vert f(X,X') - f^*_{\mathrm{hinge}}(X,X') \vert)^2 \\
        \leq{}& C_{\mathrm{LN}}\mathbb{E}_{X}( \mathbb{E}_{X'} \vert f(X,X') - f^*_{\mathrm{hinge}}(X,X') \vert^2 \cdot \vert \eta_+(X,X') - \eta_-(X,X')\vert^{q}) \\ 
        \leq{}& 4C_{\mathrm{LN}}\mathbb{E}_{X}( \mathbb{E}_{X'} \vert f(X,X') - f^*_{\mathrm{hinge}}(X,X') \vert^q \cdot \vert \eta_+(X,X') - \eta_-(X,X')\vert^{q}) \\ 
        \leq{}& 4C_{\mathrm{LN}}(\mathbb{E}_{X} \mathbb{E}_{X'} \vert f(X,X') - f^*_{\mathrm{hinge}}(X,X') \vert \cdot \vert \eta_+(X,X') - \eta_-(X,X')\vert)^{q} \\
        ={}& 4C_{\mathrm{LN}} \big(\mathbb{E}(Q\phi_{\mathrm{hinge}f}  - Q\phi_{\mathrm{hinge}f^*_{\mathrm{hinge}}})\big)^{q}.
    \end{align*}
    If \textbf{NA($q$)} holds with $q \in [0, \infty]$,  we have 
    \begin{align*} 
        & \mathbb{E}(Q\phi_{\mathrm{hinge}f} - Q\phi_{\mathrm{hinge}f^*_{\mathrm{hinge}}})^2 \\
        \leq{}& \mathbb{E}_{X}( \mathbb{E}_{X'} \vert f(X,X') - f^*_{\mathrm{hinge}}(X,X')\vert \cdot \\
        &(\mathbb{I}_{\{\vert \eta_+(X,X') - \eta_-(X,X')\vert \leq t\}} + \mathbb{I}_{\{\vert \eta_+(X,X') - \eta_-(X,X')\vert > t\}}))^2 \\
        \leq{}& 2\mathbb{E}_{X}( \mathbb{E}_{X'} \vert f(X,X') - f^*_{\mathrm{hinge}}(X,X') \vert \cdot \mathbb{I}_{\{\vert \eta_+(X,X') - \eta_-(X,X')\vert \leq t\}})^2 \\
        &+ 2\mathbb{E}_{X}( \mathbb{E}_{X'} \vert f(X,X') - f^*_{\mathrm{hinge}}(X,X') \vert \cdot \mathbb{I}_{\{\vert \eta_+(X,X') - \eta_-(X,X')\vert > t\}})^2 \\
        \leq{}& 8C^2_{\mathrm{NA}}t^{2q} + 2\mathbb{E}_{X}\mathbb{E}_{X'} \vert f(X,X') - f^*_{\mathrm{hinge}}(X,X') \vert^2 \cdot \mathbb{I}_{\{\vert \eta_+(X,X') - \eta_-(X,X')\vert > t\}} \\
        \leq{}& 8C^2_{\mathrm{NA}}t^{2q} + \frac{4}{t}\mathbb{E}_{X}\mathbb{E}_{X'} \vert f(X,X') - f^*_{\mathrm{hinge}}(X,X') \vert \cdot \vert \eta_+(X,X') - \eta_-(X,X')\vert \\
        \leq{}& 8C^2_{\mathrm{NA}}t^{2q} + \frac{4}{t}\mathbb{E}(Q\phi_{\mathrm{hinge}f}  - Q\phi_{\mathrm{hinge}f^*_{\mathrm{hinge}}}) \\
        \leq{}& 2^{\frac{2q+3}{2q+1}}C_{\mathrm{NA}}^{\frac{2}{2q+1}}q^{\frac{1}{2q+1}}(2+q^{-1})
        \big(\mathbb{E}(Q\phi_{\mathrm{hinge}f}  - Q\phi_{\mathrm{hinge}f^*_{\mathrm{hinge}}})\big)^{\frac{2q}{2q+1}}.
    \end{align*}
    In the last inequality, we choose the optimal 
    \[t:= \left(\frac{\mathbb{E}(Q\phi_{\mathrm{hinge}f}  - Q\phi_{\mathrm{hinge}f^*_{\mathrm{hinge}}})}
    {4qC^2_{\mathrm{NA}}}\right)^{\frac{1}{2q+1}}\] to minimize the sum.

    For the best variance exponent, we use the connection between \textbf{LN($q$)} and \textbf{NA($q$)}.
    If $q \in [0, 1/2]$, \textbf{LN($q$)} implies \textbf{NA($q$)} and $2q/(2q+1) \geq q$, hence $\tau = 2q/(2q+1)$.
    If $q \in (1/2, 1]$, \textbf{NA($q$)} implies \textbf{LN($q'$)} for all $q' < q$, hence $\tau$ can be chosen arbitrarily close
    to $q$ with the constant $V$ depending on $\tau$.
    If $q \in (1, \infty]$, \textbf{LN($q$)} or \textbf{NA($q$)} both imply \textbf{LN($1$)}, hence $\tau = 1$. Thus we complete the proof.
\end{proof}
From the discussion above, we see that the global low noise condition can always derive a sharper exponent $\tau$ under the same noise exponent $q$ compared with the noise condition in 
Assumption \ref{Tsybakov's noise condition} and hence can derive faster learning rates. One can examine these noise conditions carefully in an actual instance.

\bibliographystyle{plain}
\bibliography{main.bib}
\end{document}